\documentclass[a4paper]{article}

\usepackage[english]{babel}
\usepackage[utf8x]{inputenc}
\usepackage[T1]{fontenc}
\usepackage[affil-it]{authblk}

\usepackage[a4paper,top=3cm,bottom=2cm,left=2.5cm,right=2.5cm,marginparwidth=1.75cm]{geometry}

\usepackage[ruled,vlined,linesnumbered]{algorithm2e}
\usepackage{amsmath}
\usepackage{graphicx}
\usepackage{float}
\usepackage[tikz]{mdframed}
\usepackage{subfigure}
\usepackage{epsfig,amssymb,amsthm}
\usepackage[toc,page]{appendix}
\usepackage{url}
\usepackage{natbib}
\usepackage[colorinlistoftodos]{todonotes}
\usepackage[colorlinks=true, allcolors=blue]{hyperref}
\usepackage[toc,page]{appendix}

\newcommand{\argmax}{{\text{argmax}}}
\newcommand{\argmin}{{\text{argmin}}}

\newtheorem{theorem}{Theorem}
\newtheorem{lemma}{Lemma}
\newtheorem{assumption}{Assumption} 
\newtheorem{proposition}{Proposition} 
\newtheorem{remark}{Remark}
\newtheorem{corollary}{Corollary}
\newtheorem{definition}{Definition}

\newtheorem*{proof2*}{\bf{Proof Sketch}}

\newcommand{\BlackBox}{\rule{1.5ex}{1.5ex}}  

\newenvironment{proof}{\par\noindent{\bf Proof\ }}{\hfill\BlackBox\\[2mm]}
\allowdisplaybreaks
\makeatother

\title{A Bayesian Approach to (Online) Transfer Learning: Theory and Algorithms}
\author{Xuetong Wu$^1$, Jonathan H. Manton$^1$, Uwe Aickelin$^2$, Jingge Zhu$^1$ }
\affil{$^1$ Department of Electronic and Electrical Engineering \\ $^2$ Department of Computing and Information Systems  \\University of Melbourne, Australia \\ 
\textnormal{\texttt{xuetongw1@student.unimelb.edu.au}} \\
\textnormal{\texttt{\{jmanton, uwe.aickelin, jingge.zhu\}@unimelb.edu.au } }}

\begin{document}
\maketitle

\begin{abstract}
Transfer learning is a machine learning paradigm where knowledge from one problem is utilized to solve a new but related problem. While conceivable that knowledge from one task could be useful for solving a related task, if not executed properly, transfer learning algorithms can impair the learning performance instead of improving it --- commonly known as \textit{negative transfer.} In this paper, we study transfer learning from a Bayesian perspective, where a parametric statistical model is used. Specifically, we study three variants of transfer learning problems, instantaneous, online, and time-variant transfer learning. For each problem, we define an appropriate objective function, and provide either exact expressions or upper bounds on the learning performance using information-theoretic quantities, which allow simple and explicit characterizations when the sample size becomes large. Furthermore, examples show that the derived bounds are accurate even for small sample sizes. The obtained bounds give valuable insights into the effect of prior knowledge for transfer learning, at least with respect to our Bayesian formulation of the transfer learning problem. In particular, we formally characterize the conditions under which negative transfer occurs. Lastly, we devise two (online) transfer learning algorithms that are amenable to practical implementations, one of which does not require the parametric assumption. We demonstrate the effectiveness of our algorithms with real data sets, focusing primarily on when the source and target data have strong similarities.

\noindent \bf{Key Words:}  Transfer learning, Conditional Mutual Information, Prior Information, Negative Transfer

\end{abstract}

\tableofcontents

\section{Introduction}
Machine learning has been widely studied and used in many real-world applications. One crucial assumption for traditional machine learning algorithms is that the training and target data are drawn from the same distribution. However, this assumption may not always hold in practice. This may be due to the fact that training and testing data are time-varying, or that it is difficult to collect data from testing distributions due to annotation expenses or privacy considerations. To tackle the issues of distribution mismatch, there is a need to develop learning algorithms that can ``transfer'' knowledge across different domains. Such transfer learning algorithms leverage knowledge from one or more \emph{source} domains to resolve the problem (or improve the performance) in a related \emph{target} domain. This problem has been brought to the fore due to the rapid growth of different types of data and the rise of complicated learning models such as pre-trained neural networks \citep{devlin2018bert,yang2019xlnet} and has been widely employed in natural language processing \citep{ruder2019transfer}, computer vision \citep{wang2018deep} and recommender systems \citep{pan2010transfer}. 

Currently, transfer learning (also known as domain adaptation) problems are widely investigated with different setups. Most existing transfer learning methods focus on offline settings where both batch  target and source data are available in the training phase \citep{pan2009survey,weiss2016survey,zhuang2020comprehensive,yang2020transfer}. In the testing phase, the performance is evaluated with previously unseen target data. We refer to this setup as instantaneous transfer learning (ITL); see Figure~\ref{subfigure:inst}. To accommodate applications where data arrive sequentially (e.g., predicting stock market prices), we also study online transfer learning (OTL), first proposed by \cite{zhao2014online}. The framework of OTL is illustrated in Figure~\ref{subfigure:online}, where the decisions are sequentially made with the aid of source and historical target data. The online framework has been extended to many other problems such as multi-source transfers \citep{wu2017online,kang2020online},  multi-task problems \citep{he2011graphbased} and iterative domain adaptation \citep{bhatt2016multi}. The third setup we consider is time-variant transfer learning (TVTL) which is an extension of OTL where the target data are drawn sequentially from a possibly time-variant distribution (see Figure~\ref{subfigure:tv}). This assumption is relevant for some practical problems such as climate change and geographical process detection, where the data distribution may vary from time to time. Similar problems have been studied in \cite{liu2020learning, wang2020continuously}. 

In this work, we propose a general framework for instantaneous, online and time-variant transfer learning problems that is suitable for a very general transfer learning setup. Specifically, we formulate the transfer learning problems under the assumptions that the source and target data distributions are parameterized by some unknown but fixed parameters. Then we define the corresponding evaluation criterion for each case, and propose an information-theoretic based algorithm for learning in the target domain. In a nutshell, the upper bounds of the learning performance are characterized by the conditional mutual information (CMI) between the model parameters and the testing samples conditioned on training samples, and the asymptotic approximation w.r.t. the sample sizes is derived if the prior distribution is proper. Practically, the bound can also be applied to the scenario where only limited source and target data are available.

\begin{figure}[h!]
    \centering
    \subfigure[ITL\label{subfigure:inst}]{\includegraphics[height = 3cm]{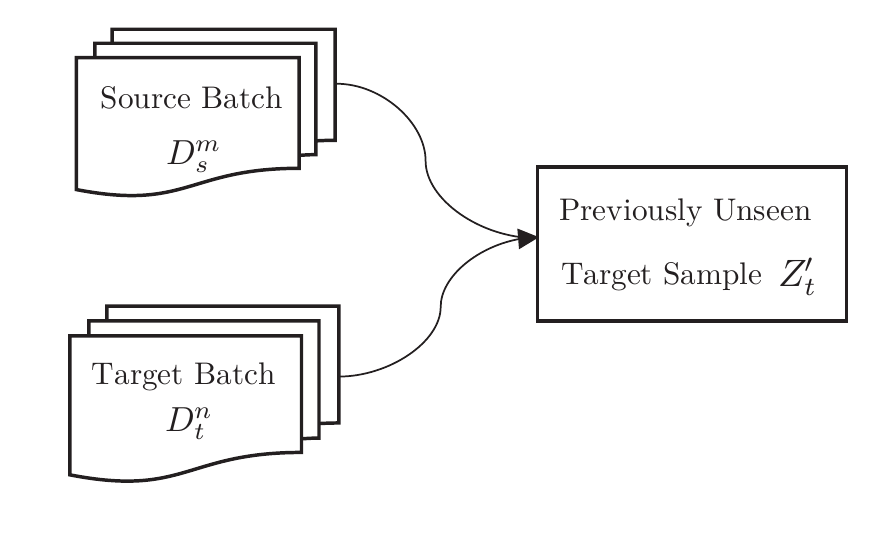}}
    \subfigure[OTL\label{subfigure:online}]{\includegraphics[height = 3cm]{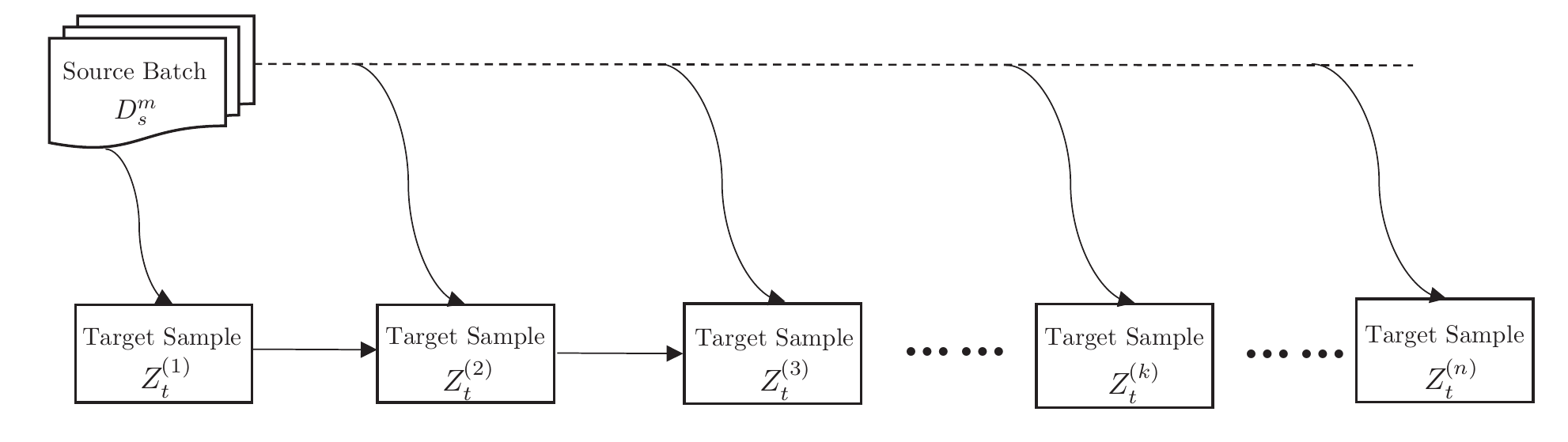}}
    \subfigure[TVTL\label{subfigure:tv}]{\includegraphics[height = 4cm]{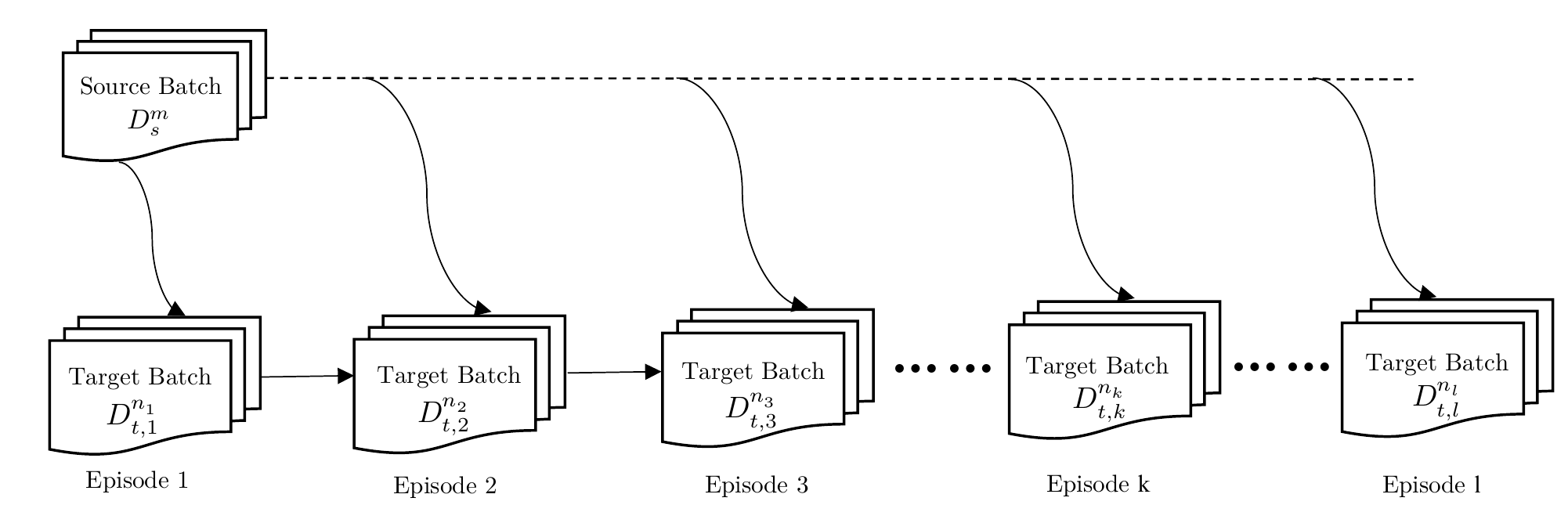}}
    \caption{We will investigate various transfer learning frameworks in this work where (a) describes the instantaneous transfer learning: given both batch target and source data, the task is to predict for a single previously unseen target data. (b) shows the learning regime for the online transfer learning where a batch source data is available and the target data will arrive sequentially from a fixed distribution, the task is to make sequential predictions for incoming target data. (c) illustrates the process for the time-variant transfer learning where a batch source data is available and the target data will arrive sequentially from a possibly time-variant distribution, the task is to make sequential predictions for incoming target data for each episode. }
    \label{fig:otl}
\end{figure}

\subsection{Related Work}\label{Sec2}
\subsubsection*{Transfer Learning}
We refer to three excellent surveys on transfer learning, \cite{pan2009survey}, \cite{weiss2016survey} and \cite{zhuang2020comprehensive}, which formally define the problem of transfer learning, and provide a comprehensive overview of transfer learning methodologies. We mainly focus on the condition that both source and target data have the same feature and label spaces, which is known as \emph{homogeneous transfer learning}. Most technologies to conquer the distribution shifting are roughly categorized into instance-based, feature-based, parameter-based and deep learning-based methods. Instance-based methods identify source samples that bear a likeness to target samples by importance re-weighting \citep{gretton2009covariate,cortes2008sample}. Feature-based methods map both the source and target data to a new latent space where the discrepancy of their (empirical) distribution embeddings (e.g., kernel embeddings) is small under some metric and then construct learning models with the new representations \citep{pan2010transfer,long2013transfer,zhang2017joint}. The idea of parameter-based methods is to initially construct a model using the source data and then learned model parameters are shared with the target domain as a pre-obtained model for further fine-tuning or regularization \citep{duan2009domain,kuzborskij2013stability}. Deep learning-based methods use deep neural networks to either learn new representations in both source and target domains for efficient knowledge transfer \citep{long2015learning,tzeng2017adversarial,long2017deep} or pre-train a model from the source domain that generalizes well in the target domain \citep{devlin2018bert,yang2019xlnet}, which takes advantage of both feature-based and parameter-based methods. However, the majority of these methods focus on the empirical verification on the effect of source samples instead of focusing on a rigorous theoretical analysis of their algorithms. For example, we lack the understanding of how the source data explicitly affects the generalization error in the target domain, and it is not clear when negative transfer happens given a specific algorithm \citep{rosenstein2005transfer}. Moreover, there is no unified framework for analysing this type of problem. Current theoretical analyses for transfer learning focus on either the co-variate shift or conditional distribution alignment \citep{redko2020survey}. To this end, various metrics, such as $\mathcal{H}\Delta\mathcal{H}$ divergence \citep{ben2010theory,zhang2019bridging}, maximum mean discrepancy \citep{pan2010transfer,long2013transfer,zhang2017joint}, KL divergence \citep{wu2020information} and density ratios of the joint distribution between the source and target \citep{wang2019characterizing}, are developed to measure the similarity between the source and target domains. It is widely recognized that minimising such divergence is more likely to bring about successful knowledge transfer, and prior knowledge over the source and target domains effectively improves the prediction as demonstrated in many papers \citep{badreddine2019injecting,cao2013practical,mieth2019using}. However, to the best of our knowledge, there is no work that theoretically defines this prior knowledge and its effect on the learning guarantees. 

\subsubsection*{Online Transfer Learning}
In contrast to conventional batch settings, online transfer learning has attracted more attention in recent research work, where the target data may arrive sequentially. In this particular setting, key questions include the following. Can the source data can help reduce the prediction loss for the target? Under what conditions will the source data be helpful? How does the source data interfere with the prediction of target data and how does the learning performance vary quantitatively?  To answer these questions formally, \cite{zhao2014online} first proposed the OTL framework for binary classification with linear models, while the learning metrics and loss functions are limited, i.e. performance is evaluated using very specific metrics such as the number of mistakes. Such a learning framework in general does not exploit the structures (or distributions) of the data or model parameters. Other related works such as \cite{hoffman2014continuous,mancini2019adagraph,liu2020learning} consider the time-evolving target domains and empirically show the usefulness of their proposed methods. Yet the theoretical guarantees on when or whether negative transfer happens are less investigated. \cite{kumagai2016learning} studies the time-variant domain adaptation problem using variational Bayesian inference, which is close to our framework. However, the latest learned hypothesis is restricted to normal distributions with a linear combination of previously learned hypotheses, which may not always be the case in real problems. Among the many other papers on OTL algorithms, relatively few focus on a rigorous theoretical analysis of the learning performance. 

\subsubsection*{Negative Transfer}
The key points for successful transfer learning are knowing how to transfer, what to transfer and when to transfer \citep{yang2020transfer}. Researchers have so far focused more on how and what to transfer but have paid less attention to when to transfer. Characterizing and determining when to transfer is of great importance since the source data is not always useful. If the source is very different from the target or if we do not execute the learning procedures properly, the source data will instead impair the performance on the target domain \citep{rosenstein2005transfer, wang2019characterizing}. As the effectiveness of transfer learning is not always guaranteed, there is a need to develop a robust methodology to overcome the \emph{negative transfer} problem (see \cite{zhang2020overcoming} for an overview). Roughly speaking, negative transfer can be a consequence of poor source and target data quality, the ``distance'' between the source and target domains, or a difference between the source and target learning objectives. The idea of avoiding negative transfer is to minimise the aforementioned \emph{domain divergence} between the source and target in terms of the underlying distributions. Despite the fact that the effectiveness of source data depends on such divergences, negative transfer is inevitable in many empirical experiments.  Notably, few papers formally study the problem of negative transfer, which is a crucial question in transfer learning.


\subsubsection*{Information-theoretic Framework and Universal Prediction}
An information-theoretic framework has been established and studied in many online learning and reinforcement learning problems (see \cite{lazaric2008transfer,lazaric2012transfer,zhan2015online,taylor2009transfer, merhav1998universal,russo2016information} for references). One advantage of this framework is that information-theoretic tools are powerful in studying asymptotic behaviours as well as deriving learning performance bounds for various statistical problems. Additionally, information-theoretic quantities such as mutual information and KL divergence (relative entropy) give natural interpretations for the learning bounds. In this paper, we establish learning bounds for various transfer learning setups by using the information-theoretic concept of universal prediction \cite{merhav1998universal}. Here, ``universal'' means that the predictor does not depend on the unknown underlying distribution and yet performs essentially as well as if the distribution was known in advance. By modelling the data distributions with parametric settings, we follow the line of work from \cite{feder1992universal,cover1996universal,merhav1998universal,shkel2018sequential,zhan2015online} and derive the CMI-based learning bounds for different transfer learning problems. Taking advantage of the asymptotic properties of the CMI, we could theoretically identify the usefulness of the source data and rigorously claim the scenarios when negative or positive transfer will happen.

\subsection{Contributions}
We brief summarize the main contributions of our paper as follows.
\begin{itemize}
    \item We formulate the instantaneous, online and time-variant transfer learning problems under parametric distribution conditions where the source data are sampled from $P_{\theta^*_s}$ and target data are sampled from $P_{\theta^*_t}$ ($P_{\theta^*_{t,i}}$ for TVTL at episode $i$) from a parametric family $\mathcal{P}_{\theta}, \theta \in \Lambda$. With the proposed \emph{mixture strategy}, the learning performance for each problem is characterized by the conditional mutual information (CMI) between the distribution parameters and data samples. Furthermore, we quantitatively give an asymptotic estimation of CMI for each problem, which clearly shows how the learning performance depends on the target and source sample sizes. These bounds also give explicit connection between the learning performance and the prior knowledge over $\theta^*_t$ ($\theta^*_{t,i}$ for TVTL) and $\theta^*_s$ along with their common parameters.
    \item  Based upon the asymptotic bounds, we define the concept of \emph{proper prior} and show that the improper prior will lead to the negative transfer. We also identify scenarios when the positive transfer will happen, which provides an effective guideline to avoid the negative transfer. To the best of our knowledge, this is the first work that theoretically and quantitatively characterizes phenomena of negative and positive transfer.
    \item To extend our theoretical study to practical implementation, we devise two efficient transfer learning algorithms inspired by the mixture strategy.  Specifically, one algorithm is not restricted to the parametric models, thus extending our results to more general applications. Experimenting on both the synthetic and real data sets, the results empirically confirm the theoretical results and show the usefulness of our proposed algorithms. 
\end{itemize}

This paper is structured as follows.  In Section~\ref{Sec3}, we formally formulate the problems for instantaneous, online and time-variant transfer learning. The main theoretical results and discussions are presented in Section~\ref{Sec4}. The proposed algorithms and  experiments are then followed and presented in Section~\ref{Sec5} along with some practical applications. Section~\ref{Sec7} concludes the work and carries out some future works and Section~\ref{Sec6} gives main proofs of the paper.

\section{Problem Formulation}\label{Sec3}
In this section, we formally formulate the ITL, OTL and TVTL problems. We will use the convention that capital letters denote random variables and lower-case letters their realizations. We define $a \vee b=\max (a, b)$ and $a \wedge b=\min (a, b)$, and denote the support of a probability distribution by $\operatorname{supp}(\cdot)$. We use $f(n) \asymp g(n)$ to signify $f(n)$ grows at the same asymptotic rate as $g(n)$, meaning there exists some positive integer $n_0$ such that for all $n > n_0$, $c_1g(n) \leq f(n) \leq c_2g(n)$ always holds for some positive values $c_1$ and $c_2$. We also use $f(n) \lesssim g(n)$ (for convenience, we will also use $f(n) = O(g(n))$ interchangeably) to mean $f(n)$ grows asymptotically no faster than $g(n)$, that is, there exists some integer $n_0$ such that for all $n > n_0$, $ f(n) \leq c_3g(n)$ always holds for some positive value $c_3$. 

\subsubsection*{Instantaneous Transfer Learning} 
Assume the source data $ D^m_s = (Z^{(1)}_s,\cdots,Z^{(m)}_s) \in \mathcal{Z}^{m}$ and the target data $ D^n_t = (Z^{(1)}_t,\cdots,Z^{(n)}_t) \in \mathcal{Z}^{n}$ are given, where each sample takes on a value in $\mathcal{Z}$. Note that $\mathcal{Z}$ is a separable metric space with probability measures assumed to be defined on the Borel subsets of $\mathbb{R}^k$. We will predict a single previously unseen target sample $Z'_t$ using $D^m_s$ and $D^{n}_t$ with the predictor $b : \mathcal{Z}^{m} \times \mathcal{Z}^{n} \rightarrow \hat{\mathcal{Z}}$. For simplicity we assume $\hat{\mathcal{Z}} = \mathcal{Z}$, but note that $\hat{\mathcal{Z}}$ could be different from $\mathcal{Z}$ in general. Associated with this predictor $b$ and the actual outcome $z'_t$, we introduce a loss function $\ell: \hat{\mathcal{Z}} \times \mathcal{Z} \rightarrow \mathbb{R}$ to evaluate the prediction performance.  We will later show how to construct $b$ properly for different loss functions. We make the following assumptions on data distributions. 
\begin{assumption}[Parametric Distributions]\label{asp:para-dist}
    We assume that source and target data are generated independently in an i.i.d.\@ fashion. Specifically, the joint distribution of the source and target data $P_{\theta^*_s,\theta^*_t}({D_t^n, D_s^m}, Z'_t)$ can be written as 
    \begin{equation}
    P_{\theta^*_s,\theta^*_t}(D_t^n, D_s^m, Z'_t)= P_{\theta_t^*}(Z'_t) \prod_{i=1}^n P_{\theta_t^*}(Z_t^{(i)}) \prod_{j=1}^m P_{\theta_s^*}(Z_s^{(j)}),
    \end{equation}
    where $P_{\theta_t^*}(Z)$ and $P_{\theta_s^*}(Z)$ are two probability density functions in a parametrized family of distributions $\mathcal{P} = \{P_{\theta}\}_{\theta \in \Lambda}$ with respect to a fixed $\sigma$-finite measure $\mu(dz)$. Here $\Lambda$ is a closed set on $\mathbb{R}^{d}$ and $\theta^*_t,\theta^*_s$ are the interior points of $\Lambda$.
\end{assumption}
After observing $n$ target samples and $m$ source samples, assume $\ell$ is integrable w.r.t. the measure $\mu$ given any $b$, we want to minimize the corresponding \emph{excess risk} defined as 
\begin{align}
    \mathcal{R}_{I} := \mathbb{E}_{\theta^*_s,\theta^*_t} \left[ \ell\left(b, Z'_t\right) - \ell(b^*, Z'_t) \right], \label{eq:excess-risk}
\end{align}
where $b$ is the predictor we made based on the source data $D^m_s$ and target data $D^{n}_t$ but without the knowledge of $\theta^*_s$ and $\theta^*_t$. Under certain loss functions, the predictor $b^*$ is set to be the optimal one that can depend on true target distributions $P_{\theta^*_t}$, which will be specified later and shown to have a unique optimal. For other general loss functions, under suitable continuity conditions, we do have optimal predictor depending on $P_{\theta^*_t}$, see \cite{haussler1998sequential,merhav1998universal} for examples. If not otherwise specified, the notation $\mathbb{E}_{\theta_s,\theta_t}[\cdot]$ (similarly, $\mathbb{E}_{\theta_t}[\cdot]$ and $\mathbb{E}_{\theta_s}[\cdot]$) denotes the expectation taken over all source and target samples that are drawn from ${P}_{\theta_s}$ and ${P}_{\theta_t}$. We will call this problem setup "\emph{instantaneous transfer learning}" where the subscript of $R_I$ originates. In this model, both target and source data are given in batch in the training phase, and the learned predictor will be applied to a single unseen target sample. 

\subsubsection*{Online Transfer Learning}
Assume the source data $ D^m_s = (Z^{(1)}_s,\cdots,Z^{(m)}_s) \in \mathcal{Z}^{m}$ are given in batch while the target data are received sequentially as $Z^{(1)}_t, Z^{(2)}_t, \cdots, Z^{(k)}_t, \cdots $ where each sample takes value in $\mathcal{Z}$. At each time instant $k$, after having seen $D^{k-1}_{t} = (Z^{(1)}_t, Z^{(2)}_t,\cdots, Z^{(k-1)}_t)$, we predict $Z^{(k)}_t$ using $D^m_s$ and $D^{k-1}_t$ with the predictor $b_k : \mathcal{Z}^{m} \times \mathcal{Z}^{k-1} \rightarrow \hat{\mathcal{Z}}$. Assume the target sequence and source batch are sampled in a i.i.d. way as described in Assumption~\ref{asp:para-dist}. After observing $n$ target samples, we want to minimise the corresponding \emph{expected regret} defined as
\begin{align}
    \mathcal{R}_{O} := \mathbb{E}_{\theta^*_s,\theta^*_t} \left[\sum_{k=1}^{n} \ell(b_{k}, Z^{(k)}_t) - \sum_{k=1}^{n} \ell(b_k^*, Z^{(k)}_t) \right], \label{eq:regret}
\end{align}
where $b_k$ is the predictor we made based on the source data $D^m_s$ and target data  $D^{k-1}_t$ but without the knowledge of $\theta^*_s$ and $\theta^*_t$. Similarly, under certain loss functions, the predictor $b_k^*$ is set to be the optimal one at each time $k$ that can depend on true target distributions $P_{\theta^*_t}$. The subscript "$O$" is short for "Online”.

\subsubsection*{Time-variant Transfer Learning} 
In the above OTL scenario, we assume the distributions of target samples are time-invariant, which may not always be the case in many real-world applications. In this model, we consider the \emph{time-evolving} target data that the distributions are parametrized by $\theta^*_{t,i}$ where $i \in \mathbb{N+}$ represents the episode of the sequential target distributions. That is, at each $i$, we will receive $n_i$ target samples sequentially drawn from an unknown but fixed distribution $P_{\theta^*_{t,i}}$. Furthermore, at episode $i$, we assume that $\theta^*_{t,i}$ shares $c_i$ common parameters with $\theta^*_{t,i-1}$, known as the target common parameters. We also assume that the source parameter $\theta^*_s$ shares $j_i$ common parameters with $\theta^*_{t,i}$ and $\theta^*_{t,i-1}$, known as the source sharing parameters. For simplicity, we suppose that the target common parameters and the source sharing parameters are not overlapped. Denote $D^m_s$ the source dataset and denote $D^{n_i}_{t,i} = (Z^{(k)}_{t,i} )_{k=1,2,\cdots,n_i}$ the received target dataset till time $n_i$ at episode $i$, we predict $Z^{(n_i+1)}_{t,i}$ using the source data $D^m_s$ and  target data $D^{n_{i-1}}_{t,i-1}$ in previous episode and $D^{n_i}_{t,i}$ in current episode with the predictor $b_{n_i+1,i} : \mathcal{Z}^{m} \times \mathcal{Z}^{n_{i-1}} \times \mathcal{Z}^{n_i} \rightarrow \hat{\mathcal{Z}}$. We further make the following assumption.
\begin{assumption}\label{asp:para-dist-timevarying}
In time-variant transfer learning, we assume that source and time variant target data are generated independently in an i.i.d. fashion. More precisely, the joint distribution of the data sequence till time $n_l$ in episode $l$ can be factorized as,
\begin{align}
    P_{\theta^*_{t,1}, \theta^*_{t,2}, \cdots, \theta^*_{t,l},\theta^*_s}(D^{n_1}_{t,1}, D^{n_2}_{t,2}, \cdots, D^{n_{l}}_{t,l}, D_s^m) = 
    \prod_{i=1}^{l} \prod_{k=1}^{n_i} P_{\theta^*_{t,i}}(Z^{(k)}_{t,i})
    \prod_{j=1}^{m}P_{\theta^*_{s}}(Z^{(j)}_{s}),
\end{align}
where $(P_{\theta_{t,i}^*})_{i=1,2,\cdots,l}$ and $P_{\theta^*_s}$ are in a parametrized family of distributions $\mathcal{P} = \{P_{\theta}\}_{\theta \in \Lambda}$. Here $\Lambda \subseteq \mathbb{R}^{d}$ is some measurable space, and $(\theta_{t,i}^*)_{i=1,2,\cdots,l}$ and $\theta^*_s$ are points in the interior of $\Lambda$. To be consistent, we define $n_0 = 0$ and $\theta^*_{t,0}$ is arbitrarily chosen in $\Lambda$ as a trivial initialization.
\end{assumption}
 Assume the number of target samples $n_i$ at each episode $i$ is known, we are interested in minimising the expected regret till episode $l$ as:
\begin{align}
    \mathcal{R}_{TV} := \sum_{i=1}^{l}\mathbb{E}_{\theta^*_s,\theta^*_{t,i},\theta^*_{t,i-1}} \left[\sum_{k=1}^{n_i} \ell(b_{k,i}, Z^{(k)}_{t,i}) - \sum_{k=1}^{n_i} \ell(b_{k,i}^*, Z^{(k)}_{t,i}) \right], \label{eq:time-variant-regret}
\end{align}
where $b_{k,i}$ is chosen based on the source data $D^m_s$ and target data $D^{k-1}_{t,i}$ and $D^{n_{i-1}}_{t,i-1}$ but without the knowledge of $\theta^*_s$, $\theta^*_{t,i-1}$ and $\theta^*_{t,i}$. The predictor $b_{k,i}^*$ is the optimal decision at each time $k$ that can depend on true target distributions $P_{\theta^*_{t,i}}$ and $P_{\theta^*_{t,i-1}}$. The subscript "$TV$" stands for "Time-Variant”.

\section{Main Results}\label{Sec4}

In this section, we will present our main theoretical results. Under the Assumption~\ref{asp:para-dist}, the parametric conditions allow us to characterize the excess risk in Equation~(\ref{eq:excess-risk}) and the expected regrets in Equation~(\ref{eq:regret}), (\ref{eq:time-variant-regret}) using information-theoretic quantities under the logarithmic loss or any bounded loss functions. To be specific, the CMI captures the performance of both the excess risk and expected regret, and their asymptotic estimations are derived when we have sufficient source and target samples.

\subsection{Information-theoretic Characterization}\label{subsec:ITC}
\subsubsection*{Instantaneous Transfer Learning}
For the sake of simplicity, we first present our main results under the \emph{logarithmic loss} for the ITL setup, which is formally defined as follows.
\begin{definition}[Logarithmic Loss] Let the predictor $b$ be a probability distribution over the target sample $Z'_t$. The logarithmic loss is then defined as
\begin{equation}
    \ell(b,Z'_t) = - \log b(Z'_t).
\end{equation}
\end{definition}
\noindent Under the log loss, the predictor $b$ can be naturally viewed as a conditional distribution $Q$ over the unseen target data given the training data $D^m_s$ and $D^n_t$. With this interpretation in mind, we define the expected loss on test data as
\begin{equation}
    L := - \mathbb{E}_{\theta^*_t,\theta^*_s}\left[ \log Q(Z'_t|D^n_t, D^m_s) \right].
\end{equation}
From \cite{merhav1998universal}, the optimal predictor $b^*$ is given by the underlying target distribution as $b^*(Z'_t) = P_{\theta^*_t}(Z'_t|D^n_t) = P_{\theta^*_t}(Z'_t)$ under the Assumption~\ref{asp:para-dist}. Then the excess risk can be expressed as

\begin{align}
    \mathcal{R}_I &= \mathbb{E}_{\theta^*_t,\theta^*_s}\left[ \ell(b, Z'_t) |D^{n}_t, D^m_s \right] - \mathbb{E}_{\theta^*_t,\theta^*_s}\left[ \ell(b^*, Z'_t) |D^{n}_t, D^m_s \right] \\
    &= \mathbb{E}_{\theta^*_t,\theta^*_s}\left[  \log \frac{P_{\theta^*_t}(Z'_t)}{Q(Z'_t|D^n_t, D^m_s)}\right]. \label{eq:excessrisk}
\end{align}
We define $\Theta_s$ and $\Theta_t$ as random variables over $\Lambda$, which can be interpreted as a random guess of $\theta^*_s$ and $\theta^*_t$, and we choose some probability distribution $\omega$ over $\Theta_s$ and $\Theta_t$ with respect to Lebesgue measure as our prior knowledge on these parameters. The predictor $Q$ is then defined as
\begin{align}
    Q(Z'_t|D^n_t, D^m_s) &= \frac{\int P_{\theta_t}(D^n_t, Z'_t) P_{\theta_s}(D^m_s) \omega(\theta_t,\theta_s) d\theta_t d\theta_s }{\int P_{\theta_t}(D^n_t) P_{\theta^*_s}(D^m_s) \omega(\theta_t,\theta_s) d\theta_t d\theta_s  } \\
    &= \int  P_{\theta_t}(Z'_t) Q(\theta_t,\theta_s|D^m_s, D^n_s) d\theta_sd\theta_t.  \label{eq:excessrisk-mixture}
\end{align}
This choice is known as \emph{mixture strategy} \citep{merhav1998universal,xie2000asymptotic}. We assign a probability distribution $\omega$ over $\Theta_s$ and $\Theta_t$ w.r.t. the Lebesgue measure to represent our prior knowledge, and we update the posterior with the incoming data to approximate the underlying distributions. The Equation~(\ref{eq:excessrisk-mixture}) gives a natural interpretation of a two-stage prediction method on $Z_t'$.  In the first stage, the joint posterior $Q(\theta_s,\theta_t|D_s,D_t)$ gives the estimation of $\theta_s$ and $\theta_t$. In the second stage, the learned $\theta_t$ is applied for prediction in terms of the parametric distribution $P_{\theta_t}(Z'_t)$. One way to comprehend the mixture strategy is that we encode our prior knowledge over target and source domain distributions in terms of the prior distribution $\omega(\Theta_s, \Theta_t)$, and its induced conditional distribution $\omega(\Theta_t|\Theta_s)$ shows our belief over target parameters given the source parameters, e.g., how close $\Theta_t$ and $\Theta_s$ are. 
\begin{remark}
Different from many other transfer learning algorithms where the predictive hypothesis is learned via the empirical risk minimization (ERM) algorithm \citep{ben2010theory, zhang2012generalization, redko2017theoretical}, we learn the distribution parameters and make the prediction from a Bayesian approach. On the one hand, with the mixture strategy, we could encode the prior knowledge on the distribution parameters with the prior distribution $\omega$ and estimate the posterior from the data, which provides us with some insights on how to incorporate the source with the prediction in the target domain. Because the ERM algorithm does not take the data distribution into account, thus it is relatively difficult to see the usefulness of the source data. On the other hand, the mixture strategy is optimal under minimax settings \citep{merhav1998universal} and we later show that it can achieve the excess risk with the optimal rate $O(\frac{1}{n})$ under certain priors, while the bounds for ERM algorithm will usually involve the domain divergence term where it does not converge to zero even with sufficient data \citep{redko2020survey}.
\end{remark}
\noindent The following theorem gives an exact characterization of the excess risk of the predictor $b$ under the logarithmic loss.
\begin{theorem}[Excess Risk with Logarithmic Loss for ITL]\label{thm:excessrisk-log}
Under the logarithmic loss, let the predictor $b = Q(Z'_t|D^{m}_s, D^n_t) = \frac{\int P_{\theta^*_t}(D^n_t, Z'_t) P_{\theta^*_s}(D^m_s) \omega(\theta_t,\theta_s) d\theta_t d\theta_s }{\int P_{\theta^*_t}(D^n_t) P_{\theta^*_s}(D^m_s) \omega(\theta_t,\theta_s) d\theta_t d\theta_s  }$ with the prior distribution $\omega$ described in Equation~(\ref{eq:excessrisk-mixture}), the excess risk can be written as
\begin{equation}
    \mathcal{R}_I  = \mathbb{E}_{\theta^*_t,\theta^*_s}\left[ \frac{P_{\theta^*_t}(Z'_t)}{Q(Z'_t|D^{m}_s, D^n_t)} \right] = I(Z'_t; \Theta_t = \theta^*_t, \Theta_s = \theta^*_s|D^{m}_s, D^n_t).
\end{equation}
\end{theorem}
All proofs in this paper can be found in Section~\ref{Sec6}. A similar learning strategy can be used for more general loss functions.  Given a general bounded loss function $\ell$, we define the predictor $b$ to be
\begin{equation}
    b = \argmin_{b} \mathbb{E}_{Q(Z'_t, D^n_t,D^{m}_s)}\left[ \ell(b,Z'_t)|D^m_s,D^{n}_t \right], \label{eq:general_bk_instant}
\end{equation}
with the choice of the mixture strategy $Q(Z'_t, D^n_t,D^{m}_s) = \int P_{\theta_t,\theta_s}(Z'_t, D^n_t,D^{m}_s)\omega(\theta_t, \theta_s)d \theta_t d\theta_s$ for some prior $\omega$. The optimal predictor is then given by
\begin{equation}
b^* = \argmin_{b} \mathbb{E}_{P_{\theta^*_t}(D^n_t,D^m_s, Z'_t)}\left[ \ell(b,Z'_t)|D^m_s, D^{n}_t \right]. \label{eq:general_bkstar_instant}
\end{equation}
We have the following theorem for general bounded loss functions.
\begin{theorem}[Excess Risk with Bounded Loss for ITL]\label{thm:excessrisk-general}
Assume the loss function satisfies $|\ell(b,z) - \ell(b^*,z)| \leq M$ for any observation $z$ and any two predictors $b,b^*$. Then the excess risk induced by $b$ and $b^*$ in Equation~(\ref{eq:general_bk_instant}) and~(\ref{eq:general_bkstar_instant}) can be bounded as
\begin{equation}
    \mathcal{R}_I  \leq M\sqrt{2I(Z'_t;\Theta_t = \theta^*_t, \Theta_s = \theta^*_s|D^{m}_s, D^n_s)}. \label{ineq:bound_instanT}
\end{equation}
\end{theorem}
The above theorems imply that under both logarithmic loss and other bounded loss functions, with a certain prior $\omega$, the excess risk induced by the mixture strategy is captured by the conditional mutual information between the sample $Z'_t$ and $\Theta_t, \Theta_s$ that are evaluated at $\theta^*_t$ and $\theta^*_s$ given the source and target data. However, the expressions involving CMI is not very informative, in the sense that it does not clearly show the effect of source data in transfer learning. Our analysis in asymptotic estimation will provide insight into the usefulness of source data. We will give the asymptotic estimation of this quantity later.

\subsubsection*{Online Transfer Learning}
Techniques for instantaneous transfer learning can be extended to handle the online transfer learning problem. We first examine the expected regret under the \emph{logarithmic loss}. Assume we have $m$ source samples, at each time $k$, we may now view the predictor as a conditional probability distribution $b_k(z^{(k)}_t) = Q(z^{(k)}_t|D^{k-1}_t, D^{m}_s)$ conditioned on both source and target data. Using the same argument as in ITL, the optimal predictor $b^*_k$ is naturally given by the true target distribution over $z^{(k)}_t$ as $b^*_k(z^{(k)}_t) = P_{\theta^*_t}(z^{(k)}_t)$. Then the expected regret till time $n$ can be written explicitly as
\begin{align}
   \mathcal{R}_O &= \mathbb{E}_{\theta^*_s,\theta^*_{t}}\left[  \log \frac{1}{Q(D^n_t| D^{m}_s)}-  \log \frac{1}{P_{\theta^*_t}(D^n_t)} \right].  \label{eq:logloss}
\end{align}
The effect of source data is reflected in the conditional distribution $Q(D^n_t| D^{m}_s)$.  In particular, we choose the predictor $Q(D^n_t|D^m_s)$ with the mixture strategy as
\begin{align}
  Q(D^n_t|D^m_s) &= \frac{\int P_{\theta_t,\theta_s}(D^n_t, D^m_s)\omega(\theta_t, \theta_s)d \theta_t d\theta_s}{\int P_{\theta_s}(D^m_s)\omega(\theta_s)d\theta_s} \nonumber \\
  &= \frac{\int P_{\theta_t}(D^n_t)\omega(\theta_t|\theta_s) P_{\theta_s}(D^m_s)\omega(\theta_s)d \theta_t d\theta_s}{Q(D^m_s)} \nonumber \\
  &= \int P_{\theta_t}(D^n_t)\omega(\theta_t|\theta_s) d\theta_t \frac{ P_{\theta_s}(D^m_s)\omega(\theta_s)} {Q(D^m_s)} d\theta_s \nonumber \\
  &= \int \int P_{\theta_t}(D^n_t)\omega(\theta_t|\theta_s)  d\theta_t Q(\theta_s|D^m_s) d\theta_s, \label{eq:mixture-online}
\end{align}
where $\omega(\theta_s)$ and $\omega(\theta_t|\theta_s)$ are induced by the joint distribution $\omega(\theta_s, \theta_t)$. From Eq~(\ref{eq:mixture-online}), the mixture strategy quantitatively explains how transfer learning is implemented via the posterior updates of $\theta_t$ sequentially. Intuitively speaking, the posterior $Q(\theta_s|D^m_s)$ firstly gives an estimate of $\theta^*_s$ from the source data, then the conditional prior $\omega(\theta_t|\theta_s)$ reflects our belief upon $\theta^*_t$ given $\theta_s$ estimated from source data. With the choice of $Q(D^n_t|D^m_s)$, the expected regret can be explicitly characterized in the following theorem.
\begin{theorem} [Expected Regret with Logarithmic Loss for OTL]\label{thm:expreg-log}
With the mixture strategy $Q(D^n_t|D^m_s)$ in (\ref{eq:mixture-online}), the expected regret in (\ref{eq:logloss}) can be written as
\begin{align}
   \mathcal{R}_O = \mathbb{E}_{\theta^*_s,\theta^*_t}\left[ \log\frac{P_{\theta^*_t}(D^n_t)}{Q(D^n_t
   |D^m_s)} \right] = I(D^n_t;\Theta_t = \theta^*_t, \Theta_s = \theta^*_s|D^{m}_s), \label{eq:cmi} 
\end{align}
where $I(D^n_t;\Theta_t = \theta^*_t, \Theta_s = \theta^*_s|D^{m}_s)$ denotes the conditional mutual information $I(D^n_t;\Theta_t, \Theta_s | D^{m}_s)$ evaluated at $\Theta_t = \theta^*_t, \Theta_s = \theta^*_s$.
\end{theorem}
\noindent For general loss functions, we define the predictor $b_k$ at time $k$ to be
\begin{equation}
    b_k = \argmin_{b} \mathbb{E}_{Q(D^k_t,D^{m}_s)}\left[ \ell(b,z^{(k)}_t)|D^m_s,D^{k-1}_t \right], \label{eq:general_bk}
\end{equation}
with the choice of the mixture strategy $Q(D^k_t,D^{m}_s) = \int P_{\theta_t,\theta_s}(D^k_t, D^m_s)\omega(\theta_t, \theta_s)d \theta_t d\theta_s$ for some prior $\omega$. The optimal predictor is then given by
\begin{equation}
b_k^* = \argmin_{b} \mathbb{E}_{P_{\theta^*_t,\theta^*_s}(D^k_t, D^m_s)}\left[ \ell(b,z^{(k)}_t)| D^m_s, D^{k-1}_t\right]. \label{eq:general_bkstar}
\end{equation}
As a consequence, we arrive at the following theorem.
\begin{theorem}[Expected Regret with Bounded Loss for OTL]\label{thm:expreg-generalloss}
Assume the loss function satisfies $|\ell(b,z) - \ell(b^*,z)| \leq M$ for any observation $z$ and the predictors $b,b^*$. Then the true expected regret induced by $b_k$ and $b^*_k$ in Equation~(\ref{eq:general_bk}) and~(\ref{eq:general_bkstar}) can be bounded as
\begin{equation}
    \mathcal{R}_O \leq M\sqrt{2n I(D^n_t;\Theta_t = \theta^*_t, \Theta_s = \theta^*_s|D^{m}_s)}. \label{ineq:bound}
\end{equation}
\end{theorem}
We can see the analogy from the above theorem that under both the logarithmic loss and bounded loss, the expected regret induced by the mixture strategy is also characterized by CMI evaluated at $\theta^*_t$ and $\theta^*_s$. The expected regret can be thought as the accumulated excess risk from sequential prediction, where each single prediction is made from the posterior of the target parameters. Nevertheless, it is difficult to directly spectate the effects of the prior and sample sizes. 

\subsubsection*{Time-variant Transfer Learning}
The treatment of time-variant transfer learning is similar. Under the logarithmic loss, we rewrite the objective function in Equation~(\ref{eq:time-variant-regret}) as
\begin{align}
    \mathcal{R}_{TV} &= \sum_{i=1}^{l}\mathbb{E}_{\theta^*_s,\theta^*_{t,i},\theta^*_{t,i-1}} \left[\sum_{k=1}^{n_i} \ell(b_{k,i}, Z^{(k)}_{t,i}) - \sum_{k=1}^{n_i} \ell(b_{i,k}^*, Z^{(k)}_{t,i}) \right] \\
    &= \sum_{i=1}^{l}\mathbb{E}_{\theta^*_s,\theta^*_{t,i},\theta^*_{t,i-1}} \left[ \frac{P_{\theta^*_{t,i}}(D^{n_i}_{t,i})}{Q(D^{n_i}_{t,i}|D^m_s, D^{n_{i-1}}_{t,i-1})}\right],
\end{align}
\noindent  We will use the mixture strategy by defining the random variable $\Theta_s$, $\Theta_{t,i}$ and $\Theta_{t,i-1}$ over $\Lambda$ such that with some prior $\omega$, we can formulate the conditional distribution as
\begin{align}
   Q(D^{n_i}_{t,i}|D^m_s, D^{n_{i-1}}_{t,i-1}) &= \frac{Q(D^m_s, D^{n_{i-1}}_{t,i-1}, D^{n_i}_{t,i})}{Q(D^{n_{i-1}}_{t,i-1}, D^{m}_s)} \\
    &= \frac{\int P_{\theta_s, \theta_{t,i-1}, \theta_{t,i}}(D^m_s , D^{n_{i-1}}_{t,i-1}, D^{n_i}_{t,i} )\omega(\theta_s, \theta_{t,i-1}, \theta_{t,i}) d\theta_s d\theta_{t,i-1} d\theta_{t,i} }{\int P_{\theta_s, \theta_{t,i-1}}(D^m_s, D^{n_{i-1}}_{t,i-1})\omega(\theta_s, \theta_{t,i-1})d\theta_s d\theta_{t,i-1}}  \\
    &= \int P_{\theta_{t,i}}(D^{n_i}_{t,i})\omega(\theta_{t,i}|\theta_s, \theta_{t,i-1} )  d\theta_{t,i} Q(\theta_s,\theta_{t,i-1}|D^m_s, D^{n_{i-1}}_{t,i-1}) d\theta_s d\theta_{t,i-1}.
\end{align}
The above prediction distribution suggested that the posterior $Q(\theta_s,\theta_{t,i-1}|D^m_s, D^{n_{i-1}}_{t,i-1})$ firstly gives an estimate of the source parameter and previous target parameter with marginal $\omega(\theta_s, \theta_{t,i-1})$, then the knowledge transfer is reflected on the conditional prior $\omega(\theta_{t,i}| \theta_s, \theta_{t,i-1})$ that may result in a good approximation of $\theta^*_{t,i}$. The conditional prior $\omega(\theta_{t,i}| \theta_s, \theta_{t,i-1})$ can be interpreted as our prior knowledge of the current target state given the previous state and auxiliary source parameters. 
\begin{remark}
At each episode $i$, we view the sequential predictors as the conditional distribution $Q(D^{n_i}_{t,i}|D^m_s, D^{n_{i-1}}_{t,i-1})$ since we only use the previous target data from episode $i-1$, and it should be recognized that this choice is not necessarily the optimal choice. The reasons that we discard earlier target data are two-fold. On the one hand, if $i$ becomes large, the posterior will be hard to compute using all earlier target data and the mixture strategy will become very complicated and inefficient. On the other hand, as the relationship between the target data at episode $i$ and the target data earlier than episode $i-1$ is not explicitly recognized, if the prior distribution is chosen improperly without any prior knowledge, introducing such data may possibly result in worse performance. 

\end{remark}
\noindent By this specific strategy, we have the expected regret in the following theorem.
\begin{theorem}[Expected Regret with Logarithmic Loss for TVTL]\label{coro:log-loss-time-variant}
Under the logarithmic loss, the expected regret in (\ref{eq:time-variant-regret}) can be written as
\begin{align}
    \mathcal{R}_{TV}  = \sum_{i=1}^{l} I(D^{n_i}_{t,i};\Theta_{t,i} = \theta^*_{t,i},\Theta_{t,i-1} = \theta^*_{t,i-1}, \Theta_s = \theta^*_s|D^m_s, D^{n_{i-1}}_{t,i-1}).
\end{align}
\end{theorem}
\noindent Similarly, we can easily generalize the logarithmic loss to other bounded loss $\ell$. Given any loss function $\ell$, we define the predictor $b_{k,i}$ at episode $i$ to be
\begin{equation}
    b_{k,i} = \argmin_{b} \mathbb{E}_{Q(D^{k}_{t,i},D^{n_{i-1}}_{t,i-1}, D^{m}_s)}\left[ \ell(b,z^{(k)}_{t,i})|D^m_s,D^{n_{i-1}}_{t,i-1}, D^{k-1}_{t,i} \right], \label{eq:general_bk_tv}
\end{equation}
with the choice of the mixture strategy
\begin{equation}
Q(D^{k}_{t,i},D^{n_{i-1}}_{t,i-1}, D^{m}_s) = \int P_{\theta_{t,i}, \theta_{t,i-1},\theta_s}(D^{k}_{t,i},D^{n_{i-1}}_{t,i-1}, D^{m}_s)\omega(\theta_{t,i}, \theta_{t,i-1} \theta_s)d \theta_{t,i} d\theta_{t,i-1} d\theta_s
\end{equation}
for some prior $\omega$. The optimal predictor is then given by
\begin{equation}
b_{k,i}^* = \argmin_{b} \mathbb{E}_{P_{\theta^*_{t,i},\theta^*_{t,i-1}, \theta^*_s }(D^{k}_{t,i},D^{n_{i-1}}_{t,i-1}, D^{m}_s)}\left[ \ell(b,z^{(k)}_{t,i})|D^m_s,D^{n_{i-1}}_{t,i-1}, D^{k-1}_{t,i}  \right]. \label{eq:general_bkstar_tv}
\end{equation}
Then following Thereom~\ref{thm:excessrisk-general} and \ref{thm:expreg-generalloss}, we arrive at the theorem that describes the expected regret for time-variant transfer learning.
\begin{theorem}[Expected Regret with Bounded Loss for TVTL]\label{coro:general-loss-time-variant}
Assume the loss function satisfies $|\ell(b,z) - \ell(b^*,z)| \leq M$ for any observation $z$ and the predictors $b,b^*$. Then the true expected regret induced by $b_k$ and $b^*_k$ in Equation~(\ref{eq:general_bk_tv}) and~(\ref{eq:general_bkstar_tv}) can be bounded as
   \begin{align}
    \mathcal{R}_{TV}  & \leq M \sqrt{2 l \sum_{i=1}^{l} n_{i} I(D^{n_i}_{t,i};\Theta_{t,i} = \theta^*_{t,i},\Theta_{t,i-1} = \theta^*_{t,i-1}, \Theta_s = \theta^*_s|D^m_s, D^{n_{i-1}}_{t,i-1}) }.
\end{align}
\end{theorem}

In this section, we characterize the excess risk and expected regrets for instantaneous and online transfer learning scenarios using information-theoretic quantities from a Bayesian perspective. Given bounded or logarithmic loss functions, the performance is captured by the conditional mutual information between the parameters and the test data. The bound implicitly embeds our prior knowledge over target and source parameters in the prior distribution $\omega$. However, the bounds in their current forms are less informative as they do not show what are the effects of the prior knowledge $\omega$ and sample sizes of the source and target domains. To this end, we will give an asymptotic approximation for conditional mutual information in the next subsection.

\subsection{Asymptotic Approximation for Conditional Mutual Information}
\subsubsection*{Instantaneous Transfer Learning}
\noindent To investigate the effect of sample size and prior, first we make the regular assumptions on parametric conditions \citep{clarke1999asymptotic, clarke1990information} and define the proper prior.
\begin{assumption}[Parametric Condition] \label{asp:para-trans}
We make the following assumptions for the source and target distributions:
\begin{itemize}
    \item The source and target distributions $P_{\theta^*_s}(Z_s)$ and $P_{\theta^*_t}(Z_t)$ are twice continuously differentiable at $\theta^*_s$ and $\theta^*_t$ for almost every $Z_s$ and $Z_t$.
    \item  For any $\theta_t, \theta_s \in \Lambda $, there exist $\delta_s, \delta_t > 0$ satisfying, 
\begin{align}
    \mathbb{E}_{\theta_t} \left[ \sup _{\left\|\theta_t-\theta^*_t\right\| \leq \delta}\left|\frac{\partial}{\partial \theta_{t,i}} \log P_{\theta_t}\left(Z_t \right)\right| \right] &< \infty \\
    \mathbb{E}_{\theta_s} \left[ \sup _{\left\|\theta_s-\theta^*_s\right\| \leq \delta}\left|\frac{\partial}{\partial \theta_{s,i}} \log P_{\theta_s}\left(Z_s \right)\right| \right] &< \infty  
\end{align}
for $i = 1,\cdots, d$. In addition, we assume,
\begin{align}
\mathbb{E}_{\theta_t} \left[ \sup _{\left \|\theta_t-\theta^*_t\right\| \leq \delta}\left|\frac{\partial^{2}}{\partial \theta_{i} \partial \theta_{j}} \log P_{\theta_t}\left(Z_t \right)\right|^{2}\right] <\infty \\
\mathbb{E}_{\theta_s} \left[ \sup _{\left\|\theta_s-\theta^*_s\right\| \leq \delta}\left|\frac{\partial^{2}}{\partial \theta_{i} \partial \theta_{j}} \log P_{\theta_s}\left(Z_s \right)\right|^{2}\right] <\infty 
\end{align}
for any $i,j = 1,\cdots, d$. 
\item Let $D_{\textup{KL}}(P_{\theta_s^*}\|P_{\theta_s})$ and  $D_{\textup{KL}}(P_{\theta_t^*}\|P_{\theta_t})$ denote the information (KL) divergence for source distribution $P_{\theta_s}$ and target distribution $P_{\theta_t}$. We assume they are twice continuously differentiable at $\theta^*_s$ and $\theta^*_t$, with the Hessian $J_s(\theta^*_s)$ and $J_t(\theta^*_t)$ positive definite, which are defined by:
\begin{align*}
J_s\left(\theta_s\right) &=\left[\frac{\partial^{2}}{\partial \theta_{i} \partial \theta_{j}} D_{\textup{KL}}(P_{\theta_s^*}\|P_{\theta_s})\right]_{i, j=1 \cdots, d}, \\ 
J_t\left(\theta_t\right) &=\left[\frac{\partial^{2}}{\partial \theta_{i} \partial \theta_{j}} D_{\textup{KL}}(P_{\theta_t^*}\|P_{\theta_t})\right]_{i, j=1 \cdots, d}.
\end{align*}
\item The convergence of a sequence of parameter values is equivalent to the weak convergence of the distributions they index, e.g.,
\begin{equation*}
\theta \rightarrow \theta^* \Leftrightarrow P_{\theta} \rightarrow P_{\theta^*,}
\end{equation*}
for both source and target distributions.
\end{itemize}

\end{assumption}
\begin{definition}[Proper Prior] \label{def:proper-prior}
Given a prior $\omega(\Theta_s, \Theta_t)$, we say,
\begin{itemize}
    \item the induced marginal density $\omega(\Theta_s)$ is proper if it is continuous and positive over the whole support $\Lambda \subseteq \mathbb{R}^d$.
    \item the conditional density $\omega(\Theta_t|\Theta_s)$ is proper if there exist some $\delta_s > 0$ and $\delta_t > 0$ such that $\omega(\theta_t|\theta_s) > 0$
    for any $\theta_s$ and $\theta_t$ satisfying $\|\theta_s - \theta^*_s\| \leq \delta_s$ and $\|\theta_t - \theta^*_t\| \leq \delta_t$. 
    \item $\omega(\Theta_s, \Theta_t)$ is proper if both $\omega(\Theta_s)$ and $\omega(\Theta_t|\Theta_s)$ are proper.
\end{itemize}
\end{definition}
\begin{remark}
We define the proper prior to ensure that the posterior distribution of $\Theta_t$ and $\Theta_s$ given $D^{n}_t$ and $D^{m}_s$ will asymptotically concentrate on neighborhoods of $\theta^*_{t}$ and $\theta^*_s$, respectively. Specifically, if $\omega(\Theta_s)$ is positive and continuous in $\Lambda$, the posterior of $\Theta_s$ will concentrate around $\theta^*_s$ with sufficient source data. With the posterior of $\Theta_s$, the knowledge is transferred via the conditional distribution $\omega(\Theta_T|\Theta_s)$ where we could eventually obtain a good estimation of $\theta^*_t$ only if $\omega(\Theta_T|\Theta_s)$ has positive density around $\theta^*_t$.
\end{remark}
Allowing both $k$ and $m$ to be sufficiently large, we can obtain the following asymptotic results for the excess risk when both $\Theta_s$ and $\Theta_t$ are scalars.
\begin{theorem}[Instantaneous Prediction with Scalar Parameters] \label{thm:inst-scalar} 
Under Assumptions~\ref{asp:para-dist} and \ref{asp:para-trans}, for $\Lambda = \mathbb{R}$ and $\theta^*_s \neq \theta^*_t$,   as $n, m \rightarrow \infty$, the mixture strategy with a proper prior $\omega(\Theta_s,\Theta_t)$ for logarithmic loss yields
\begin{equation}
 \mathcal{R}_I  \asymp \frac{1}{n}.
\end{equation}
\end{theorem}
When $\theta^*_s \neq \theta^*_t$, the above theorem characterizes the excess risk with the rate of $O(\frac{1}{n})$, which achieves the optimal convergence rate for parametric distribution estimation \citep{haussler1995general,xu2020minimum}. However, the expression does not involve the source sample $m$ and the prior distribution $\omega(\Theta_s,\Theta_t)$, indicating that the source data does not help the prediction asymptotically. This is intuitive because given enough target data, we could have a precise estimation of the underlying target distribution, hence the source data is not needed in this asymptotic regime. However, we will see that this is not the case anymore in the OTL setup. 

The above result can be extended to a more typical transfer learning scenario where $\Theta_t, \Theta_s \in \mathbb{R}^d$ with $d > 1$ share some common parameters $\Theta_c \in \mathbb{R}^j$ for $0 \leq j \leq d$. To illustrate, we can write the parameters in the following way,
\begin{align}
    \Theta_s &= (\Theta_{c,1}, \Theta_{c,2},\cdots, \Theta_{c,j}, \quad \Theta_{s,1}, \cdots, \Theta_{s,d-j}  ) =  (\Theta_c, \Theta_{sr}),  \label{eq:paraways}\\
    \Theta_t &= (\underbrace{\Theta_{c,1}, \Theta_{c,2},\cdots, \Theta_{c,j}}_{\text{common parameters}}, \quad \underbrace{\Theta_{t,1}, \cdots, \Theta_{t,d-j})}_{\text{task-specific parameters}} =  (\Theta_c, \Theta_{tr}), \label{eq:parawayt} 
\end{align}
where $\Theta_c \in \mathbb{R}^{j}$ denotes the common parameter vector and $\Theta_{sr} , \Theta_{tr} \in \mathbb{R}^{d-j}$ are task-specific parameter vectors for source and target data, respectively. Then we reach the following theorem that gives the asymptotic estimation of the excess risk with $d > 1$.
\begin{theorem}[Instantaneous Prediction with General Parameterization]\label{thm:instant-gene} Under Assumptions~\ref{asp:para-dist} and \ref{asp:para-trans}, assume $\theta^*_s$ and $\theta^*_t$ are characterized in (\ref{eq:paraways}) and (\ref{eq:parawayt}), and $m \asymp n^p$ for some $p \geq 0$. Let $n \rightarrow \infty$, the mixture strategy with a proper prior $\omega(\Theta_s,\Theta_t)$ for logarithmic loss yields
\begin{equation}
 \mathcal{R}_I  \asymp \frac{d - j}{n} + \frac{j}{n \vee n^p }  , \label{eq:instant_general}
\end{equation}
\end{theorem}
\begin{remark}
From the theorem above, we can conclude that if there is no common parameter $j = 0$ and $d = 1$, we could then recover the result in Theorem~\ref{thm:inst-scalar}. If $j > 0$, the source domain will share some parameters with the target domain and the source data will indeed help improve the "learning cost" of the common parameters. Compared with the typical result without the source as
\begin{align}
 \mathcal{R}_I \asymp \frac{d}{n}, \label{eq:instant_general_without}
\end{align}
the improvement is associated with the component $\frac{j}{n \vee n^p}$, which can be interpreted as the learning cost of $\theta^*_c$. If $m$ is superlinear in $n$ (e.g., $p > 1$), the source samples indeed improve the convergence rate of the estimation for $\theta^*_c$ (but does not change the rate of the estimation for $\theta^*_{tr}$). Moreover, if we consider the extreme case $j = d$ such that the source and target have the same parameterization, the risk will be
\begin{align}
    \mathcal{R}_I \asymp  \frac{d}{n^p \vee n} ,
\end{align}
where source data can yield a better convergence rate for the excess risk compared to Equation~(\ref{eq:instant_general_without}) if $p > 1$.  
\end{remark}

\subsubsection*{Online Transfer Learning}
 For online transfer learning where both $\Theta_s$ and $\Theta_t$ are scalars, we give the asymptotic estimation for CMI as follows.
\begin{theorem}[Online Prediction with Scalar Parameters]\label{thm:consistency-scalar}
 Under Assumptions~\ref{asp:para-dist} and \ref{asp:para-trans}, for $\Lambda = \mathbb{R}$ and $\theta^*_s \neq \theta^*_t$, as $n, m \rightarrow \infty$, the mixture strategy with proper prior $\omega(\Theta_s,\Theta_t)$ for logarithmic loss yields
\begin{align}
      \mathcal{R}_O  -  \frac{1}{2}\log \frac{n}{2\pi e}  \rightarrow \frac{1}{2}\log I_t(\theta^*_t)  + \log\frac{1}{\omega(\theta^*_t|\theta^*_s)},  \label{eq:withsource}
\end{align}
where we define the Fisher information matrix $\mathbb{E}_{\Theta_t} \left[ - \nabla^2_{\Theta_t} \log P_{\Theta_t}(Z_t)\right]$ evaluated at $\Theta_t = \theta^*_t$ as $I_t(\theta^*_t)$.
\end{theorem}
\begin{remark} \label{remark:consist}
Compared to the result without the source data when target sample is abundant \citep{clarke1990information},
\begin{align}
     \mathcal{R}_O   - \frac{1}{2}\log \frac{n}{2\pi e} \rightarrow   \frac{1}{2}\log I_t(\theta^*_t)  + \log\frac{1}{\hat{\omega}(\theta^*_t)} \label{eq:withoutsource}
\end{align}
for some prior $\hat{\omega}(\Theta_t)$, the difference between Equation~(\ref{eq:withsource}) and (\ref{eq:withoutsource}) is $\frac{\hat{\omega}(\theta^*_t)}{\omega(\theta^*_t|\theta^*_s)}$. It says that if the distribution $\omega$ can be chosen such that $\frac{\hat{\omega}(\theta^*_t)}{\omega(\theta^*_t|\theta^*_s)} < 1$, the source data will help to reduce the regret by some constant in the scalar parameter case. However, it should be noted that $\omega$ is chosen without knowing the exact value of $\theta^*_t$ and $\theta^*_s$ so it is not immediately clear if this is always possible. We will show later that if the conditional prior $\omega(\theta_t|\theta_s)$ is proper, it is always possible to find a distribution such that $\frac{\hat{\omega}(\theta^*_t)}{\omega(\theta^*_t|\theta^*_s)} < 1 $. On the contrary, if the prior information between the source and target is incorrect, we cannot guarantee that $\frac{\hat{\omega}(\theta^*_t)}{\omega(\theta^*_t|\theta^*_s)} > 1$ always hold, hence negative transfer will occur in the worst-case scenario. This result provides a formal  characterization of \textit{negative transfer}.
\end{remark}
Notice that the source samples change the constant from $\log\frac{1}{\hat{\omega}(\theta_t)}$ to $\log\frac{1}{\omega(\theta_t|\theta_s)}$, which is independent from $n$. Hence the effect of the source samples vanishes asymptotically as $n$ goes to infinity. However, the asymptotic analysis is still useful for two reasons. Firstly, we will show later in Corollary~\ref{coro:rate-gene-online} that when both $n$ and $m$ approach infinity, the sample complexity of the regret (i.e., how regret scales in terms of $m$ and $n$) will change, depending on how fast $m$ and $n$ grow relative to each other. Secondly, our numerical results show that the asymptotic bound is in fact very accurate even for relatively small $m$ and $n$.

Theorem~\ref{thm:consistency-scalar} holds when the distributions are parametrized by scalars, that is, the source and target do not share parameters and the knowledge transfer is totally reflected on the prior knowledge. Therefore, the effect of the source sample size $m$ is not exhibited in this case. The following theorem characterize the expected regret under general parametrization where $\Theta_s$ and $\Theta_t$ will share $j$ common parameters.
\begin{theorem}[Online Prediction with General Parametrization]\label{theorem:gene-para}
 Under Assumptions~\ref{asp:para-dist} and \ref{asp:para-trans}, with $\Theta_s,\Theta_t \in \mathbb{R}^d$ defined above and as $n \rightarrow \infty$ and $m = cn^p$ for some $c>0$, the mixture strategy with proper prior $\omega(\Theta_s,\Theta_t)$ yields
 \begin{align}
\mathcal{R}_O  - \frac{1}{2} \log \operatorname{det}(\mathbf{I}_{j\times j} + \frac{1}{cn^{p-1}} \Delta_t \Delta^{-1}_s)   - \frac{1}{2} \log \operatorname{det}(n I_t(\theta^{*}_{tr}))  \rightarrow  \frac{d-j}{2}\log \frac{1}{2\pi e} + \log \frac{1}{\omega(\theta^*_t|\theta^*_s)} \label{eq:gene-para-reg}
\end{align}
where $\Delta_t$, $\Delta_s$ and $I_t(\theta^{*}_{tr})$ are Fisher information matrices related quantities and their definitions can be found in (\ref{eq:fisher-start}) - (\ref{eq:fisher-end}) in Section~\ref{proof:gene-para}.
\end{theorem}

\begin{corollary}[Rate of $\mathcal{R}_O$] \label{coro:rate-gene-online} Under the conditions in~Theorem~\ref{theorem:gene-para}, for $0 \leq  p < 1$:
\begin{align}
\mathcal{R}_O \asymp  j(1-p)\log n  + (d-j)\log n .
\end{align}
For $p \geq 1$:
\begin{align}
\mathcal{R}_O  \asymp  \frac{j}{n^{p-1}}   + (d-j)\log n .
\end{align}
\end{corollary}

\begin{remark}\label{remark:knowledge_transfer}
We can intuitively interpret the term $\frac{1}{2} \log \operatorname{det}(\mathbf{I}_{j\times j} + \frac{1}{cn^{p-1}} \Delta_t \Delta^{-1}_s)$ in the Equation~(\ref{eq:gene-para-reg}) as the learning cost of $\theta_c$, which is captured by the source sample sizes $m$. Compared with the typical result without the source as
\begin{align}
    \mathcal{R}_O  - \frac{d}{2}\log \frac{n}{2\pi e}  \rightarrow   \frac{1}{2}\log  \operatorname{det}(I_t(\theta^*_{t}))  +  \log \frac{1}{\hat{\omega}(\theta^*_{t})} \label{eq:online_general_without}
\end{align}
with the rate of $O(d\log n)$, if $m$ is
\begin{itemize}
    \item sublinear in $n$ ($p < 1$), the rate is improved to $O(j(1-p)\log n)$. In this case, as the source data size is small compared to the target data size, the learning rate (for the common parameter part) is only improved by a constant factor $1-p$, changed from $j\log n$ to $j(1-p)\log n$.
    \item linear in $n$ ($p = 1$), the rate (for learning in the common parameters) in this case improves from $O(j\log n)$ to a constant $O(j)$.
    \item superlinear in $n$ ($p > 1$), the rate is $O(\frac{j}{n^{p-1}})$, indicating that abundant source samples indeed improve the performance and the cost for learning the common parameters vanishes in this case.
\end{itemize}
Furthermore, the prior knowledge $\omega(\theta^*_t|\theta^*_s)$ becomes involved in characterizing the expected regret as we discussed in Remark~\ref{remark:consist}. However, it can only change the constant but does not change the rate.
\end{remark}
As a special case, if there is no common parameters ($j = 0$), then as both $m$ and $n$ are sufficiently large:
\begin{align}
    \mathcal{R}_O  - \frac{d}{2}\log \frac{n}{2\pi e}  \rightarrow   \frac{1}{2}\log  \operatorname{det}(I_t(\theta^*_{tr}))  +  \log \frac{1}{\omega(\theta^*_{t}|\theta^*_{s})} 
\end{align}
Let $d = 1$, we can recover the results in Theorem~\ref{thm:consistency-scalar} and the knowledge transfer is only reflected on the prior knowledge $\omega(\theta^*_t|\theta^*_s)$. If the number of the common parameters is $d$ ($j = d$), that is, the source and target distributions are characterized by the same parameters, which yields the asymptotic estimation as
\begin{align}
   \mathcal{R}_O   & - \frac{1}{2}\log \operatorname{det}(\mathbf{I}_{d \times d} + \frac{1}{cn^{p-1}}\Delta_t  \Delta^{-1}_s)  \rightarrow \log\frac{1}{\omega(\theta^*_t|\theta^*_s)} 
\end{align}
Under this case, the rate of the regret depends on the source sample size by setting $j = d$ in~Corollary~\ref{coro:rate-gene-online}. Compared to the ITL case, as can be seen in Theorem~\ref{thm:excessrisk-log} and \ref{thm:excessrisk-general}, the excess risk induced by ITL is only associated with the sample number $m$ and $n$ since the posterior overwhelms the prior for single data prediction. While the expected regret in OTL is further determined by conditional prior distribution as shown in Theorem~\ref{thm:consistency-scalar} and \ref{theorem:gene-para}.

\subsubsection*{Time-variant Transfer Learning}
Under the time-variant transfer learning, with the aforementioned assumption that at episode $i$, the target parameter $\Theta_{t,i}$ have $c_i$ target common parameters with $\Theta_{t,i-1}$, and $\Theta_s$ shares $j_i$ source sharing parameters with $\Theta_{t,i-1}$ and $\Theta_{t,i}$. Let us define the random variables $\Theta_s$, $\Theta_{t,i-1}$ and $\Theta_{t,i}$ with the following parameterization,
\begin{small}
\begin{align}
    \Theta_s &= (\Theta_{c,1}, \Theta_{c,2},\cdots, \Theta_{c,j_i}, \quad \underbrace{\Theta_{s,1}, \cdots, \Theta_{s,d-j_i})}_{\text{source-specific parameters}} =  (\Theta_{c,i}, \Theta_{sr,i}), \\
    \Theta_{t,i-1} &= (\Theta_{c,1}, \Theta_{c,2},\cdots, \Theta_{c,j_i}, \quad \Theta_{v,1}, \cdots, \Theta_{v,c_i}, \quad   \Theta_{t',1}, \cdots, \Theta_{t',d- j_i - c_i}) =  (\Theta_{c,i}, \Theta_{v,i}, \Theta_{tr,i-1}), \\
    \Theta_{t,i} &= (\underbrace{\Theta_{c,1}, \Theta_{c,2},\cdots, \Theta_{c,j_i},}_{\text{source sharing parameters}} \quad \underbrace{\Theta_{v,1}, \cdots, \Theta_{v,c_i},}_{\text{target common parameters}} \quad \underbrace{\Theta_{t,1}, \cdots, \Theta_{t,d-j_i - c_i})}_{\text{target-specific parameters}} =  (\Theta_{c,i}, \Theta_{v,i}, \Theta_{tr,i}).
\end{align}
\end{small}
Here $\Theta_{c,i}$ and $\Theta_{v,i}$ represent the source sharing parameters and the target common parameters under episode $i$, which are not overlapped. The parameters changing from $\Theta_{tr,i-1}$ to $\Theta_{tr,i}$ exhibit the nature of time-varying target domains. By this particular parameterization and we assume at each episode $i$, the sample sizes $n_i$ are comparable with $n_{i-1}$ (e.g., $n_{i-1} \asymp n_i$), we reach the following theorem.

\begin{theorem}[Time-variant Target Regret Bounds]\label{thm:time-varying} Given the time-variant target domain described in the problem formulation, suppose that conditions in Therorem~\ref{thm:expreg-generalloss} and Assumptions~\ref{asp:para-dist-timevarying} and \ref{asp:para-trans} hold for each $\theta^*_{t,i}$ and $\theta^*_s$ for $i = 1,2,\cdots,k$. We further assume that source parameters will share $j_i$ parameters with every $\theta^*_{t,i}$. In addition, $\theta^*_{t,i}$, $\theta^*_{t,i-1}$ have $c_i$ common parameters. As $n_i\rightarrow \infty$ for any $i$ and assume $n_{i-1} \asymp n_i$ and $m  \asymp n^p_i$ for some $p \geq 0$, the mixture strategy with proper prior $\omega(\theta_s,\theta_{t,i}, \theta_{t,i-1})$ yields
\begin{align}
    \mathcal{R}_{TV} \lesssim &  \sqrt{ l \sum_{i=1}^{l}  n_i \Big( j_i (1 \wedge n_i^{1-p})+ c_i   + (d-c_i-j_i)\log n_i + \frac{2}{\omega(\theta^*_{t,i}|\theta^*_{t,i-1},\theta^*_s)} \Big) }.
\end{align}
\end{theorem}
As can be seen in the theorem, the estimations of the parameters are now four-fold. In the summation, the first term stands for the estimation cost of source sharing parameters and the rate depends on the sample size of source samples. Hence, a large sample size ($p>1$) will contribute to boosting the rate. The second term, concerned with the target common parameters $\Theta_{v,i}$, is a constant $c_i$ calculating from the ratio $\frac{n_i}{n_{i-1}} \sim O(1)$ where it entails that the target data with the previous episode improves the estimation of $\Theta_{v,i}$. For target-specific parameter, the rate is $O((d-j_i-c_i)\log(n_i))$ which coincides with the typical regret growth results. Lastly, the prior knowledge $\omega(\theta^*_{t,i}|\theta^*_{t,i-1},\theta^*_s)$ (e.g., the knowledge over $\theta^*_{t,i}$ given the previous $\theta^*_{t,i-1}$ and $\theta^*_s$) also plays an important role in the prediction performance for each episode $i$ as can be seen in the OTL case.

\subsection{Negative and Positive Transfer}
As previously discussed, $\omega(\Theta_s,\Theta_t)$ should be chosen properly so that the posterior will asymptotically converge to the true parameter $\theta^*_s$ and $\theta^*_t$. However, if the prior distribution (particularly $\omega(\Theta_t|\Theta_s)$) is imposed improperly, the extra source data do not necessarily help improve our prediction for target data. Roughly speaking, if our prior knowledge on $\theta_s^*$ and $\theta_t^*$ is incorrect, under our scheme, this could translate to an improper prior distribution for the mixture strategy. We will show that in the worst case, with an improper prior, the extra source data will in fact cause a higher regret (i.e., worse performance) compared to the case without source data, known as the \emph{negative transfer}. We also study the \emph{positive transfer} case under the proposed mixture strategy where the source data indeed improve the prediction performance.
\subsubsection*{Negative Transfer}
To understand the negative transfer, let us first start with a simple Bernoulli transfer example. Consider $Z_s$, $Z_t$ take values in $\{0,1\}$ and assume $\theta^*_s$ and $\theta^*_t$ are the probabilities that the source and target samples take value in $1$. Also assume that our prior knowledge on the parameters is that $|\theta^*_s-\theta^*_t| \leq 0.1$ given any $\theta^*_s\in \Lambda$. Suppose the underlying parameters are $\theta^*_t = 0.6$ and $\theta^*_s = 0.8$. In other words, our prior knowledge is incorrect. In this case, even with the perfect knowledge of $\theta^*_s$, no algorithm can correctly estimate $\theta^*_t$ if the (incorrect) prior knowledge $|\theta_s-\theta_t| \leq 0.1$ is imposed, even with infinitely many target samples. As a consequence, the expected risk becomes higher compared to the case without knowing such prior. In section~\ref{subsec:Bernoulli}, we present both detailed analytical and experimental analysis for Bernoulli examples with different priors. Generally speaking, it is recognized that learning performance is captured by the divergence between the estimated distribution $P_{\theta_{est}}$ and the true distribution $P_{\theta^*}$.  Firstly, we consider the ITL problem under logarithmic loss and then concretely point out that improper prior will lead to the negative transfer. 
\begin{proposition}[Negative Transfer for Instantaneous Transfer Learning] \label{prop:neg-inst} 
Under Assumptions~\ref{asp:para-dist} and \ref{asp:para-trans}, as $n, m \rightarrow \infty$, the mixture strategy with a proper $\omega(\Theta_s)$ but an improper $\omega(\Theta_t|\Theta_s)$ for logarithmic loss yields
\begin{equation}
\mathcal{R}_I \geq  D_{\textup{KL}}(P_{\theta^*_t}\|P_{\tilde{\theta}_t}),
\end{equation}
where $\tilde{\theta}_t = \min_{\theta_t \in\operatorname{supp}(\omega(\Theta_t|\Theta^*_s))} D_{\textup{KL}}(P_{\theta^*_t}\|P_{{\theta_t}})$.
\end{proposition}
Compared to the result in Theorem~\ref{thm:excessrisk-log}, the excess risk in this case does not converge to zero even when $n$ goes to infinity because it can be shown that $\tilde{\theta}_t$ does not coincide with $\theta^*_t$. In other words, if the wrong prior information about the parameters leads to a choice of an improper prior distribution, the posterior of $\theta_t$ will approach $\tilde{\theta}_t$ instead of the true parameter $\theta^*_s$ and the source data will hurt the performance instead even if we have abundant target data. It can be shown that a similar phenomenon occurs in online predictions, as characterized by the following proposition.
\begin{proposition}[Negative Transfer for Online Transfer Learning]\label{prop:neg-online}
Let $\mathcal{R}^{\omega(\Theta_s,\Theta_t)}_O (n)$ denote the regret induced by the mixture strategy  $Q(D^n_t|D^m_s)$ with the source data and the prior $\omega(\Theta_s,\Theta_t)$. Under logarithmic loss, if $\omega(\Theta_s)$ is proper but $\omega(\Theta_t|\Theta_s)$ is improper\footnote{We say $\omega(\Theta_t|\Theta_s)$ is improper if it does not satisfy conditions in Def~\ref{def:proper-prior}.}, the following inequality holds when both $n$ and $m$ are sufficiently large.
\begin{align}
     \mathcal{R}^{\omega(\Theta_s,\Theta_t)}_O(n) \geq  nD_{\textup{KL}}(P_{\theta^*_t}\|P_{\tilde{\theta}_t}),
\end{align}
where $\tilde{\theta}_t = \argmin_{\theta_t \in\operatorname{supp}(\omega(\Theta_t|\Theta^*_s))} D_{\textup{KL}}(P_{\theta^*_t}\|P_{{\theta_t}})$. Furthermore, for any proper $\hat{\omega}(\Theta_t)$, let $\mathcal{R}^{\hat{\omega}(\Theta_t)}_O (n)$ denote the regret induced by $\hat{Q}(D^n_t)$ with the prior $\hat{\omega}(\Theta_t)$ but without the source data. We have
\begin{align}
    \mathcal{R}^{\omega(\Theta_s,\Theta_t)}_O(n) > \mathcal{R}^{\hat{\omega}(\Theta_t)}_O(n).
\end{align}
\end{proposition}
For online transfer learning, if we misapply the wrong prior information, the rate of regret is $\Omega(n)$ while the regret is of $O(\log n)$ without the source data introduced. Such improper prior leads to a higher regret and the negative transfer is characterized by the parameters that are the closest to the true parameters $\theta^*_t$ in terms of the KL divergence. It immediately follows that the averaged regret $\frac{1}{n}\mathcal{R}^{\omega(\Theta_s,\Theta_t)}_O(n)$ does not asymptotically go to zero. As such, we concretely demonstrate when the negative transfer will happen for both ITL and OTL scenarios and theoretically confirm that the regret is quantitatively captured by the KL divergence between the true distribution and the estimated posterior. 

\subsubsection*{Positive Transfer}
In contrast, if $\omega(\Theta_s,\Theta_t)$ is chosen properly, we can always find a prior such that the knowledge transfer from source data encourages lower regret, leading to the \emph{positive transfer}. Take the Bernoulli case as an example again, without the source, the prior $\hat{\omega}(\Theta_t)$ can be selected as a uniform distribution over the whole probability range $[0,1]$. With the source, if the prior knowledge $\omega(\Theta_s|\Theta_t)$ could encourage a tighter support near $\theta^*_t$, say $[\theta^*_t - c, \theta^*_t + c] \subset [0,1]$, then there will exist some proper prior that leads to a lower regret. It can be interpreted that the source data narrow down the uncertainty on the choice of $\theta^*_t$ and it is more likely to obtain a more accurate estimation. To mathematically interpret the positive transfer, we firstly consider the OTL case under logarithmic loss and establish the following proposition. 
\begin{proposition}[Positive Transfer for OTL]\label{claim:positive}
When both $n$ and $m$ are sufficiently large, for any proper $\hat{\omega}(\Theta_t)$, we can always find a proper prior $\omega(\Theta_s,\Theta_t)$ satisfying the following inequality if the support of $\omega(\Theta_t|\Theta_s)$ is a proper subset of $\Lambda$ for any $\|\theta_s - \theta^*_s\| \leq \delta_s$ with some $\delta_s > 0$.
\begin{align}
    \mathcal{R}^{\omega(\Theta_s,\Theta_t)}_O(n) < \mathcal{R}^{\hat{\omega}(\Theta_t)}_O(n)
\end{align}
\end{proposition}
In our claim, if $\omega(\Theta_t|\Theta_s)$ encourages a tighter support over $\Theta_t$, making use of source data appropriately can narrow down the uncertainty range over $\Theta_t$. It then follows that there will always exist a prior that assigns more concentrated mass around $\theta^*_t$, which reduces the expected regret. 

  For the ITL scenario, from Theorem~\ref{thm:inst-scalar} one can see that the prior knowledge is not revealed in characterizing the excess risk since the data overwhelms the prior as both $n$ and $m$ goes sufficiently large. However, the prior knowledge still plays an important role when the sample sizes are limited. To see this, by generalizing Theorem 1 from \cite{haussler1995general}, for any $n$ and $m$, the excess risk of ITL in~(\ref{eq:excess-risk}) can be upper bounded by
\begin{align}
    \mathcal{R}_I \leq -\log \int P(\theta_t|D^n_t)\frac{\omega(\theta_t|\theta_s)}{\hat{\omega}(\theta_t)} e^{-n D_{KL}(P_{\theta^*_t}\| P_{\theta_t})}d\theta_t P(\theta_s|D^m_s) d\theta_s,
\end{align}
where $P(\theta_t|D^n_t) = \frac{P_{\theta_t}(D^n_t)\hat{\omega}(\theta_t)}{P(D^n_t)}$ denotes the posterior induced by $\hat{\omega}(\theta_t)$. Here we allow $m$ goes to infinity (abundant source data) but let $n$ be a constant (limited target data). With the proper $\omega(\Theta_s)$, the posterior is approximated as $P(\theta_s|D^m_s) = \delta(\theta^*_s)$, then we can rewrite the inequality for any $n \in \mathbb{N}$, 
\begin{align}
    \mathcal{R}_I \leq -\log \int P(\theta_t|D^n_t)\frac{\omega(\theta_t|\theta^*_s)}{\hat{\omega}(\theta_t)} e^{-n D_{KL}(P_{\theta^*_t}\| P_{\theta_t})}d\theta_t. \label{eq:single_pos}
\end{align}
Compared to the results without the source, we have the bound
\begin{align}
    \mathcal{R}_I \leq -\log \int P(\theta_t|D^n_t) e^{-n D_{KL}(P_{\theta^*_t}\| P_{\theta_t})}d\theta_t. \label{eq:single_pos2}
\end{align}
It is observed that the ratio $\frac{\omega(\theta_t|\theta^*_s)}{\hat{\omega}(\theta_t)}$ still plays a role in this non-asymptotic bound. Informally speaking, to achieve positive transfer in the single prediction, one may select $\omega(\theta_t|\theta^*_s)$ to be large for those $\theta_t$ that leads to low divergence $D_{KL}(P_{\theta^*_t}\| P_{\theta_t})$ and vice versa. As a consequence, the RHS in~(\ref{eq:single_pos}) will be smaller than the RHS in~(\ref{eq:single_pos2}). However, due to the fact that the posterior distribution $P(\theta_t|D^n_s)$ cannot be well estimated with finite $n$, it is relatively difficult to tell whether the positive transfer will happen given a specific prior. Rather, we intuitively explain the role of prior knowledge from its non-asymptotic upper bound. 

\subsection{Bernoulli Example}\label{subsec:Bernoulli}
In this section, we illustrate the results presented for OTL case using a simple Bernoulli example, the results can be calculated using the same procedures for ITL and TVTL cases. We assume the parametric distributions are $P_{\theta^*_s} \sim \text{Ber}(\theta^*_s)$ and $P_{\theta^*_t} \sim \text{Ber}(\theta^*_t)$ for $\theta^*_s, \theta^*_t \in [0,1]$, that is, $P(Z^{(k)}_t = 1) = \theta^*_t$ and $P(Z^{(k)}_s = 1) = \theta^*_s$ and we also assume $\theta^*_s \neq \theta^*_t$. Given a batch of source data $D^m_s$ with $m$ samples i.i.d. drawn from $P_{\theta^*_s}$ and $n$ target samples $D^n_t$ i.i.d. drawn from $P_{\theta^*_t}$ sequentially, we will make predictions for each target sample $Z^{(k)}_t$ at time $k$ based upon the received target data $D^{k-1}_t$ and source data $D^m_s$. Under the logarithmic loss $\ell$,  let $\ell(b_k, Z^{(k)}_t) = Q(Z^{(k)}_t|D^{k-1}_t,D^m_s)$ and $\ell(b^*_k, Z^{(k)}_t)= P_{\theta^*_t}(Z^{(k)}_t)$, we are interested in the expected regret $\mathcal{R}_O$ as:
\begin{align}
\mathcal{R}_O &= \mathbb{E}_{\theta^*_s,\theta^*_t} \left[\sum_{k=1}^{n} \ell(b_{k}, Z^{(k)}_t) - \sum_{k=1}^{n} \ell(b_k^*, Z^{(k)}_t) \right] \\
&=\sum_{D^m_s}\sum_{D^n_t} P_{\theta^*_s}\left(D^m_s \right) P_{\theta^*_t}\left(D^n_t \right) \log \left(\frac{P_{\theta^*_t}\left(D^n_t \right)}{Q\left(D^n_t| D^m_s\right)}\right).
\end{align}

\subsubsection*{Without Prior Knowledge}
In this section, we analytically show that without any prior knowledge on the relationships between the source and target data (apart from the assumption $\theta^*_s \neq \theta^*_t$), the transfer learning algorithms lead to negative transfer in the worst case. To see this, by assigning a prior distribution $\omega(\Theta_t,\Theta_s)$ over $[0,1]^2$, we firstly define the distribution $Q$ with the mixture strategy as
\begin{align}
Q\left(D^n_t| D^m_s\right) &=  \frac{\int_{0}^{1} \int_0^1 \omega(\theta_s,\theta_t) P_{\theta_t}(D^n_t)P_{\theta_s}(D^m_s)  d\theta_sd\theta_t }{\int_{0}^{1}  \omega(\theta_s) P_{\theta_s}(D^m_s) d\theta_s},
\end{align}
on account of the assumption $P_{\theta_s,\theta_t}(D^k_t,D^m_s) = P_{\theta_t}(D^k_t)P_{\theta_s}(D^m_s)$. Without any prior knowledge, we will arbitrarily choose two types of priors $\omega(\Theta_s,\Theta_t)$ and analytically examine their regrets. We start by looking at the case where the joint prior distribution is assumed to be $\omega(\Theta_s, \Theta_t) = \Theta_s + \Theta_t$, which gives the marginal distributions $\omega(\Theta_s) = \Theta_s + \frac{1}{2}$ and $\omega(\Theta_t) = \Theta_t + \frac{1}{2}$. By knowing $\theta_s$, there is still some uncertainty over $\Theta_t$. It can be seen from the conditional distribution
\begin{align}
  \omega(\Theta_t|\theta_s) = \frac{\theta_s}{\theta_s+\frac{1}{2}} + \frac{\Theta_t}{\theta_s + \frac{1}{2}}.
\end{align}
The predictor distribution $Q(D^n_t|D^m_s)$ can be calculated explicitly as
\begin{align}
Q\left(D^n_t| D^m_s\right) &= \frac{1}{(n+1)}\frac{1}{\left(^{n}_{{k_t}}\right)} \frac{ 2k_s + 2 + (k_t+1) \frac{2(m+2)}{n+2} }{m+2k_s + 4}, 
\end{align}
where we denote number of 1's received from the source and target by $k_t$ and $k_s$, respectively. If $m, n$ are sufficiently large and $\mathbb{E}_{\theta^*_s}[k_s] = \theta^*_s m $ and $\mathbb{E}_{\theta^*_ts}[k_t] = \theta^*_t n $, we expect over a long sequence that the conditional mutual information can be calculated as
%
\begin{align}
    \mathcal{R}^{{\omega}(\Theta_t,\Theta_s)}_O(n) = & \frac{1}{2}\log(n+1) + \frac{1}{2}\log\frac{1}{\pi \theta^*_t (1 - \theta^*_t)}  + \log \frac{\theta^*_s + \frac{1}{2}}{\theta^*_t + \theta^*_s}, \label{eq:bernoulli-with-weak}
\end{align}
Compared to the mixture distribution $Q$ by the marginal $\hat{\omega}(\theta_t) = \theta_t + \frac{1}{2}$ without introducing the source
\begin{align}
Q\left(D^n_t\right) &= \int^1_0 (\theta_t+\frac{1}{2}) (\theta_t)^{k_t} (1-\theta_t)^{n-k_t}d\theta_t =\frac{1}{(n+1)}\frac{1}{\left(^{n}_{{k_t}}\right)} (\frac{1}{2} + \frac{k_t+1}{n+2} ),
\end{align}
and the expected regret induced by this predictor can be calculated as
\begin{align}
    \mathcal{R}^{\hat{\omega}(\Theta_t)}_O(n)  &= \frac{1}{2}\log(n+1) + \frac{1}{2}\log \frac{1}{\pi \theta^*_t (1 - \theta^*_t)} + \log \frac{1}{\theta^*_t + \frac{1}{2}}. \label{eq:bernoulli-without-weak}
\end{align}
The difference of the regrets in Equation~(\ref{eq:bernoulli-with-weak}) and (\ref{eq:bernoulli-without-weak}) is $\log \frac{\omega(\theta_t)}{\omega(\theta_t|\theta_s)} = \log \frac{(\theta^*_s + \frac{1}{2})(\theta^*_s + \frac{1}{2})}{\theta^*_t + \theta^*_s}$. From this example, due to the fact that we do not have any specific prior knowledge over $\theta^*_t$ and $\theta^*_s$, the prior distribution $\omega(\theta_t|\theta_s)$ may not lead to a better estimation compared to the case without the source data, and the expected regret with the source can larger (e.g., if $(\theta^*_s - \frac{1}{2})(\theta^*_t - \frac{1}{2}) > 0$) in the worst case and the negative transfer will happen. Additionally, even when $m$ goes to infinity and the source data does not change the convergence rate w.r.t. $n$, this result confirms the Theorem~\ref{thm:consistency-scalar} numerically. 

If we consider another extreme case where the prior distribution $\omega(\Theta_t|\Theta_s) = \delta(\Theta_s)$ and $\omega(\Theta_s) = 1$, that is, $\Theta_s$ is uniformly distributed and knowing $\Theta_s$ is equivalent to knowing $\Theta_t$.  Then we obtain the mixture as
\begin{align}
Q\left(D^n_t| D^m_s\right)  = \frac{\frac{1}{n+m+1}\frac{1}{\left(^{m+n}_{k_s+k_t}\right)}}{\frac{1}{m+1}\frac{1}{\left(^m_{k_s}\right)}}.
\end{align}
Analogously we expect over a large $m$, where $k_s = \theta^*_s m$ and $k_t = \theta^*_t n$ and the source samples are abundant, e.g., $m \gg n $. We have
\begin{align}
&  \mathcal{R}^{{\omega}(\Theta_t,\Theta_s)}_O(n) \leq \log(1+\frac{n}{m+1})  + n\left( \theta^*_t \log \frac{\theta^*_t}{\theta^*_s} + (1-\theta^*_t) \log \frac{1 - \theta^*_t}{1 - \theta^*_s} \right) + C,
\end{align}
where $C$ is some constant that depends on $\theta^*_t$ and $\theta^*_s$. By introducing abundant source data, $\log(1+\frac{n}{m+1})$ term will vanish with the rate $O(\frac{n}{m})$ but it will also introduce the term $n\left( \theta^*_t \log \frac{\theta^*_t}{\theta^*_s} + (1-\theta^*_t) \log \frac{1 - \theta^*_t}{1 - \theta^*_s} \right)$ that grows linearly in $n$. This specific choice of prior is \emph{improper} since $\omega(\theta^*_t | \theta^*_s) = 0$ whereas $\theta^*_t \neq \theta^*_s$, which will finally lead to the inaccurate estimation of $\theta^*_t$ while both $m$ and $n$ are sufficiently large. This result is unsurprising since we can estimate $\theta^*_s$ accurately and $\omega(\theta_t |\theta^*_s)$ enforces $\theta_t = \theta^*_s$ and the regret for each sample is at least $D(P_{\theta^*_t}\|P_{\theta^*_s})$. We can confirm it by calculating the true regrets precisely as $m$ is sufficiently large,
\begin{equation}
     \mathcal{R}^{{\omega}(\Theta_t,\Theta_s)}_O(n) =  n\left( \theta^*_t \log \frac{\theta^*_t}{\theta^*_s} + (1-\theta^*_t) \log \frac{1 - \theta^*_t}{1 - \theta^*_s} \right) = n D_{\textup{KL}}(P_{\theta^*_t}\|P_{\theta^*_s}).
\end{equation}
Note that the chosen prior will lead to the negative transfer, and the rate of the expected regret in this case is $O(n)$ as we proved in Proposition~\ref{prop:neg-online}. 

\subsubsection*{With Prior Knowledge}
In the previous section, we have demonstrated two choices of prior distributions for the mixture strategy, both of which lead to negative transfer in the worst case. Now we show that with some correct prior knowledge, we can always choose an appropriate $\omega$ for better predictions. Noticing that the expected regret is captured by the conditional prior $\omega(\Theta_t|\Theta_s)$, even though we do not have access to the true parameters, we may know some relationship between the source and target parameters. In particular, for $\Lambda = [0,1] \subset \mathbb{R}$, we make the following assumption.
\begin{assumption}[Prior knowledge with $\ell_1$ norm] \label{asp:l1}
For $\theta^*_s, \theta^*_t \in [0,1]$ and $c>0$, 
\begin{equation}
|\theta^*_s - \theta^*_t| \leq c.
\end{equation}
\end{assumption}
This assumption implies that $\theta^*_t$ is not far away from $\theta^*_s$, and if $\theta^*_s$ can be approximated accurately, $\theta^*_t$ can be estimated more precisely with a tighter support. We encode this particular relationship in terms of the conditional prior distribution $\omega(\Theta_t|\theta_s)$, say, given any $\theta_s$, $\Theta_t$ is uniformly distributed over $[\theta_s - c, \theta_s + c]$ with density $\frac{1}{2c}$ and the hyperparameter $c$ can be interpreted as our knowledge level. Larger $c$ indicates that we are more uncertain about $\theta^*_t$ given $\theta^*_s$ and vice versa. Additionally, we assume the marginal $\omega(\theta_s)$ is proper, e.g., $\theta^*_s$ can be estimated accurately with sufficient source samples. As a result, we can give the explicit expression of the mixture distribution $Q(D^n_t|D^m_s)$,
\begin{lemma}\label{lemma:bernoulli}
Under Assumption~\ref{asp:l1}, given any $\theta_s$, $\Theta_t$ is uniformly distributed over $[\theta_s - c, \theta_s + c]$ with density $\frac{1}{2c}$, and $\theta^*_s$ can be estimated accurately via the source samples $D^m_s$. We have,
\begin{align}
Q(D^n_t|D^m_s) =& \frac{1}{2c}\left(\begin{array}{l}
n \\
k_t
\end{array}\right)^{-1} \Bigg( \sum^{k_t}_{i=1} \left(\begin{array}{l}
n \\
i
\end{array}\right)\frac{(\theta^*_s - c)^{i}(1- \theta^*_s + c)^{n-i+1}}{n-i+1} - \frac{ (\theta^*_s + c)^{i}(1- \theta^*_s - c)^{n-i+1}}{n-i+1} \nonumber \\
& + \frac{(1-\theta^*_s + c)^{n+1}-(1- \theta^*_s - c)^{n+1}}{n+1} \Bigg).
\end{align}
\end{lemma}

\begin{remark}
It is relatively hard to directly tell the effect of $c$ from the above expression. Since we can calculate $Q(D^n_t|D^m_s)$ explicitly if all the parameters $\theta^*_s, c, n, k_t$ are known, we shall verify the intuition by conducting the experiments with different parameters, which are presented later.
\end{remark}

To examine the effect of hyperparameter $c$ under Assumption~\ref{asp:l1}, numerical experiments are conducted and results are presented. Consider the following specific settings, let the true parameter $\theta^*_t = \frac{1}{3}$. Under Assumption~\ref{asp:l1}, we assume $\omega(\Theta_s) = 1$ over $[0,1]$ and given $\theta_s$, $\Theta_t$ is distributed uniformly over $[\theta_s - c, \theta_s + c]$. As the target arrives sequentially and we have sufficient source data with $m = 100000$, we plot the predictive probability $P(Z^{(k)}_t|D^{k-1}_t,D^m_s)$ in Figure~\ref{subfigure:predprob1} and \ref{subfigure:predprob2} for a single trial with different $\theta^*_s$ and $c$.

\begin{figure}[h!]
\centering
\subfigure[$\theta^*_s = 0.4$, $c = 0.1$\label{subfigure:predprob1}]{\includegraphics[width = 1.8in]{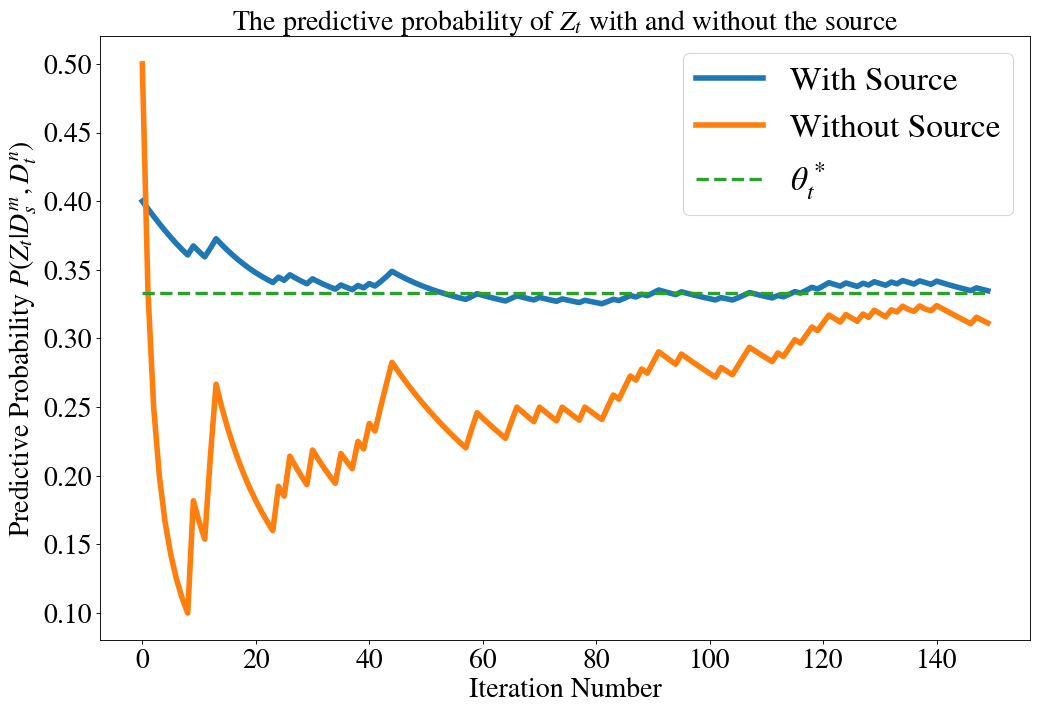}} 
\subfigure[$\theta^*_s = 0.35$, $c = 0.1$\label{subfigure:predprob2}]{\includegraphics[width = 1.8in]{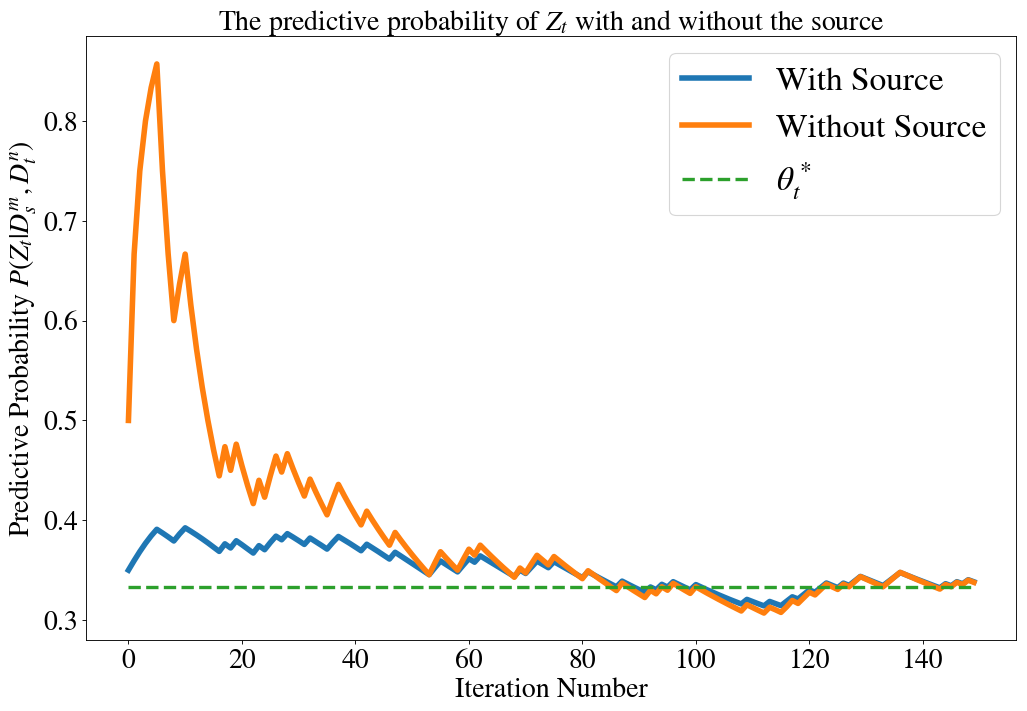}} \\
\subfigure[$\theta^*_s = 0.4$, $c = 0.1$\label{subfigure:postprob1}]{\includegraphics[width = 1.8in]{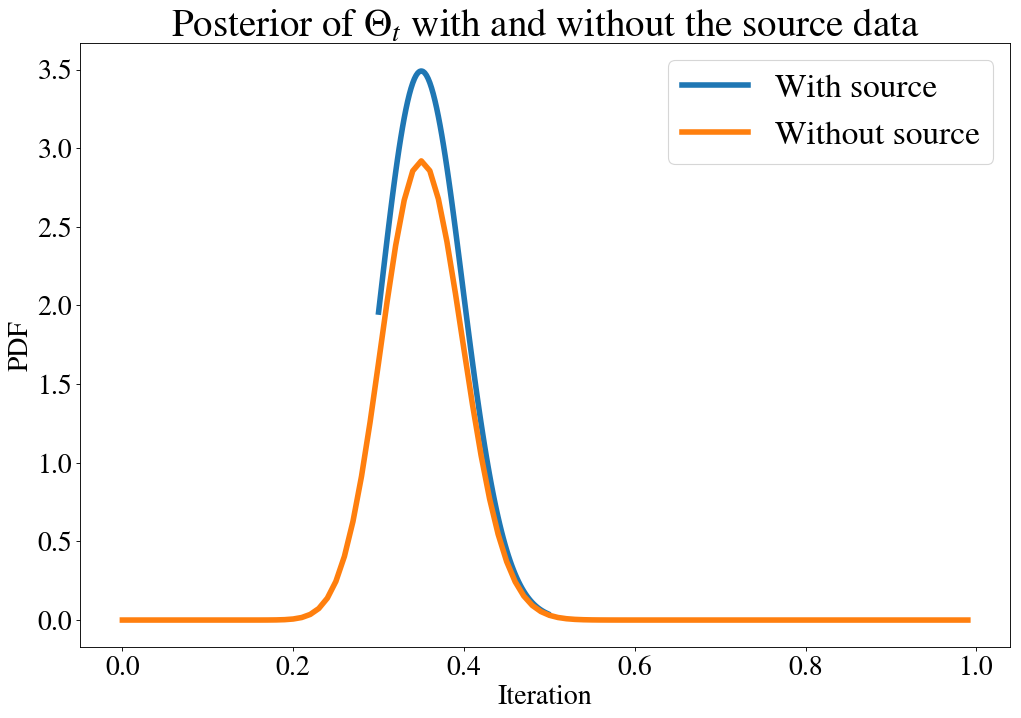}}
\subfigure[$\theta^*_s = 0.35$, $c = 0.05$\label{subfigure:postprob2}]{\includegraphics[width = 1.8in]{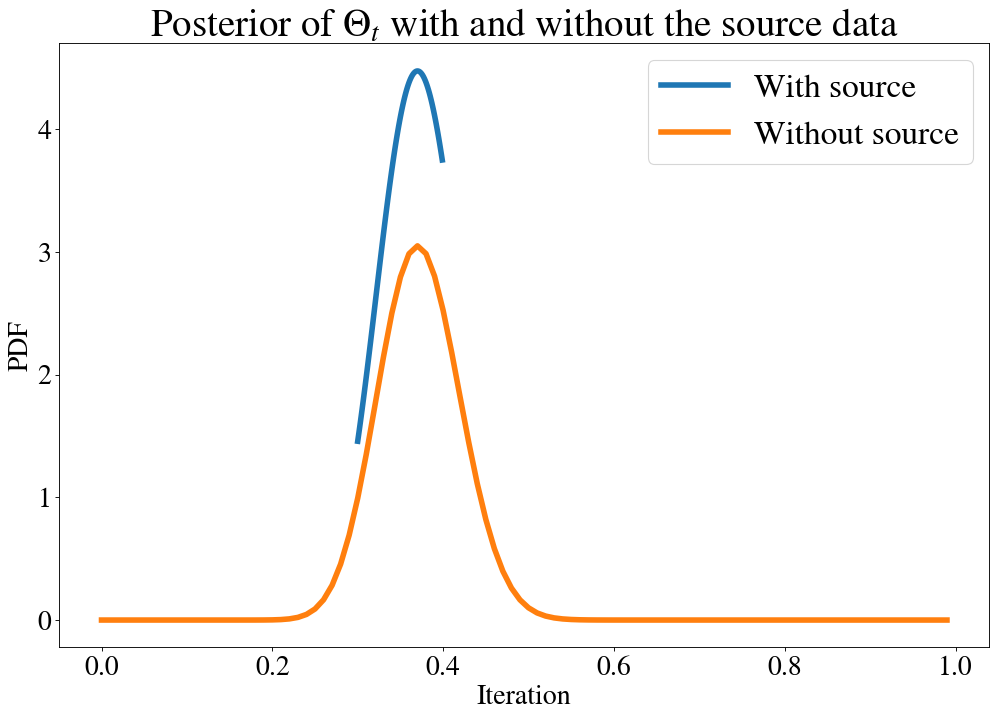}}\\
\caption{With different $c$ and $\theta^*_s$, subfigures (a) and (b) show the predictive probability $P(z^{(k)}_t|D^{k-1}_t,D^m_s)$ with single trial, the prediction the curves that are closer to $\theta^*_t =\frac{1}{3}$ entails more accurate estimation. Subfigures (c) and (d) show the probability density of the posterior $P(\Theta_t|D^n_t,D^m_s)$ (Blue) and $P(\Theta_t|D^n_t)$ (Orange) after receiving 150 target samples. The posterior with the closer $\theta^*_s$ and smaller $c$ leads to higher density around $\theta^*_t$. }\label{fig:posterior_bernoulli}
\end{figure}

From Figures~\ref{subfigure:predprob1} and \ref{subfigure:predprob2}, we observe that when $k$ is relatively small, the prior knowledge about the source and the target distribution can help estimate the true value of $\theta_t$ better than without introducing the source data. It can also be seen from  Figure \ref{subfigure:predprob2}, when $k$ is relatively large, both estimates of $\theta^*_t$ with or without the source are fairly close. Since the single trial does not fully reflect the usefulness of the source data, we next examine the posterior distribution $P(\Theta_t | D^n_t, D^m_s)$ with source and $P(\Theta_t|D^n_t)$ without source after receiving 150 target data with different $\theta^*_s$ and $c$ to see the effect of prior knowledge as shown in Figure~\ref{subfigure:postprob1} and \ref{subfigure:postprob2}. From the comparisons, the posterior distribution of $\Theta_t$ with the source data is more concentrated. Moreover, closer $\theta^*_s$ and smaller $c$ will yield more concentrated density around $\theta^*_t$, which fits in line with our intuition. 

\begin{figure}[h!]
\centering
\subfigure[$\Theta_s = 0.35$, $c = 0.05$\label{subfigure:reg3}]{\includegraphics[width = 1.8in]{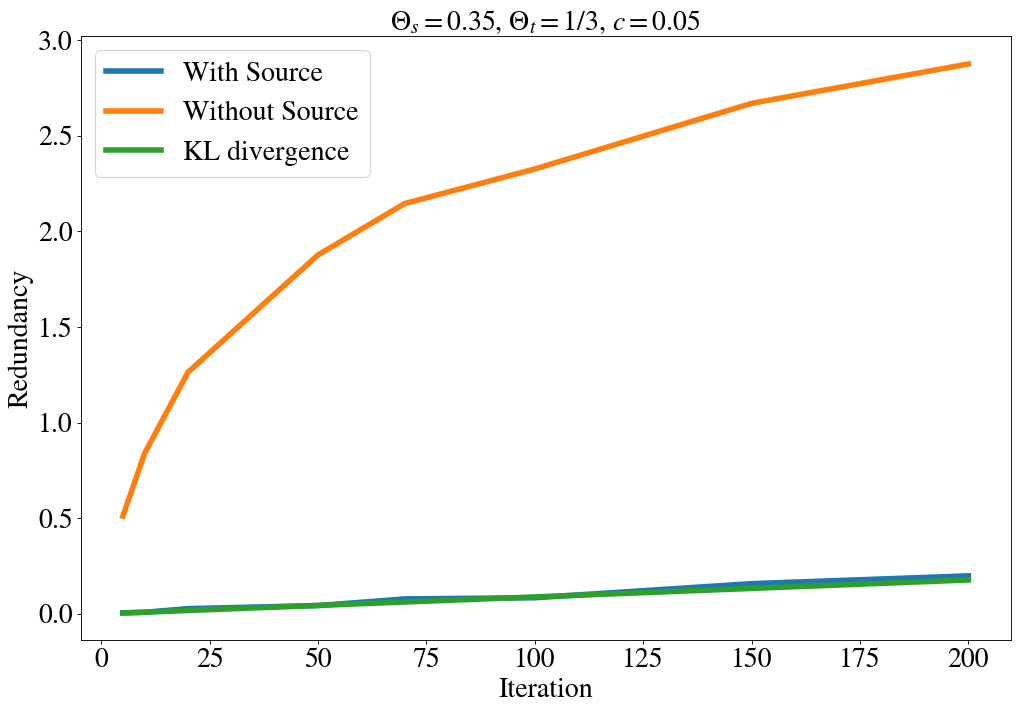}}
\subfigure[$\Theta_s = 0.35$, $c = 0.1$\label{subfigure:reg2}]{\includegraphics[width = 1.8in]{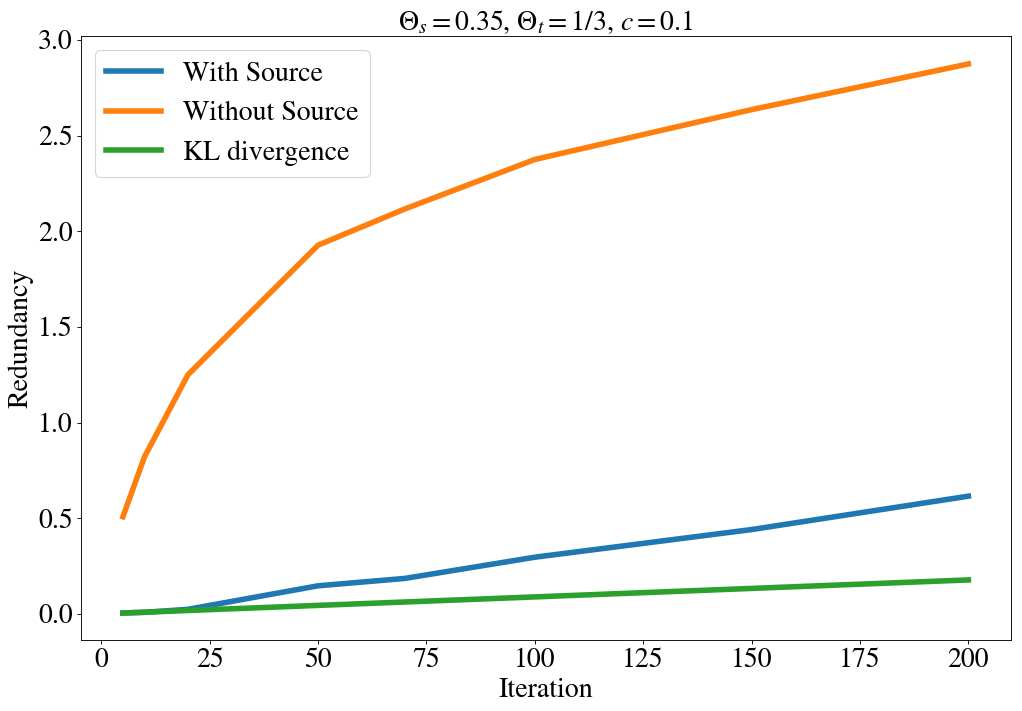}} \\
\subfigure[$\Theta_s = 0.4$, $c = 0.1$\label{subfigure:reg1}]{\includegraphics[width = 1.8in]{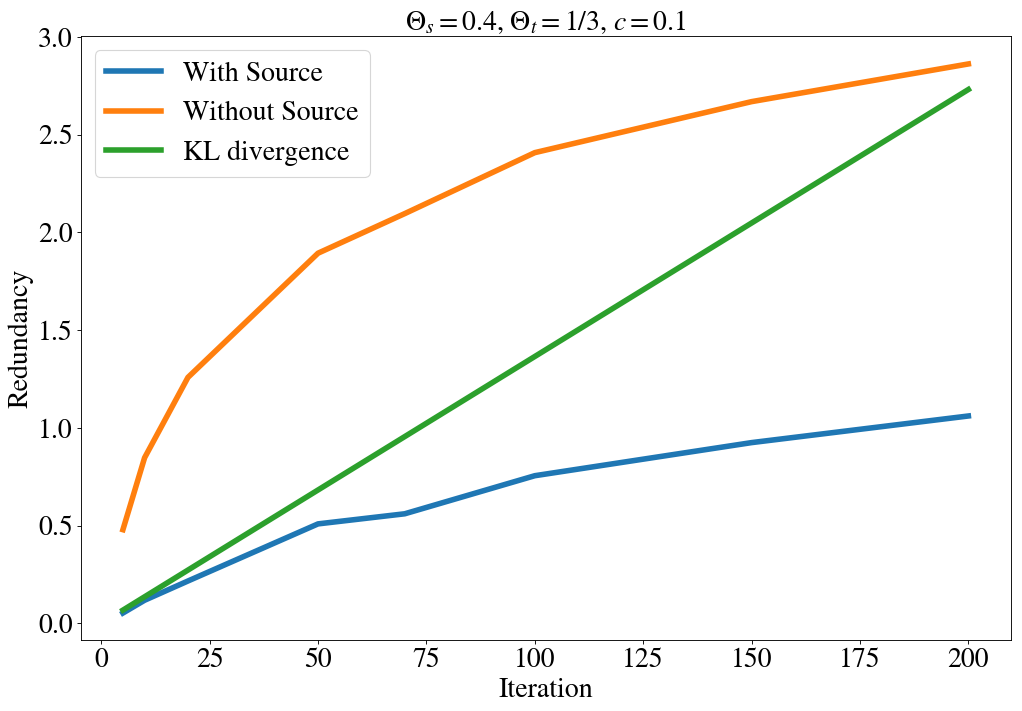}}
\subfigure[$\Theta_s = 0.4$, $c = 0.0001$\label{subfigure:reg4}]{\includegraphics[width = 1.8in]{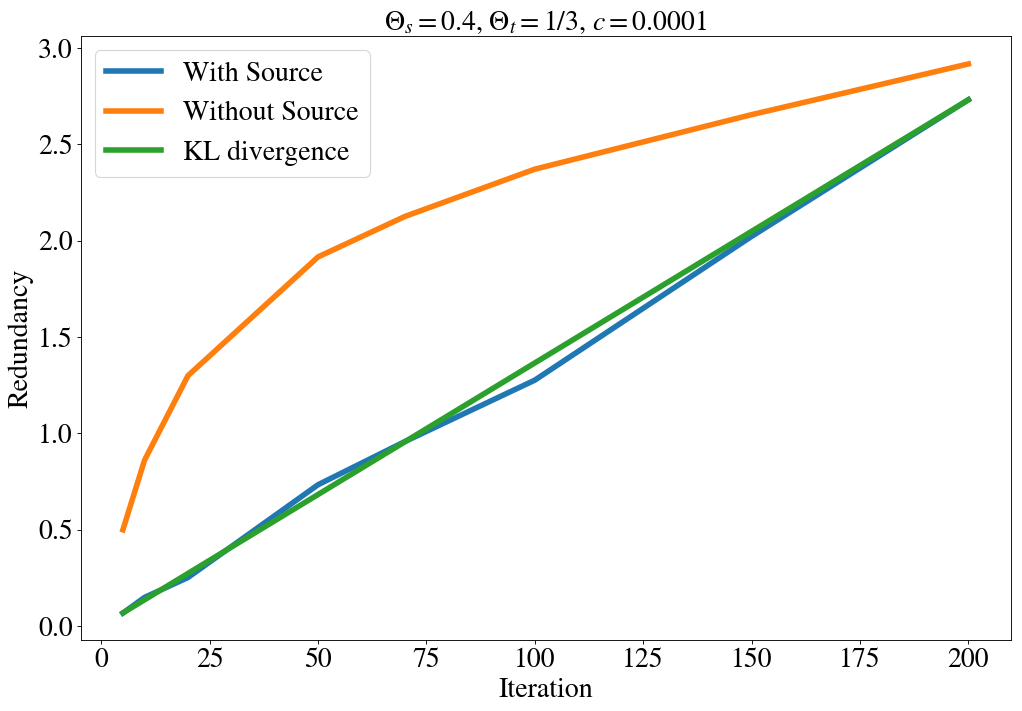}}
\caption{Expected regrets by 2000 repeats with different $c$ and $\theta^*_s$. Orange, blue and green curves represents the expected regrets without the source, with the source, and $kD_{\textup{KL}}(P_{\theta^*_t}\| P_{\theta^*_s})$, respectively.}\label{fig:regret_bernoulli}
\end{figure}

To evaluate the expected regrets, we repeat the experiments 2000 times and take the average for different number of target samples. The results are shown in Figure~\ref{fig:regret_bernoulli}. The regret curves reflect the influence of the knowledge level $c$ and the gaps between the true parameter $\theta^*_s$ and $\theta^*_t$. When $\omega(\theta_t|\theta^*_s)$ is proper, e.g., the density is concentrated around $\theta^*_t$, then smaller $c$ will yield lower regrets, which can be seen from Figure~\ref{subfigure:reg3} and~\ref{subfigure:reg2}. However, if the conditional prior is improper (e.g., $[\theta^*_s - c, \theta^*_s + c]$ does not cover $\theta^*_t$), the regrets are determined by both $c$ and the distance $|\theta^*_s - \theta^*_c|$. For example, compare the case $\theta^*_s = 0.4, c = 0.1$ (Figure~\ref{subfigure:reg1}) with $\theta^*_s = 0.4, c= 0.0001$ (Figure~\ref{subfigure:reg4}), the former case is the proper, while the latter does not cover the true $\theta^*_t$, so the worse regrets. In addition, if $c$ is small enough ($c = 0.001$), the estimation of $\theta_t$ will be centered at $\theta^*_s$ thus the regrets will coincide with the KL divergence $kD(P_{\theta^*_t}\| P_{\theta^*_s})$, which confirms the negative transfer case as discussed before. Overall, once the prior information $\omega(\theta_t|\theta_s)$ is located around $\theta^*_t$ and the target samples are inadequate to make accurate prediction, the knowledge transfer is sensible and indeed the regret can be further optimized.

\section{Algorithms and Experiments}\label{Sec5}
The mixture strategy requires the calculation of the posterior of $\Theta_s$ and $\Theta_t$ given the source and target data. This computation has a high complexity in general, especially when $\Lambda$ lies in a high dimensional space. Since our main aim is to estimate the true points $\theta^*_t$ and $\theta^*_s$, we may not need the full posterior. To this end, we propose an efficient algorithm for approximating the underlying parameters, which will be called the efficient mixture posterior updating (EMPU) algorithm. Furthermore, we  propose an algorithm which works without the parametric assumption on the statistical model. This algorithm, based on Dirichtlet process,  can be seen as a nonparametric version of the mixture strategy. We exhibit the experimental results of these two algorithms for both synthetic logistic regression transfer problems and real-world learning problems. We show that our methods achieve both computational efficiency and high performance.

\subsection{Parametric Model}
In this subsection, we introduce an efficient algorithm to approximate the mixture strategy in parametric models. To illustrate, we consider the OTL case as an example, which can be extended to ITL and TVTL cases straightforwardly. From an algorithmic perspective, the mixture strategy introduced in section~\ref{subsec:ITC} can be rewritten as the following steps in Algorithm~\ref{proc:ms}. However, estimating the posterior of $\Theta_t$ is difficult in practice (line \ref{line:initialize} and \ref{line:postupdate}), especially for high-dimensional problems.

\begin{algorithm}[h!]
\SetAlgoLined
\SetKwInOut{Input}{input}
\SetKwInOut{Output}{output}
\Input{$\mathcal{D}^m_S$, loss function $\ell$, prior knowledge over $\theta_t$ and $\theta_s$}
 Encode the prior knowledge as $\omega(\Theta_t,\Theta_s)$\;
 Calculate the posterior $\theta_s$ from $\mathcal{D}^m_S$, i.e., $P(\theta_s|D^m_s)$\;
 Initialize target dataset $D^{k-1}_t = []$\;
 Initialize the prior $P(\Theta_t|D^m_s)$ in Equation~(\ref{eq:mixture-online}) with $\omega(\Theta_t,\Theta_s)$ and $P(\Theta_s|D^m_s)$ \label{line:initialize}\;
 \For{$k = 1,\cdots,T$}{ 
 Receive target sample $Z^{(k)}_t$\;
 Make prediction for $Z^{(k)}_t$ with the posterior $P(\Theta_t|D^{k-1}_t,D^m_s)$ under loss $\ell$\;
 Add $Z^{(k)}_t$ to $D^{k-1}_t$\;
 Update the posterior $P(\Theta_t|D^k_t,D^m_s)$ \label{line:postupdate}\;
 }
 \Output{Sequential predictions for $Z^{(k)}_t$}
 \caption{Mixture Strategy in OTL}
 \label{proc:ms}
\end{algorithm}

 Concerning the computational infeasibility issues, we propose an algorithm for an efficient posterior updating algorithm to make the mixture strategy amenable to a faster implementation. The algorithm, Efficient Mixture Posterior Updating (EMPU) algorithm, is illustrated in Algorithm~\ref{alg:EMPU}. Since it is challenging to directly estimate the full posterior as the sample increases, we discretize the support of $\Theta_t$ and propose a gradient-based mixture strategy for efficiently estimating the posterior with the given prior knowledge. Precisely, we first learn the distribution parameter $\hat{\theta}_s$ from the $\mathcal{D}^m_S$ by statistical methods such as maximum a posterior or maximum likelihood estimation (line~\ref{line:empu1}). Then we quantize the support of the parameter space $\Lambda$ into $N$ points with the prior knowledge $\omega(\Theta_t|\Theta_s)$ for posterior approximation. For example, we will sample $\theta_{t,i}, i = 1,2,\cdots,N$ from $\Lambda$ according to the prior distribution $\omega(\Theta_t|\hat{\theta}_s)$ given the learned $\hat{\theta}_s$  (line~\ref{line:empu2}). Each $\theta_{t,i}$ will have a corresponding weight $\omega_i$ initialized by the conditional prior $\omega(\Theta_t|\hat{\theta}_s)$  (line~\ref{line:empu3}). When we receive a new target sample $Z_t$, we will update $\theta_i$ based on the gradient descent with the step size $\eta$ and loss function $\ell$ and also update $\omega_i$ using the exponential weighting strategy (line~\ref{line:empu4} and \ref{line:empu5}). After receiving $n$ target samples, the updated $\theta_{t,i}$ and $\omega_i$ will be used for future predictions (depending on the tasks). 

\begin{algorithm}[h!]
\SetAlgoLined
\SetKwInOut{Input}{input}
\SetKwInOut{Output}{output}
\Input{$\mathcal{D}^m_S$, quantization number $N$, loss function $\ell$, prior knowledge over $\theta_t$ and $\theta_s$}
 Encode prior knowledge as distribution $\omega(\Theta_t,\Theta_s)$\;
 Estimate $\hat{\theta}_s$ from $\mathcal{D}^m_S$ \label{line:empu1}\; 
 Randomly sample $\theta_{t,i}$ from $\omega(\theta_t|\hat{\theta}_s)$ for $i = 1,2,\cdots,N$ \label{line:empu2}\;
 Initialize distribution $\omega_i$ with the prior knowledge  \label{line:empu3}\;
 \For{$k = 1,\cdots,T$}{ 
 Receive target sample $Z^{(k)}_t$\;
 \For{$i = 1,\cdots, N$}{
 Update $\theta_{t,i}$ by $\theta_{t,i}(k+1) = \operatorname{Proj}(\theta_{t,i}(k) - \eta \nabla_{\theta_t}\ell(\theta_{t,i}(k),Z^{(k)}_t))$ \label{line:empu4} \;
 Update $\omega_i$ by $\omega_i(k+1) = \omega_i(k) e^{-\ell(\theta_{t,i}(k+1), Z^{(k)}_t)}$ \label{line:empu5}\;
 Normalize $\omega_i(k+1) = \frac{\omega_i(k+1)}{\sum_{i=1}^{N}\omega_i(k+1)}$
 }
 }
 \Output{$\omega_i$ and $\theta_{t,i}$}
 \caption{Efficient Mixture Posterior Updating in OTL}
 \label{alg:EMPU}
\end{algorithm}

\begin{remark}
Instead of applying the Bayes rule, we use the gradient descent based method to update $\theta_{t,i}$ and apply the exponential weighting strategy for the weights $\omega_i$. It should be noted that EMPU is an approximation of the posterior distributions, and hence the theoretical analysis in Theorem~\ref{thm:consistency-scalar} and \ref{theorem:gene-para} does not apply to this algorithm. However, numerical results show that the learning performance by EMPU is similar as the mixture strategy suggested in Algorithm~\ref{proc:ms}.  Notice that in contrast to evaluating the full posterior, the learning performance will depend on both the hyperparameters, e.g., the quantization number $N$ and the step size $\eta$.
\end{remark}

\subsection{Logistic Regression Example} \label{subsection:logistic}
 We consider a logistic regression problem in a 2-dimensional space to compare Algorithm~\ref{proc:ms} the EMPU algorithm (Algorithm~\ref{alg:EMPU}) in the OTL scenario. For the given parameter $\theta \in [0,1]^2$ and $Z_i = (X_i,Y_i) \in \mathbb{R}^{2} \times \{0,1\}$, each label $Y_i \in \{0, 1 \}$ is generated from the Bernoulli distribution with probability $p(Y_i = 1|X_i) = \frac{1}{1+e^{-\theta^TX_i}}$. Suppose that the source and target input features $X^{(k)}_s$ and $X^{(k)}_t$ are drawn from the same normal distribution $\mathcal{N}(\begin{bmatrix}
5  \\
-5
\end{bmatrix} ,\begin{bmatrix}
1 & 0\\
0 & 1
\end{bmatrix})$. The loss function is then given by
\begin{align}
    \ell(\theta,Z_i) := -(Y_i\log (\sigma(\theta^TX_i)) + (1-Y_i)\log (1 - \sigma(\theta^TX_i))), \label{eq:cross_entropy}
\end{align}
where $\sigma(x) = \frac{1}{1+e^{-x}}$. Let $\theta^*_t = (0.3,0.5)$ and $\theta^*_s = (0.2,0.4)$ denote the true parameters for the target and source domains. Given $m = 5000$, let the marginal prior $\omega(\Theta_s)$ be uniformly distributed over $[0,1]^2$ and our prior knowledge $\omega(\Theta_t|\Theta_s)$ assumes that $\Theta_t$ is normally distributed with the mean of $\Theta_s$ and covariance of $\begin{bmatrix}
c^2 & 0\\
0 & c^2
\end{bmatrix}$, here $c$ represents the prior belief on $\Theta_t$ such that smaller $c$ implies $\Theta_t$ is closer to $\Theta_s$ and vice versa. To show the usefulness of the source data, we compare with the target only case ($m = 0$) where we assume the prior $\hat{\omega}(\Theta_t)$ is uniformly distributed over $[0,1]^2$. 

\subsubsection*{Full Posterior}
Let $Q(\Theta_s|D^m_s)$ denote the posterior of $\Theta_s$ induced by the (marginal) prior $\omega(\Theta_s)$ and $Q(\Theta_t|D^m_s, D^n_t)$ denote the posterior of $\Theta_t$ induced by the prior $\omega(\Theta_t,\Theta_s)$ for transfer scenarios, let $\hat{Q}(\Theta_t|D^n_t)$ denote the posterior induced by the prior $\hat{\omega}(\Theta_t)$ without the source data. After receiving $n$ target samples,  we plot different posteriors to see the effect of the mixture strategy induced by the chosen prior in Figure~\ref{fig:post}. From~\ref{subgifure:post1}, given sufficient source data, the posterior of $\Theta_s$ will give a precise estimation of $\theta^*_s$ and the density will mostly stick around $[0.2,0.4]$. While there is a lack of target samples ($n = 20$), the posterior $\hat{Q}(\Theta_t|D^n_t)$ (Figure~\ref{subgifure:post2}) without the source is relatively scattered and the density around $\theta^*_t$ is quite low. When $n$ increases to 150 (Figure~\ref{subgifure:post3}), the posterior is more concentrated but still not be centered at $\theta^*_t$, implying that more target data are needed for accurate estimation. On the contrary, with the prior knowledge $\omega(\Theta_t|\Theta_s)$ and small $c = 0.1$, the posterior $Q(\Theta_t|D^m_s, D^n_t)$ will be concentrated more around $\theta^*_t$ as source and target parameters are particularly close, which can be seen in~Figure~\ref{subgifure:post4} and ~\ref{subgifure:post5}. However when $c$ increases to $1$, the source data is no longer helpful as $\Theta_t$ is roughly distributed uniformly on $[0,1]^2$ and the posterior behaves similarly to target only case as shown in~Figure~\ref{subgifure:post6}.
\begin{figure}[h!]
\centering
\subfigure[$Q(\Theta_s|D^{5000}_s)$ \label{subgifure:post1}]{\includegraphics[width = 1.8in]{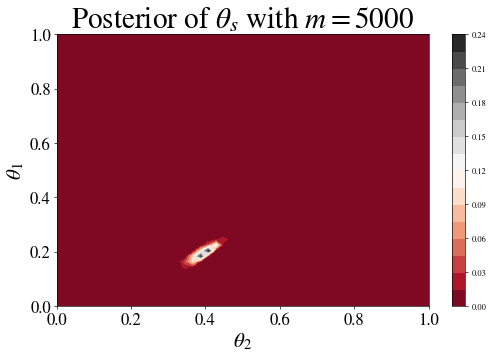}} 
\subfigure[$\hat{Q}(\Theta_t|D^{20}_t)$ \label{subgifure:post2}]{\includegraphics[width =  1.8in]{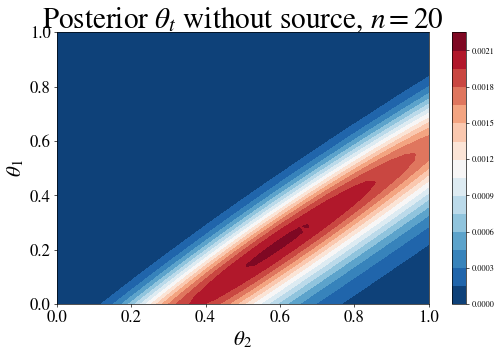}}
\subfigure[$\hat{Q}(\Theta_t|D^{150}_t)$  \label{subgifure:post3}]{\includegraphics[width =  1.8in]{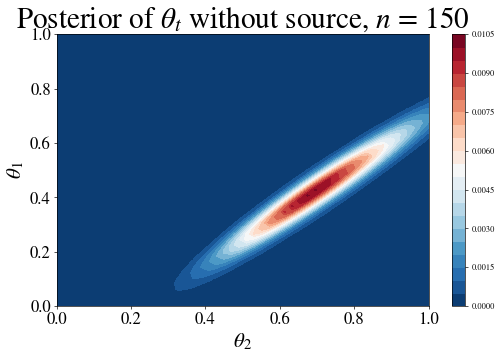}} \\
\subfigure[$Q(\Theta_t|D^{20}_T, D^{5000}_s), c = 0.1$  \label{subgifure:post4}]{\includegraphics[width =   1.8in]{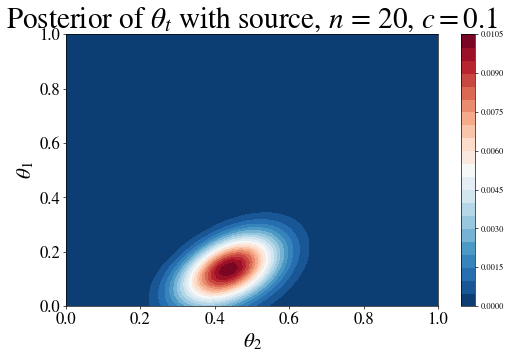}} 
\subfigure[$Q(\Theta_t|D^{150}_T, D^{5000}_s), c = 0.1$  \label{subgifure:post5}]{\includegraphics[width =   1.8in]{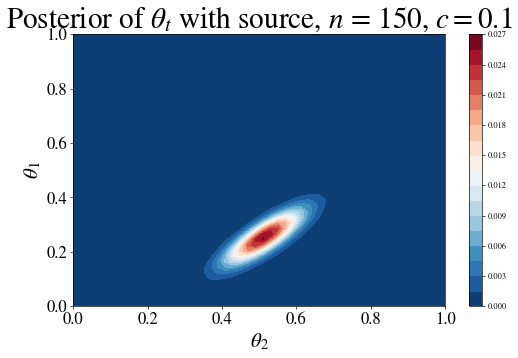}}
\subfigure[$Q(\Theta_t|D^{150}_T, D^{5000}_s), c = 1$  \label{subgifure:post6}]{\includegraphics[width =   1.8in]{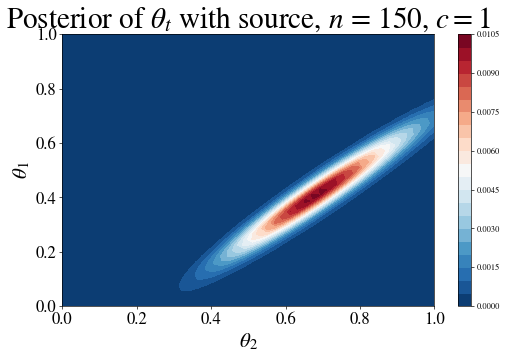}}
\caption{The posterior of $\theta_s$ and $\theta_t$ given $D^m_s$ and $D^n_t$ under different prior belief $c$ and target sample size $n$. The posterior of $\Theta_t$ is more concentrated around $\theta^*_t$ with the source data introduced.}
\label{fig:post}
\end{figure}
\begin{figure}[h!]
    \centering
    \subfigure[Positive Transfer]{\includegraphics[width = 2.7in]{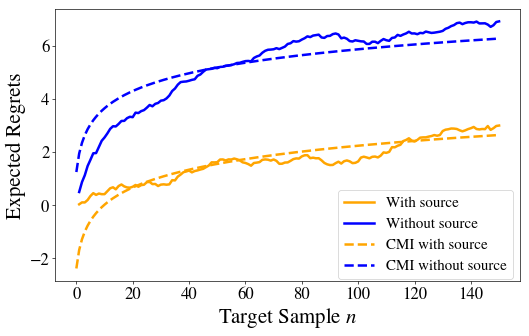}}
    \subfigure[Negative Transfer]{\includegraphics[width = 2.7in]{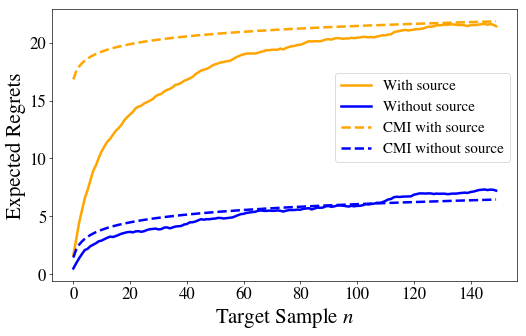}}
    \caption{The comparisons of the expected regret $\mathcal{R}_O$ of positive transfer with $\theta^*_s = [0.2,0.4]$ (left) and the negative transfer (right) with $\theta^*_s = [0.8,0.15]$ under the same settings where $\theta^*_t = [0.3,0.5]$ and $c = 0.1$. The results are approximated by 200 experimental repeats. The regrets without the source data are sketched in orange and those with the source data are sketched in blue.}
    \label{fig:pos_transfer}
\end{figure}

To further demonstrate our theoretical results, we plot the expected regrets in Figure~\ref{fig:pos_transfer} for positive and negative transfer cases, and we also plot the asymptotic estimation of CMI in dashed lines from Theorem~\ref{thm:expreg-log} and \ref{theorem:gene-para} to numerically evaluate the difference. From the left figure, it is observed that introducing the source indeed yields lower regret, which fits our intuition from the posteriors. Even for small $n$($\approx 40$), CMI captures the regret quite well and the gap is roughly $\log\frac{\omega(\theta^*_t|\theta^*_s)}{\omega(\theta^*_t)}$ as noted in Remark~\ref{remark:consist}. In contrast, we also examine the negative transfer case with $\theta^*_s = [0.8,0.15]$ where the results are shown in the right figure. The prior distribution $\omega(\theta^*_c|\theta^*_s)$ in this case has an extremely low density and the estimation will hardly approach the true parameters. As a result, the negative transfer happens and source samples will hurt the performance instead. It appears that the expected regret with source data is far superior to the target only case, and CMI captures this trend well when $n$ goes reasonably large ($\approx 80$). 

Overall, from both positive and negative transfer cases, the  gaps between the regrets are mainly reflected on the prior knowledge $\omega(\theta^*_t|\theta^*_s)$ when $n$ is reasonably large as mentioned in Remark~\ref{remark:consist} and~\ref{remark:knowledge_transfer}, which experimentally confirms Theorem~\ref{theorem:gene-para}. Moreover, it shows that the asymptotic bounds can be also applied to the case when $n$ and $m$ are limited and may provide some practical insights on avoiding the negative transfer. 

\subsubsection*{Efficient Mixture Posterior Updating}
When estimating the full posterior for $\theta^*_t$ and $\theta^*_s$, it is time-consuming when data sizes are large. To examine the efficiency and usefulness of the proposed EMPU algorithms, we set the quantization number $N = 100$ and step size $\eta = 0.01$ to conduct the iterations on the identical logistic regression transfer problem settings for comparisons. The results are shown in the Figure~\ref{fig:empu-compare}. 

\begin{figure}[h!]
    \centering
    \subfigure[Positive Transfer]{\includegraphics[width = 2.5in]{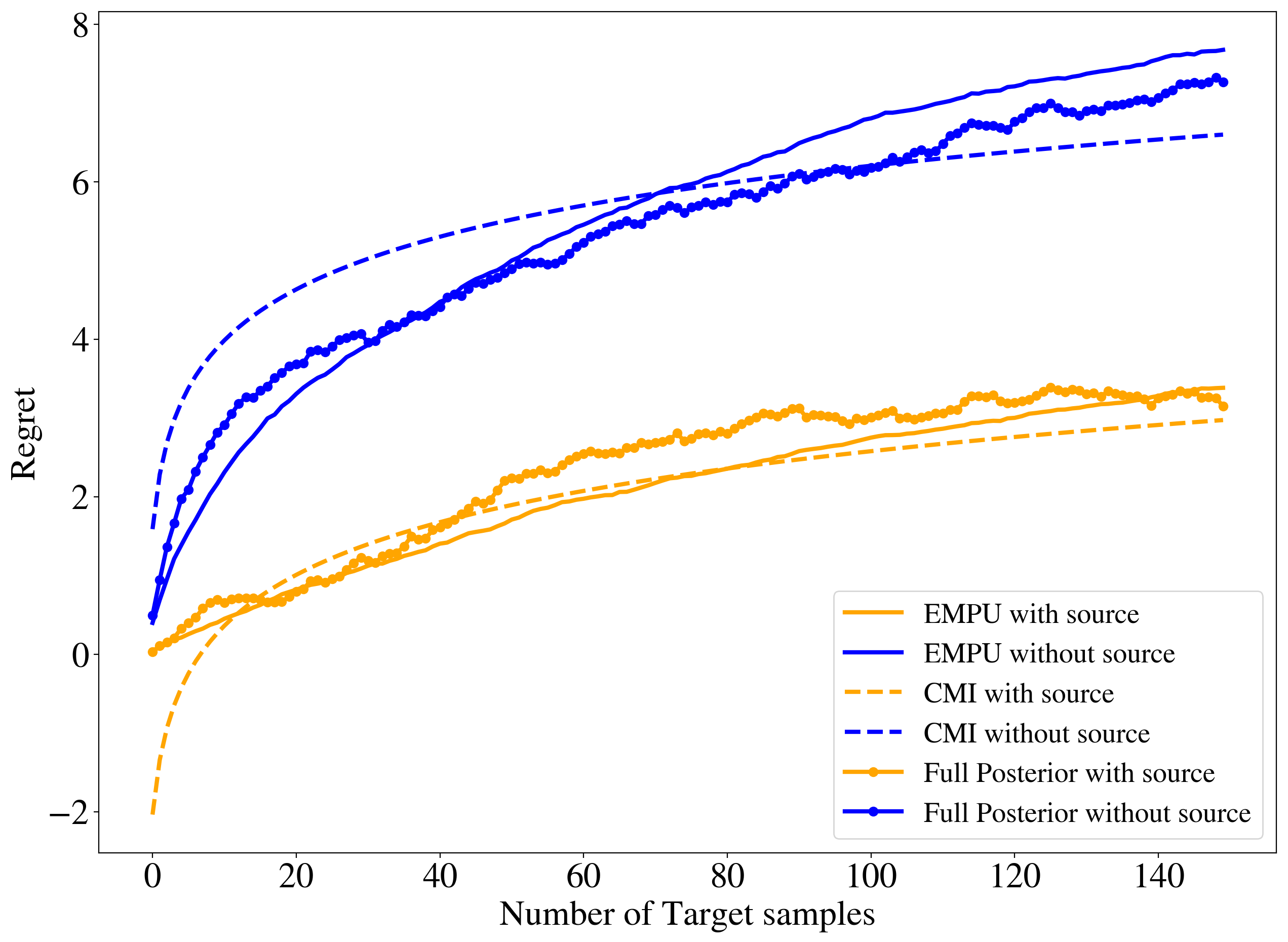}}
    \subfigure[Negative Transfer]{\includegraphics[width = 2.5in]{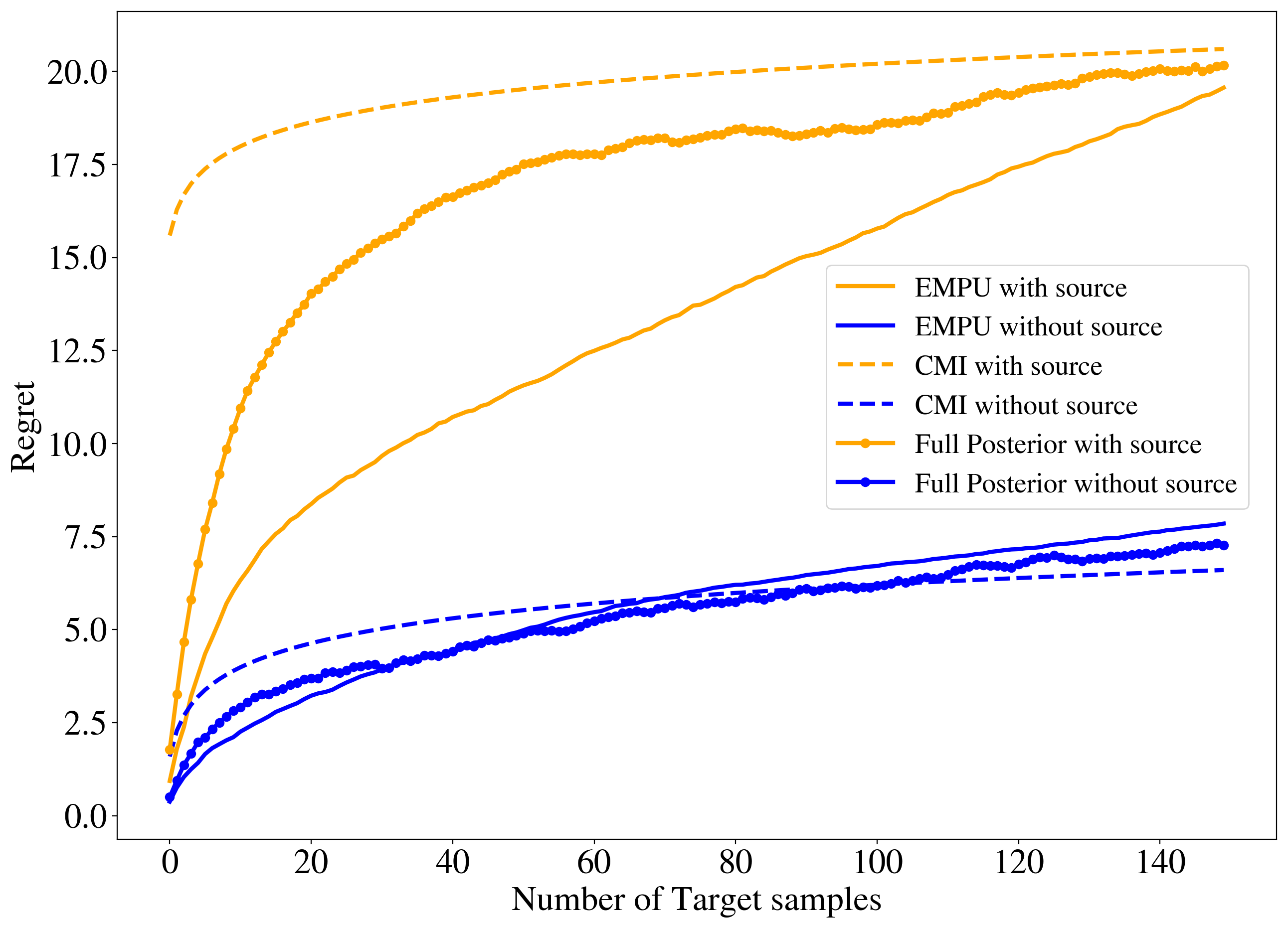}}   
    \caption{Comparisons of full posterior and EMPU algorithms under the positive and negative transfer scenarios. The results are approximated by 200 experimental repeats. The regrets without the source data are sketched in orange and those with the source data are sketched in blue.}\label{fig:empu-compare}
\end{figure}

From the figure, it is easy to see that EMPU achieves similar regrets induced by the full posterior estimation and our derived bounds under both positive and negative transfer cases. As noted, the performance of EMPU depends on the hyperparameters of the step size $\eta$ and the quantization number $N$. We investigate the effects of these hyperparameters and plot the expected regret by varying different $\eta$ and $N$ after receiving 150 target samples under the negative transfer case, non-transfer case (without the source data) and the positive transfer case. From Figure~\ref{subfigure:eta}, it is observed that in non-transfer and positive transfer cases, the lowest expected regrets are achieved at $\eta = 0.01$ while in the negative transfer case, the best performance is achieved at $\eta = 0.001$. It is speculated that under negative transfer case, the estimation of $\theta_t$ will never approach $\theta^*_t$ with the improper prior, then it is more likely to achieve the optima $\tilde{\theta}_t$ defined in Proposition~\ref{prop:neg-online} faster within its support and plausibly jump over it with a larger step size, which incurs a higher regret.  The effects of quantization number $N$ are illustrated in Figure~\ref{subfigure:N}. In all three cases, increasing the quantization number will decrease the regret when $N$ is small (e.g., $N<100$). Nevertheless, no significant changes are spotted when $N$ increases to more than 100. Since larger $N$ leads to higher computational complexity, the best choice for this quantity is $N = 100$.  Furthermore, we plot the running time by varying $N$ from 10 to 1000. Compared to the mixture strategy algorithm with the full posterior, EMPU is 20 times faster under $N = 100$ but achieves similar regret, demonstrating its efficiency. 

\begin{figure}[h!]
    \centering
    \subfigure[Effect of $\eta$ \label{subfigure:eta} ]{\includegraphics[width = 2in]{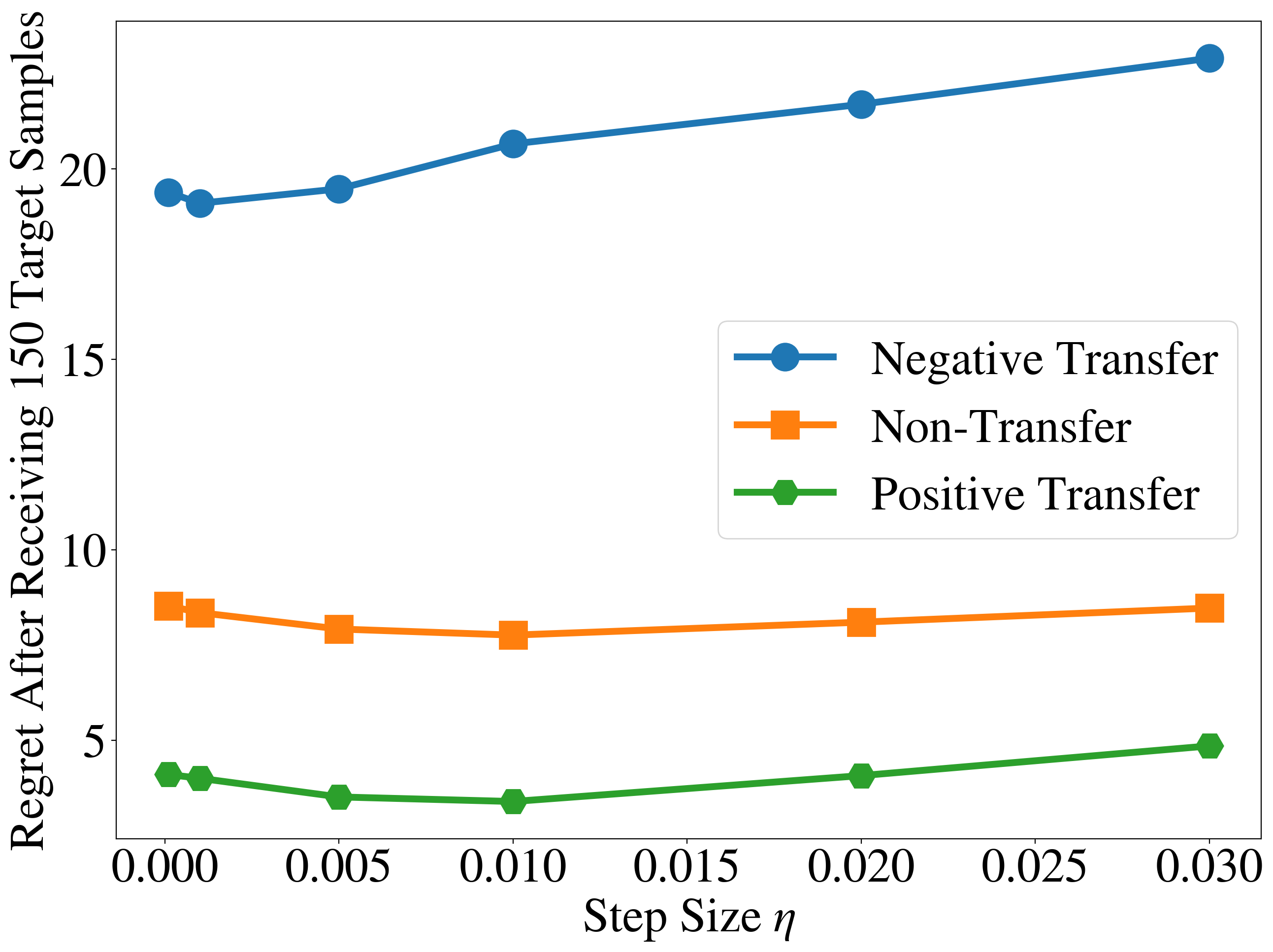}}
    \subfigure[Effect of $N$ \label{subfigure:N} ]{\includegraphics[width = 2in]{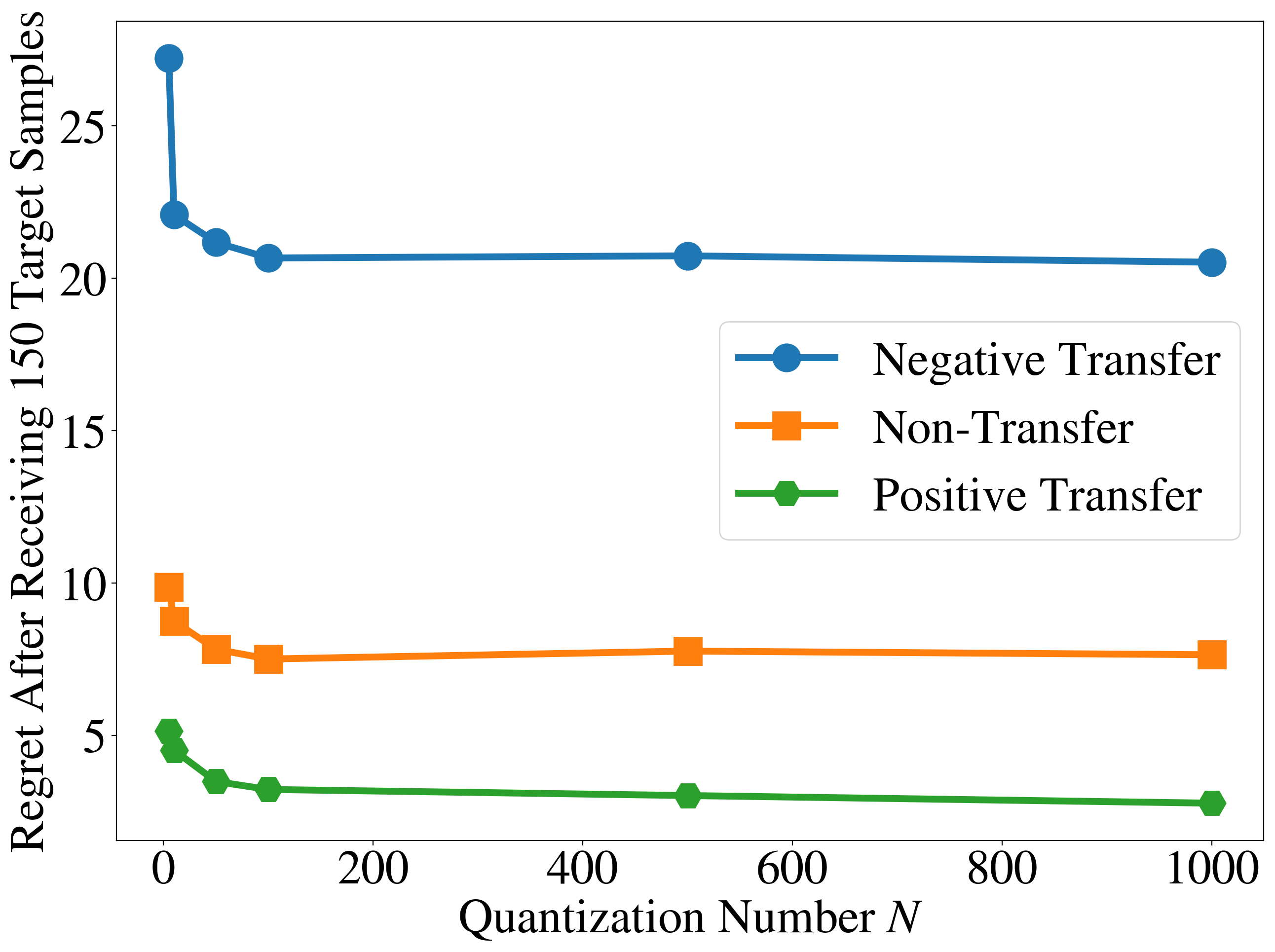}}
    \subfigure[Time Comparisons \label{subfigure:time}]{\includegraphics[width = 2.1in]{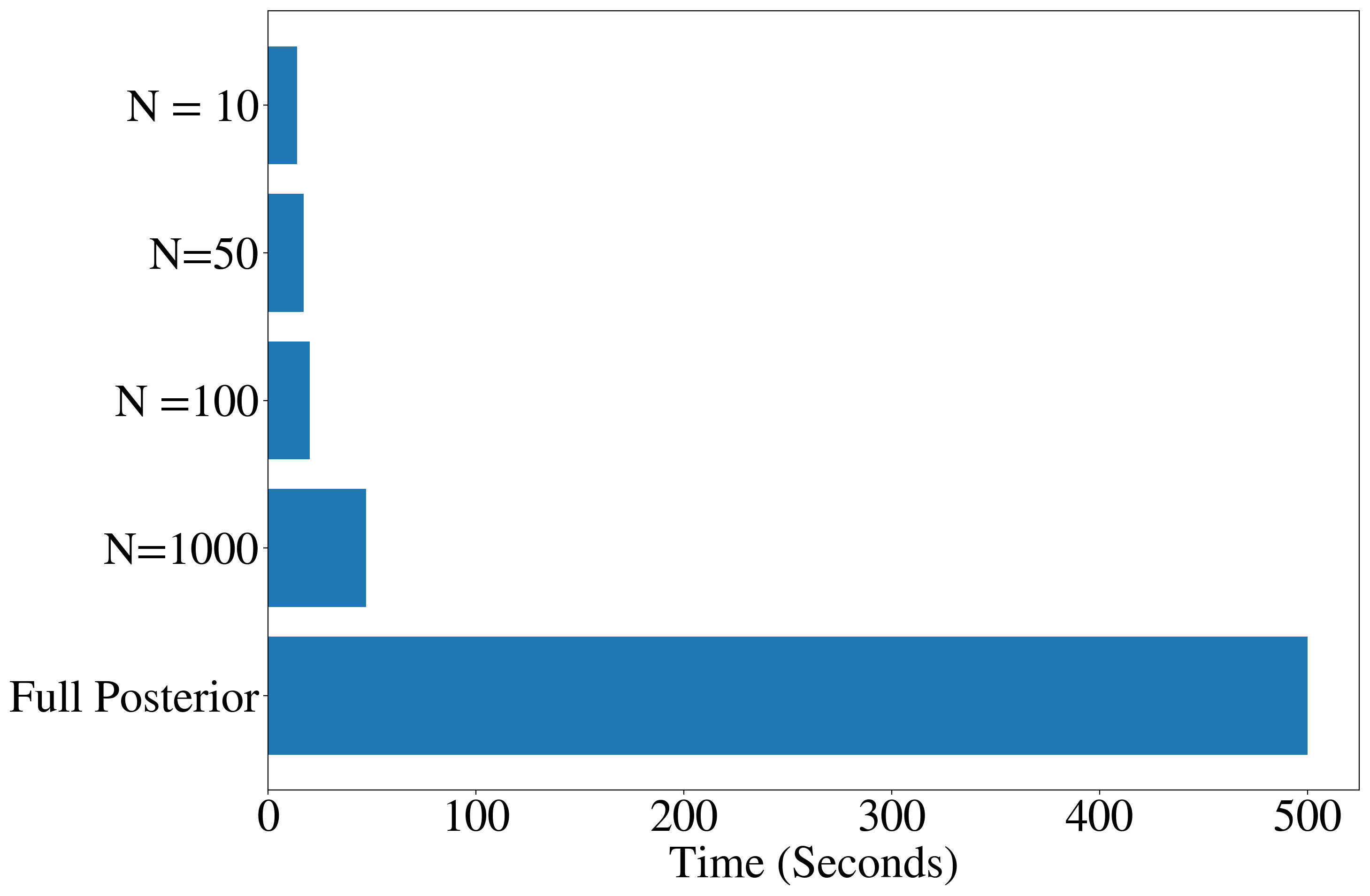}}
    \caption{After receiving 150 target samples, we plot the expected regret by varying different $\eta$ and $N$ in (a) and (b). We also compare the running time under different quantization number $N$ and the full posterior algorithm. The results are approximated by 200 experimental repeats.}
    \label{fig:sensitivity}
\end{figure}

\subsubsection*{Comparisons to other OTL algorithms}
We compare our proposed EMPU with the online transfer learning algorithm proposed in \cite{zhao2014online}, the HomOTL-I algorithm. In their scheme, the authors assign the weights $\omega_s$ and $\omega_t$ for source and target domains and each domain is endowed with the model parameters $\theta_s$ and $\theta_t$. At each time $k$, the weights and the model parameters are updated with the following rules:
\begin{equation}
\left\{\begin{aligned}
\omega^k_{s} &=\frac{\omega^{k-1}_{s}  g\left(\ell^*\left(\theta_s, Z^k_t\right)\right)}{\omega^{k-1}_{t}  g\left(\ell^* (\theta^{k-1}_t, Z^k_t\right)+\omega^{k-1}_{t} g\left( \ell(\theta_s, Z^k_t) \right)}, \omega^0_s = 0 \\
\omega^{k}_{t} &=\frac{\omega^{k-1}_{t}  g\left(\ell^*\left(\theta_t, Z^k_t\right)\right)}{\omega^{k-1}_{t}  g\left(\ell^* (\theta^{k-1}_t, Z^k_t\right)+\omega^{k-1}_{t} g\left( \ell(\theta_s, Z^k_t) \right)}, \omega^0_t = 0  \\
\theta^k_t &= \theta^{k-1}_t - \eta\nabla_{\theta_t}\ell (\theta^{k-1}_t, Z^k_t) 
\end{aligned}\right. \label{eq:update}
\end{equation}
where $g$ is some decaying function, $\ell$ is the hinge loss and $\ell^*$ is the squared loss. The predictions are made from the learned model parameters $\theta_s$ and $\theta_t$ along with the weights $\omega_s$ and $\omega_t$ by a truncated linear model. Similar updating rules are proposed in \cite{wu2017online,yan2017online} with different choice of $g$. The EMPU algorithm is closely related to this framework. To illustrate, if we regard the distribution parameters $\Theta_s$ and $\Theta_t$ as the model parameters and choose the decaying function $g$ as the exponential weighting function, e.g., $g(x) = e^{-x}$, and let both $\ell$ and $\ell^*$ be the cross-entropy loss, then the above scheme will practically coincide with our proposed method if we choose $N = 2$. The main differences are that $\theta_t$ will be updated and $\theta_s$ is kept fixed at each iteration in their scheme, and our model is not limited to linear models, and the loss function is not limited to hinge loss. As a consequence, there is no prior knowledge over the target domain, and the performance at the beginning will rely heavily on the source domain since $\theta_s$ remains unchanged all the time, thus may perform badly if two domains are totally distinct. While in our scheme, the source parameter $\Theta_s$ behaves as the prior knowledge for the target parameter, which is not explicitly engaged in the prediction. Thus we can choose a proper prior at the beginning to either avoid the negative transfer if the source domain varies differently from the target domain or improve the prediction if two domains are close. 

Using the same settings of the logistic regression problems, we study the \emph{small discrepancy} case where the target and source data are generated with $\theta^*_t = [0.3,0.5]$ and $\theta^*_s = [0.2, 0.4]$ and the \emph{large discrepancy} case where $\theta^*_t = [0.3,0.5]$ and $\theta^*_s = [0.8, 0.15]$. In our method, we will set quantization number $N =2$ and the prior knowledge $c = 0.3$ so that we will have proper priors under both cases. Then at each iteration $k$ we predict the label $Y_k$ by first sampling $\theta_{k,i}$ according to the distribution $\omega_k$, the predicted $\hat{Y}_k$ will be $1$ if the probability $\sigma(\theta_{k,i}^TX_i) > 0.5$ and 0 otherwise. In HomOTL-I, we firstly learn $\hat{\theta}_s$ by the linear regression method and initialize $\theta_{t,0} = 0$, then at each iteration $k$ we conduct the HomOTL-I to predict $\hat{Y}_k$. We will plot the accumulated number of mistakes $\sum^{T}_{i=1}|\hat{Y}_k - Y_k|$ by averaging 200 experimental repeats. We choose the step size $\eta = 0.01$ for both algorithms. The comparisons can be found in Figure~\ref{fig:comparison_mistake_oz}. From the small discrepancy case, HomOTL-I indeed performs slightly better than the EMPU algorithm since the initially learned $\hat{\theta}_s$ is very helpful for prediction in the target domain as $P_{\theta^*_s}(Y|X)$ and $P_{\theta^*_t}(Y|X)$ are relatively close in terms of the parameterization, which leads to a smaller number of mistakes. In the large discrepancy case, the domain divergence becomes relatively large, and the initially learned $\hat{\theta}_s$ is more of a hindrance which causes a larger number of mistakes in the target domain. The EMPU algorithm performs reasonably well in both cases, which also shows that the updating with the prior knowledge is particularly beneficial for achieving a lower regret when two domains are far different.

\begin{figure}[h!]
    \centering
    \subfigure[Small discrepancy \label{subfigure:postive_comp} ]{\includegraphics[width = 2.5in]{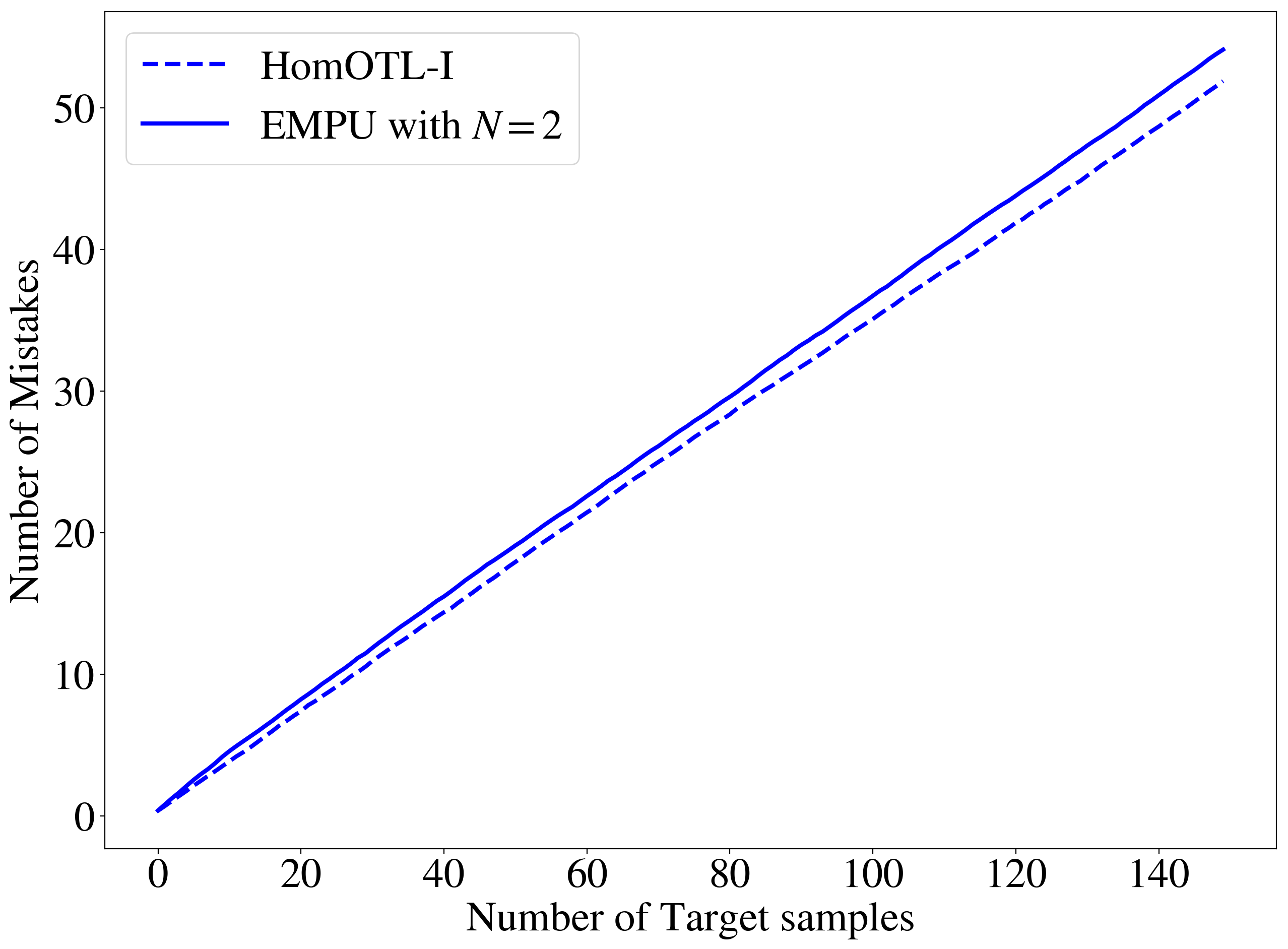}} \quad 
    \subfigure[Large discrepancy \label{subfigure:nega_comp} ]{\includegraphics[width = 2.5in]{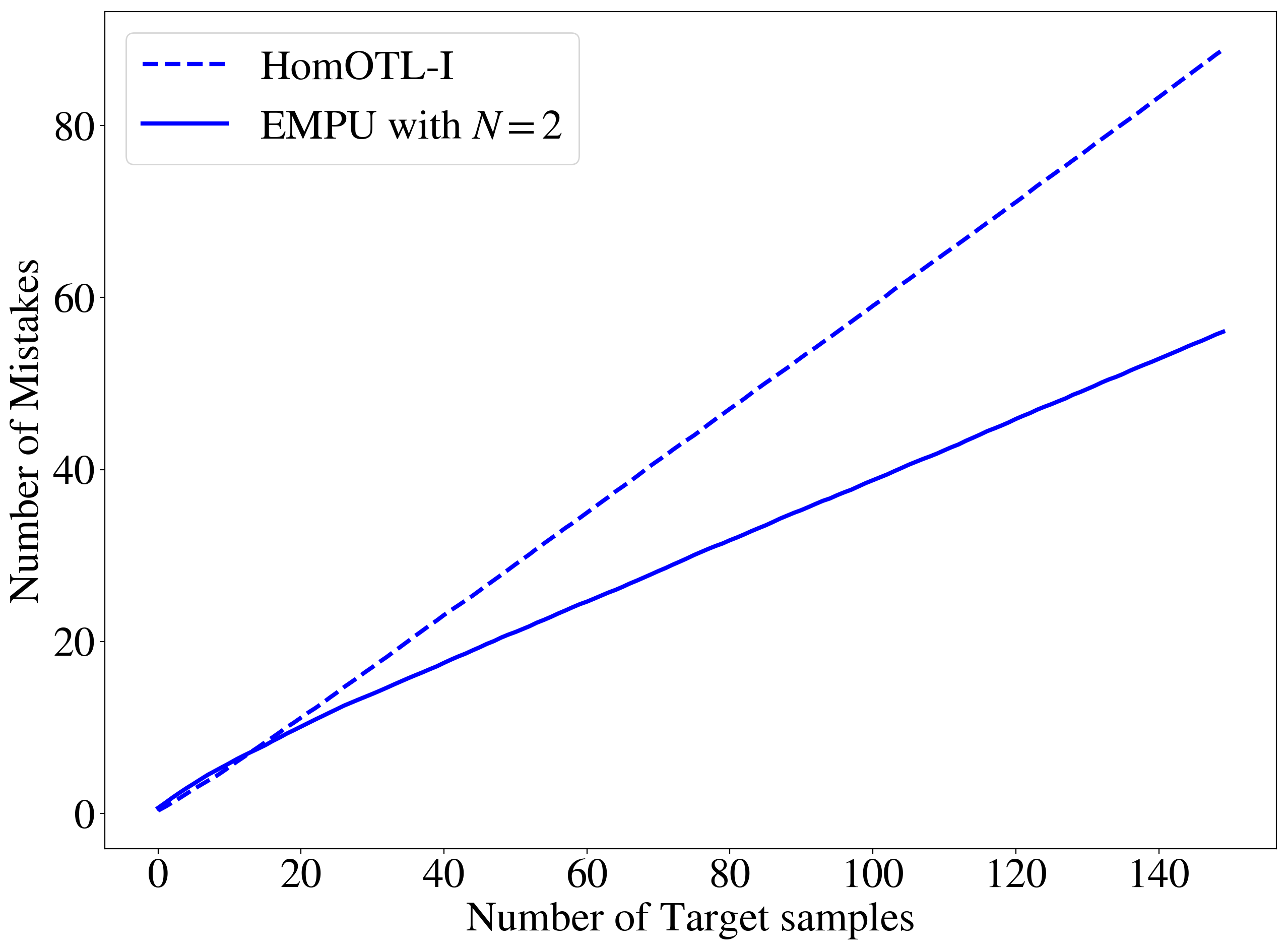}}
    \caption{Comparisons of EMPU with $N=2$ and HomOTL-I method \citep{zhao2014online} by the average number of mistakes made at each iteration. We set the step size $\eta$ to be 0.01 for all experiments. The results are approximated by 200 experimental repeats.}
    \label{fig:comparison_mistake_oz}
\end{figure}

Likewise, we adapt the HomOTL-I algorithm to our framework by choosing both $\ell$ and $\ell^*$ to be the cross-entropy loss defined in Equation~(\ref{eq:cross_entropy}), and we use the logistic regression as our predictive model. To elaborate, we firstly learn $\hat{\theta}_s$ by the logistic regression method and initialize $\theta_{t,0}$ uniform randomly in $\Lambda$. The comparisons can be found in Figure~\ref{fig:comparison_regret} under the small and the large discrepancy scenarios by the expected regret. We observe similar patterns in both scenarios. HomOTL-I indeed performs better than both EMPU and the mixture strategy with full posterior since the estimated $\hat{\theta}_s$ is already very close to the true target parameters $\theta^*_t = [0.3,0.5]$ at the beginning of the prediction, thus a lower regret. However, in the negative transfer case, the initialized $\hat{\theta}_s$ is located fairly far away from $\theta^*_t$, HomOTL-I will mainly rely on the bad-performing $\hat{\theta}_s$ and induce a much higher expected regret. 
\begin{figure}[h!]
    \centering
    \subfigure[Small discrepancy \label{subfigure:postive_comp2} ]{\includegraphics[width = 2.5in]{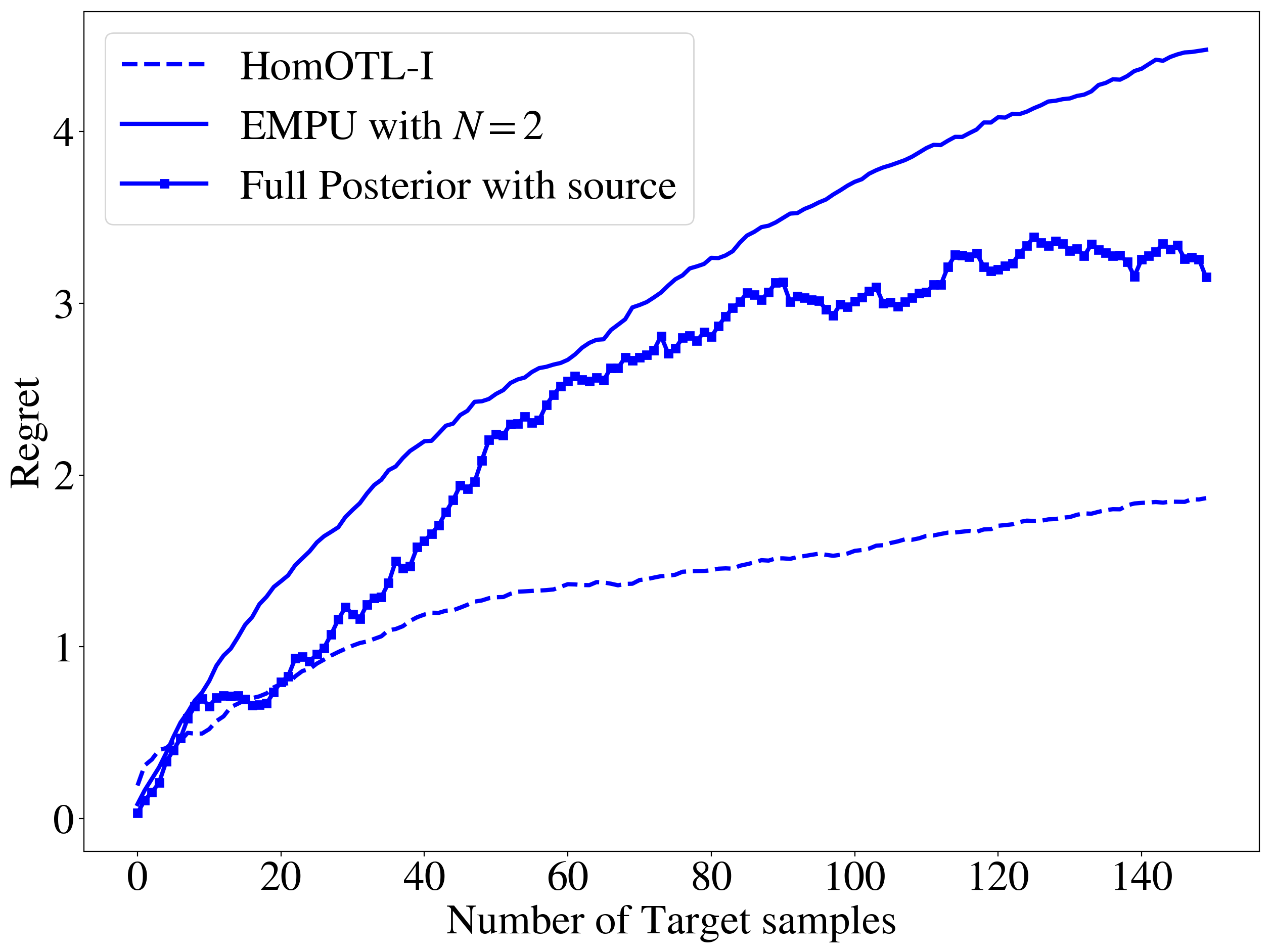}} \quad 
    \subfigure[Large discrepancy \label{subfigure:nega_comp2} ]{\includegraphics[width = 2.5in]{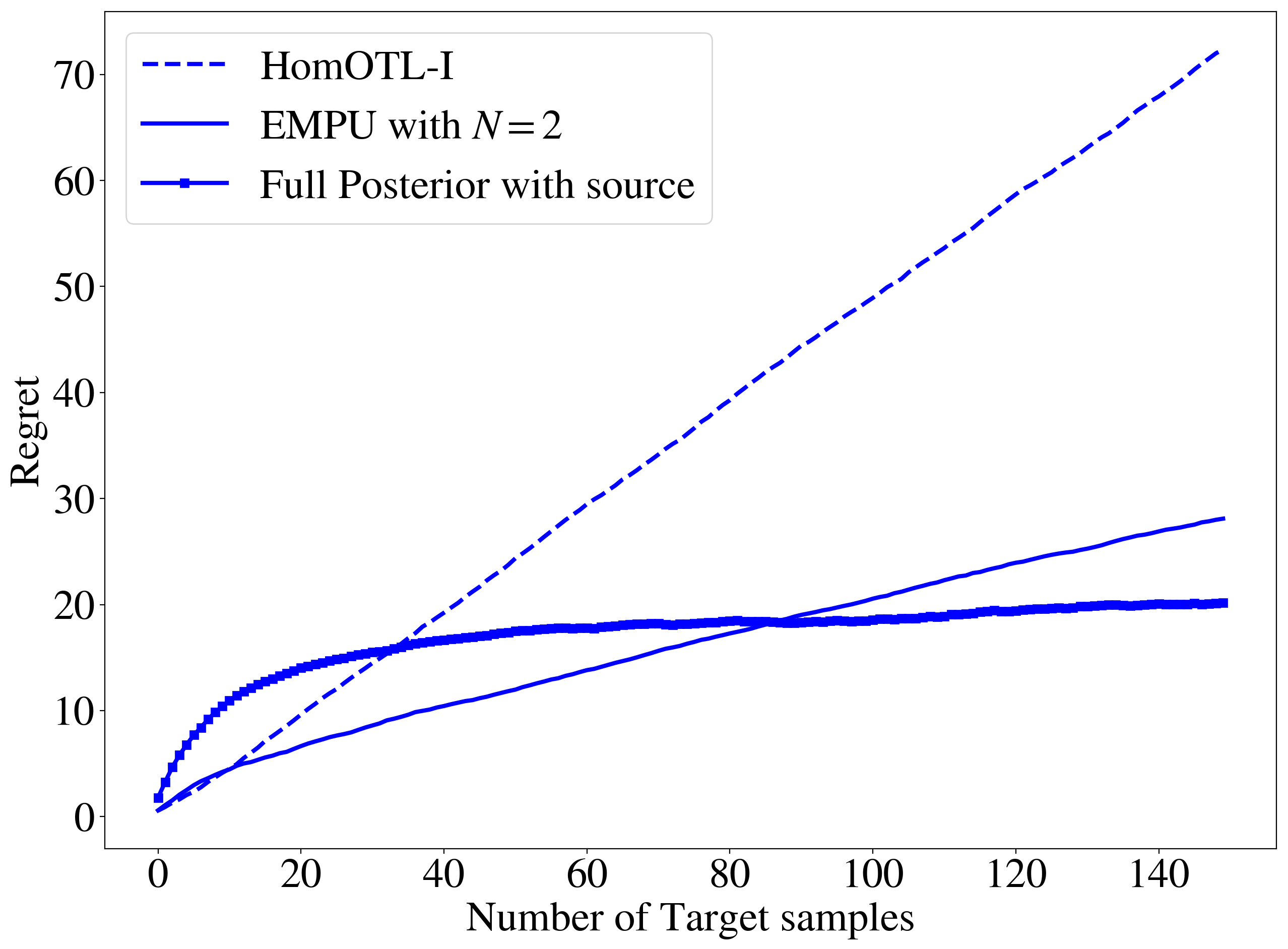}}
    \caption{Comparisons of EMPU with $N=2$ and HomOTL-I \citep{zhao2014online} by the expected regret. We set the step size $\eta$ to be 0.01 for all experiments. The results are approximated by 200 experimental repeats.}
    \label{fig:comparison_regret}
\end{figure}

\subsection{Nonparametric Modelling}
Previous results are derived under the assumption that the data distributions are parametric, say, $Z \sim P_{\theta}(Z)$. Such an assumption of the parametric model may not always hold in practical machine learning problems such as image processing and sentiment analysis. Inspired by the mixture strategy, we devise a novel algorithm that works in a nonparametric setting for more general models. Roughly speaking, the nonparametric algorithm allows that the number of parameters in the model grows with the number of samples. Similar to the mixture strategy on a parametric model, we will define a distribution on the parameter of possibly infinite dimension. Specifically, we use the Dirichlet Process Mixture (DP) \citep{ferguson1973bayesian, antoniak1974mixtures, neal2000markov} to construct the prior and permit the posterior inference from the data. The Dirichlet Process is usually denoted by
\begin{align*}
    G \sim DP(\alpha, G_0),
\end{align*}
where $G$ is a random discrete measure. There are two parameters in the Dirichlet process, the base distribution $G_0$, which represents a random guess of the data true distribution, and $\alpha$ is a positive scalar controlling the concentration. From the DP, we could model the data distributions by the DP mixture (DPM) where the distribution parameters $\theta_i, i\in\mathbb{N}$ are drawn from some random measure $G$ sampled from $DP(\alpha,G_0)$, and each data $Z_i$ is drawn from the parametric distribution $P_{\theta_i}(Z)$. If we integrate a parametric density $P_{\theta_i}(Z)$ against the random measure $G$, we obtain a mixture model from the stick-breaking process \citep{paisley2010simple}  as
\begin{align}
    P(Z)=\sum_{i \in \mathbb{N}} \pi_{i} P_{\theta_i}\left(Z \right).
\end{align}
 We sketch the following procedure to illustrate how the data are generated from the DPM:
\begin{equation}
\begin{aligned}
G \sim DP(\alpha, G_0), \\
\theta_1,\theta_2,\cdots |G \sim G,\\
Z_i|\theta_i \sim P_{\theta_i}(Z).
\end{aligned}
\end{equation}
Using this generative model, we view $\theta_i$ as the latent variables and our aim is to estimate the posterior $ P(\theta_1,\theta_2,\cdots,\theta_n|Z_1,Z_2,\cdots,Z_n)$ from the dataset as each sample corresponds to an atom $\theta_i$ that is drawn from $G$. More specifically, we can find conditional distributions of posterior distribution of model parameters by combining the conditional probability as \citep{neal2000markov}
\begin{align}
\theta_i |\theta^{-i}, Z_i \sim \sum_{j\neq i} q_{ij}\delta_{\theta_j} + r_iG_i, \label{eq:nonpara}
\end{align}
where $\theta^{-i} = \{\theta_1,\cdots, \theta_{i-1}, \theta_{i+1}, \cdots, \theta_n\}$ denotes the set of $\theta_j$ for all $j \neq i$, $G_i$ is the posterior distribution of $\theta$ based on the base distribution $G_0$ and data $Z_i$, and the coefficients $q_{ij}$ and $r_i$ are defined as
\begin{align}
    q_{ij} &= bP_{\theta_j}(Z_i), \\
    r_i  &= b\alpha \int P_{\theta}(Z_i) dG_{0}(\theta),
\end{align}
where $b$ is such that $\sum_{j \neq i} q_{ij}+r_{i}=1$. At this stage, we can use the MCMC sampling method to approximate the posterior. Inspired by this nonparameteric Bayesian framework, we now describe how the idea of mixture strategy in the parametric model carries over to the nonparametric model. Assume that the distributions of each source data $Z^{(i)}_s$ and target data $Z^{(i)}_t$ are parameterized by $\theta_{s,i}$ and $\theta_{t,i}$, respectively. And we assume that $\theta_{s,i}$ and $\theta_{t,i}$ are drawn from the distributions $G_s$ and $G_t$ generated by the prior $DP(\alpha, G_0)$. That is:
\begin{equation}
\begin{array}{ccc}
G_s \sim DP(\alpha, G_0) & & G_t \sim DP(\alpha, G_0) \\
\theta_{s,1},\theta_{s,2},\cdots | G_s\sim G_s & &  \theta_{t,1},\theta_{t,2},\cdots | G_t\sim G_t \\
Z^{(i)}_t|\theta_{t,i} \sim P_{\theta_{t,i}}(Z) &  &  Z^{(j)}_s |\theta_{s,j} \sim P_{\theta_{s,j}}(Z) 
\end{array}.
\end{equation}
Following~Equation~(\ref{eq:nonpara}), we can write out the posterior of $\theta_{t,i}$ given the target data $Z^{(i)}_{t}$, the set of $\theta_{t,j}$ for all $j \neq i$ (written as $\theta^{-i}_t$) and the set of $\theta_{s,k}$ for $k = 1,2, \cdots,m$ (written as $\theta^m_s$):
\begin{equation}
  \theta_{t,i} |\theta^{-i}_t, \theta^m_s, Z_i \sim \sum_{j\neq i} q_{ij}\delta_{\theta_{t,j}} + \sum_{k} q_{ik}\delta_{\theta_{s,k}} + r_iG_i, \label{eq:transfer-nonpara}
\end{equation}
The coefficients $q_{ij}$, $q_{ik}$ and $r_i$ are defined as follows
\begin{align}
    q_{ij} &= bP_{\theta_{t,j}}(Z_i), \\
    q_{ik} &= \beta b\alpha P_{\theta_{s,k}}(Z_i), \\
    r_i  &= (1-\beta) b\alpha  \int P_{\theta}(Z_i) dG_{0}(\theta).
\end{align}
where $b$ is such that $\sum_{j \neq i} q_{ij} + \sum_{k} q_{ik} + r_{i}=1$. Here we introduce the balancing coefficient $\beta \in [0,1]$ as our prior knowledge over the source and target domain. Larger $\beta$ implies that we will rely more on the (empirical) source distribution as our prior and smaller $\beta$ shows more beliefs on $G_0$. The next step is to estimate the posterior of $\theta_{t,i}$ from~Equation~(\ref{eq:transfer-nonpara}) by the MCMC sampling algorithm. We then heuristically propose an efficient algorithm for online transfer learning under nonparametric Bayesian learning framework.

\begin{algorithm}[h!]
\SetAlgoLined
\SetKwInOut{Input}{input}
\SetKwInOut{Output}{output}
\Input{$\mathcal{D}^m_t$, Base distribution $G_0$, $\alpha$, weighting coefficient $\beta$,  parametric family $P_{\theta}$ }
 Estimate $G_s$ from $\mathcal{D}^m_S$ from the Gibbs sampling by~(\ref{eq:nonpara}) \; 
 Initialize distribution $DP(\alpha, (1-\beta) G_0 + \beta G_s)$ as our prior knowledge for the target domain as indicated by~(\ref{eq:transfer-nonpara})\;
 \For{$k = 1,\cdots,T$}{ 
 Receive target sample $Z^{(k)}_t$\;
 \For{$i = 1,\cdots, K$}{
 Estimate the posterior of the target distribution $G_t$ using Equation~(\ref{eq:transfer-nonpara}) by Gibbs sampling, say, we sample $\theta^{i}_{t,k}$\;   
 }     
 Predict $P_k(Z^{(k)}_t) = \sum_{i=1}^{K}P_{\theta^{i}_{t,k}}(Z^{(k)}_t)$ using sampled $G_t$ \;
 }
 \Output{Sequential prediction $P_k(Z^{(k)}_t)$}
 \caption{Nonparametric Posterior Updating and Prediction}
 \label{alg:nonpara}
\end{algorithm}

\begin{remark}
We can see the analogy between the parametric model and the nonparametric model from the above algorithm. Firstly we estimate $G_s$ from the source data with the Gibbs sampling, which behaves similarly as estimating $\theta_s$ in the parametric model. Then we define the coefficient weight $\beta$ that controls whether $G_t$ is similar to $G_s$ as our prior knowledge, which corresponds to $\omega(\Theta_t|\Theta_s)$ in the parametric model. When we receive new target samples $Z^{(k)}_t$, the posterior is updated and $\theta_{t,k}$ are sampled $K$ times from $G_t$.  Then the probability distribution of $Z^{(k)}_t$ is approximated by the mixture of each $\theta_i$. The effectiveness of this algorithm will depend on $\alpha$, $\beta$, choice of $G_0$, the parametric family $P_{\theta}$ and the sampling number $K$. The analogy between the parametric and nonparametric modelling is summarized in Table~\ref{tab:comparison}. Despite the fact that the two methods have some similarities, there is no theoretical guarantee for the generalization performance of the nonparametric models. To show the effectiveness of the prior knowledge in the nonparametric algorithm, we conduct some simple experiments for empirical verification.
\begin{table}[H]
    \centering
    \begin{tabular}{|c|c|c|}
    \hline 
   &  Parametric & Nonparametric \\
  \hline 
  Prior         &   $\Theta_t \sim \omega(\Theta_t|\Theta_s)$   & $G_t \sim DP(\alpha, \beta G_s + (1-\beta) G_0)$    \\
  Likelihood    &   $P_{\theta^*_s}(Z_s), P_{\theta^*_t}(Z_t)$  &  $P_{\theta_{s,j}}(Z^{(j)}_s), P_{\theta_{t,i}}(Z^{(i)}_t)$   \\
  Mixture    &    $\int P_{\theta_s}(D^n_t)\omega(\theta_t|\theta_s) d\theta_t P(\theta_s|D^m_s)d\theta_s$ &   $\sum_{j\neq i} q_{ij}\delta_{\theta_{t,j}} + \sum_{k} q_{ik}\delta_{\theta_{s,k}} + r_iG_i$  \\
  Prediction   &  $\argmin_{b} \mathbb{E}_Q\left[\ell(b,Z_t)\right]$ &  $\argmin_{b} \sum_{i=1}^{K}\mathbb{E}_{\theta_{t,i}}\left[\ell(b,Z_t)\right]$   \\
  \hline 
    \end{tabular}
    \caption{Comparisons between Parametric and Nonparametric Models}
    \label{tab:comparison}
\end{table}

\end{remark}

\subsubsection*{Experimental Results}

We illustrate the effect of the prior knowledge in the nonparametric model by firstly validating the algorithm in logistic regression problems for small and large discrepancy scenarios as described in~\ref{subsection:logistic}. Here we use the same hyperparameters and parametric models given in \cite{shahbaba2009nonlinear} (See Simulation 1 for more details). We set $\alpha = 0.01$ and vary $\beta$ The results are shown in Figure~\ref{fig:log-nonpara}. From the regret curves, we can see that if the domain divergence is small, one can achieve a better regret with a large $\beta$ because the posterior will rely more on the source data, which is helpful for prediction. On the contrary, if the target distribution differs much from the source distribution, the regret will be higher with large $\beta$ as relying more on the source data will hurt the performance on the target. If we decrease $\beta$ to $0.01$, the prediction counts more on the target data. Regardless of the source data, the regrets become close under both small and large discrepancies. 

\begin{figure}[h!]
    \centering
    \includegraphics[width = 5.5cm]{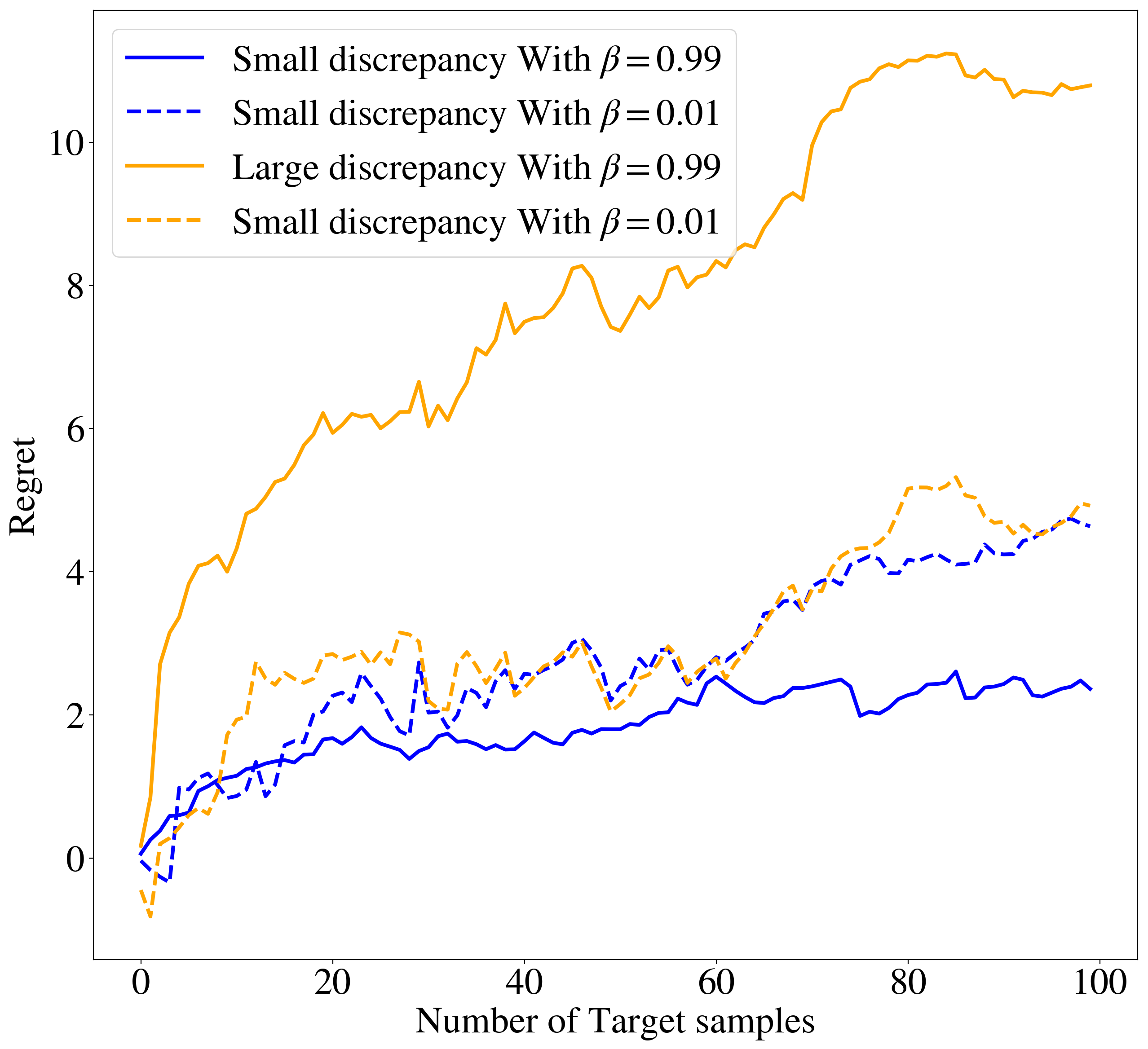}
    \caption{Regret of the Logistic regression problems under the small and large discrepancy scenarios, we choose $\beta$ to be $0.99$ and $0.01$ to show the usefulness of the source.}
    \label{fig:log-nonpara}
\end{figure}

 We also validate the nonparametric algorithm for a real-world transfer problem using the MNIST \citep{lecun-mnisthandwrittendigit-2010} and USPS datasets \citep{uspsdataset}, which are standard digit recognition datasets containing hand-written digits from 0-9. USPS contains 7291 training samples and 2007 testing samples with the resolution of 16 $\times$ 16 while MNIST is consist of 60000 images for training and 10000 images for testing with the resolution of $28 \times 28$. Examples for each dataset are shown in Figure~\ref{fig:handwritten}.
\begin{figure}[h!]
    \centering
    \subfigure[MNIST]{\includegraphics[height = 4cm]{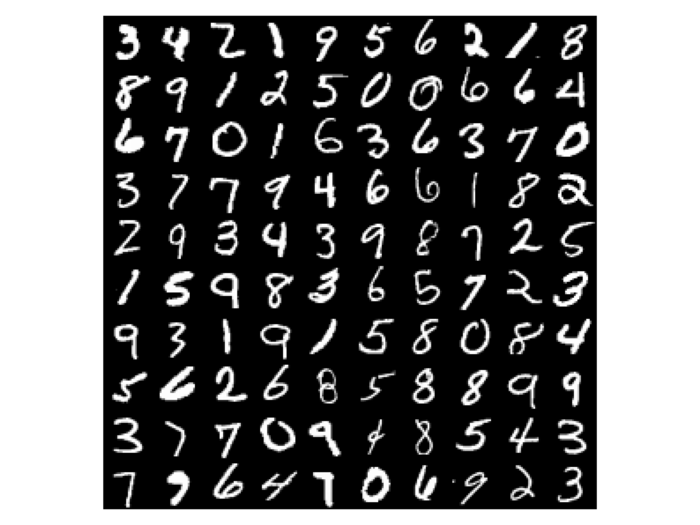}}
    \subfigure[USPS]{\includegraphics[height = 3.9cm]{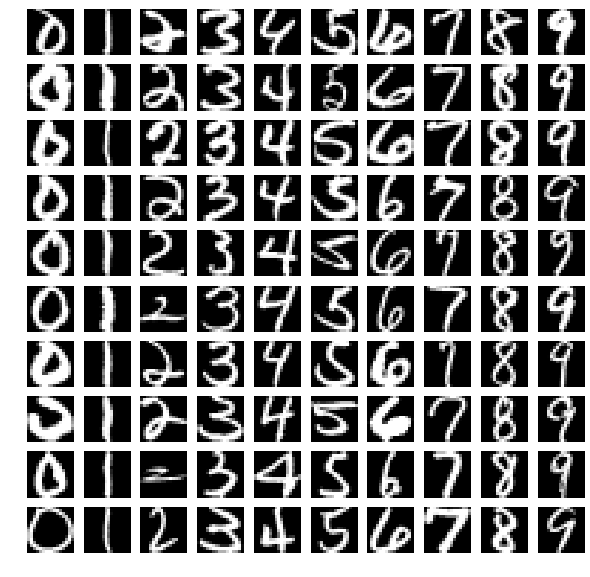}}    
    \caption{Hand-written digit recognition transfer problems}
    \label{fig:handwritten}
\end{figure}
We sampled 2000 images from the MNIST dataset and 1800 images from the USPS dataset and adopted the 256 SURF features\footnote{\url{https://github.com/jindongwang/transferlearning/tree/master/data}} following the existing literature \citep{wang2020transfer}. We conduct the sequential knowledge transfer between these two datasets. For simplicity, we use U $\rightarrow M$ (M $\rightarrow$ U) to denote the transfer from source domain USPS (MNIST) to the target domain MNIST (USPS). For each transfer case, we use the whole batch of the source data and randomly sample 80 samples from the target data as the initialization, then the rest target will arrive sequentially. Our nonparametric algorithm will output the predictive probability for each class, and we will plot the regret by the multi-class cross-entropy loss according to the prediction: 
\begin{align}
 \mathcal{R}_O = - \sum^T_{i=1}\sum_{k=1}^{C} \mathbf{1}_{Y_{i} = k} \log \left(P_{\theta_i}(Y = k|X_i)\right)
\end{align}
where $\mathbf{1}_{Y_{i} = k}$ denotes the indicator function that whether the output $Y_i$ is equal to the $k$th-class, and $P_{\theta_i}(Y = k|X_i)$ denotes the predictive probability of $k$th-class given the parameters $\theta_i$ from the posterior of the Dirichlet process mixture and $X_i$. Moreover, we will make prediction $\hat{Y}_{i}$ at each iteration according to the class with the highest predictive probability and plot the number of mistakes:
\begin{align}
\mathcal{R}_{mistake} = \sum^T_{i=1} \mathbf{1}_{Y_{i} = \hat{Y}_{i}}.
\end{align}
Owing to the fact the computational cost for sampling the posterior with 256 input dimensions becomes an issue in real implementation, we first reduce the dimension from 256 to 16 by some projection methodologies in domain adaptation such as \cite{long2013transfer}. By this, we could not only alleviate the computational cost issues but also reduce the domain divergence. Then we vary $\beta$ from $0.001$ to $0.999$ to show our prior belief on the usefulness of the source and target data. We also compare our method to the baseline dpMNL method described in \cite{shahbaba2009nonlinear} by treating the source and target equally as one domain. After receiving 100 target samples, we plot the accumulated mistakes make by the results are shown in the Figure~\ref{fig:comparison_mistake}.

\begin{figure}[h!]
    \centering
    \subfigure[U $\rightarrow$ M by Cross-Entropy Loss]{\includegraphics[width = 1.8in]{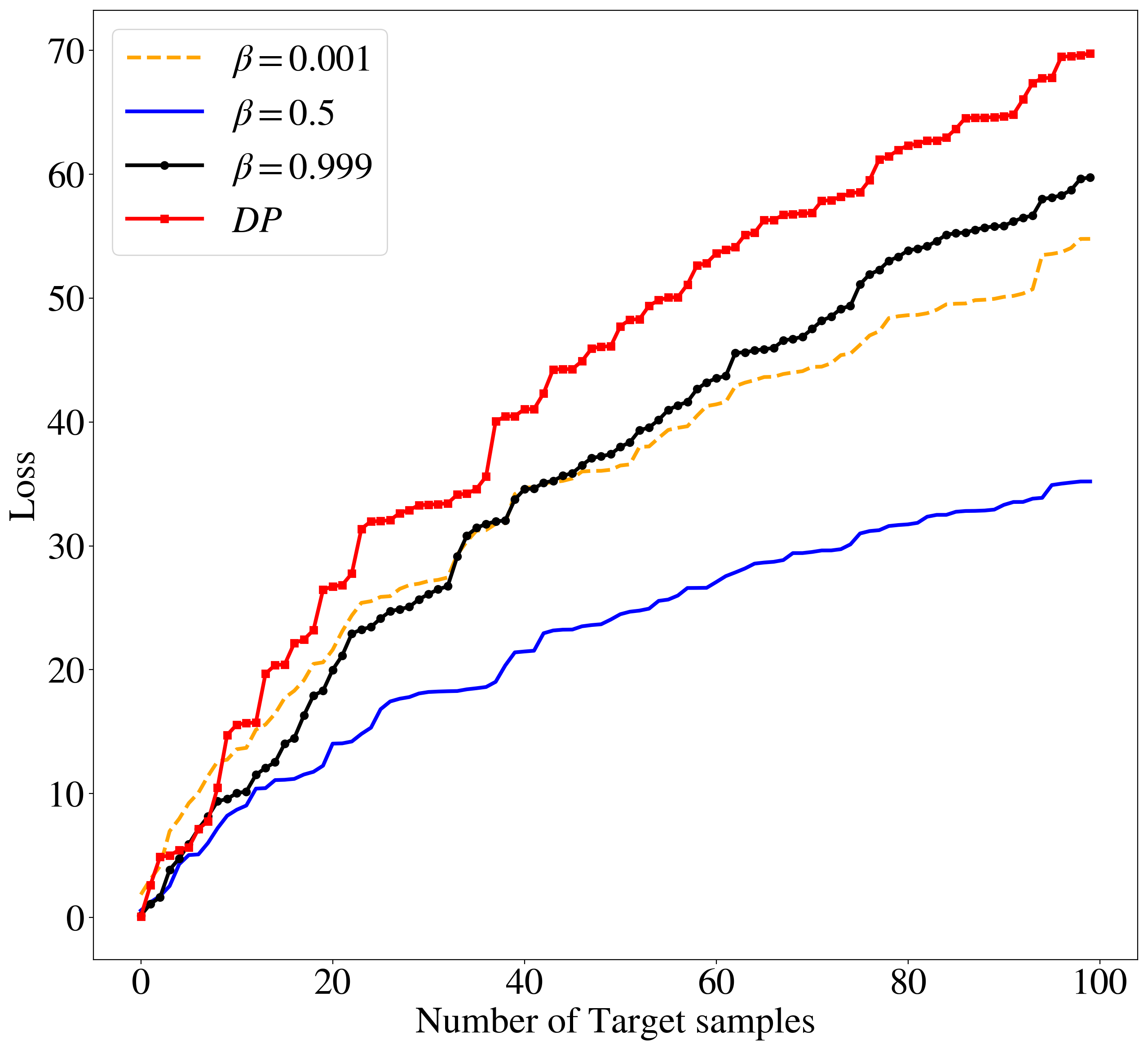}} \quad 
    \subfigure[M $\rightarrow$ U by Cross-Entropy Loss]{\includegraphics[width = 1.8in]{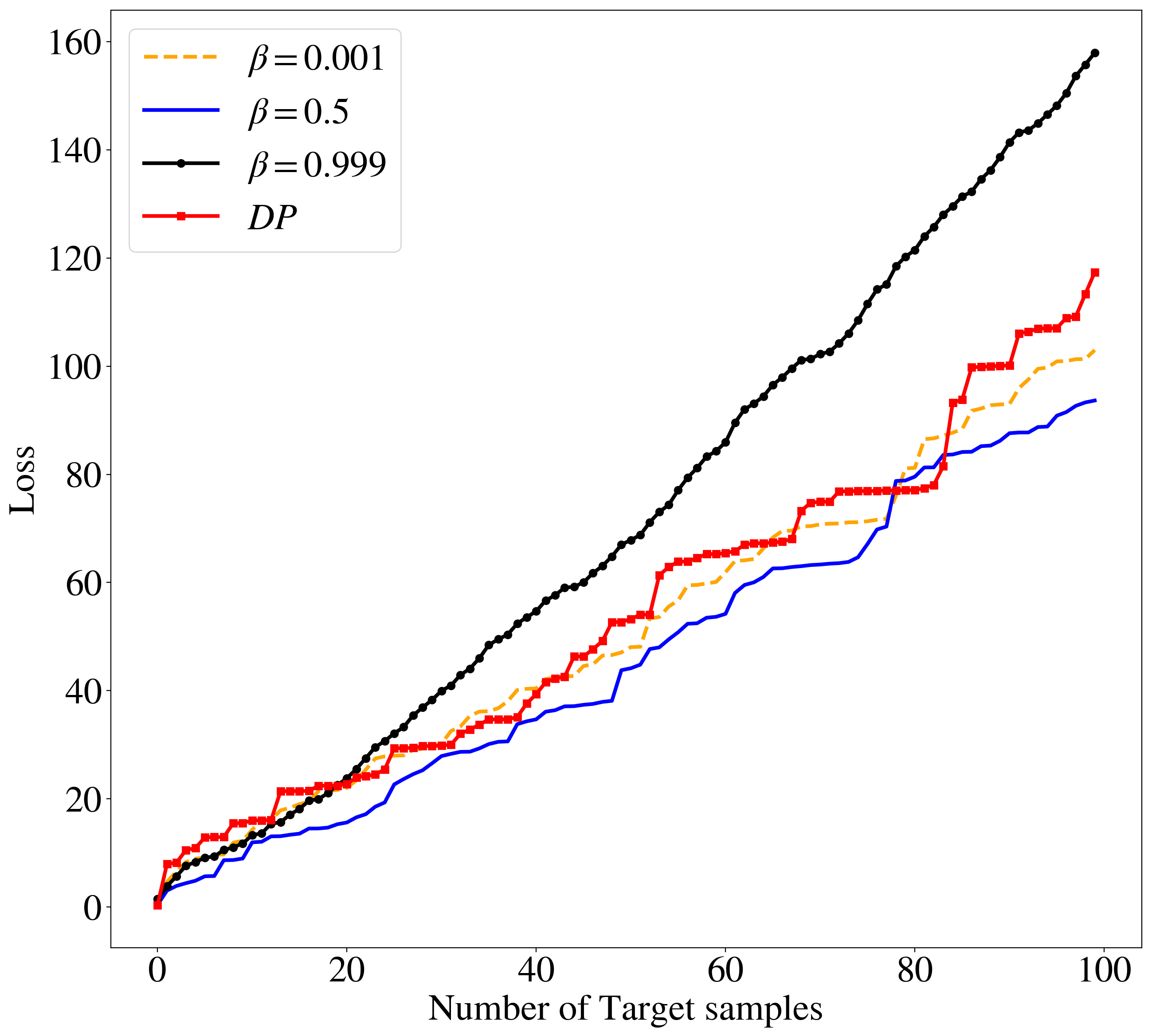}} \\
    \subfigure[U $\rightarrow$ M by Mistakes]{\includegraphics[width = 1.8in]{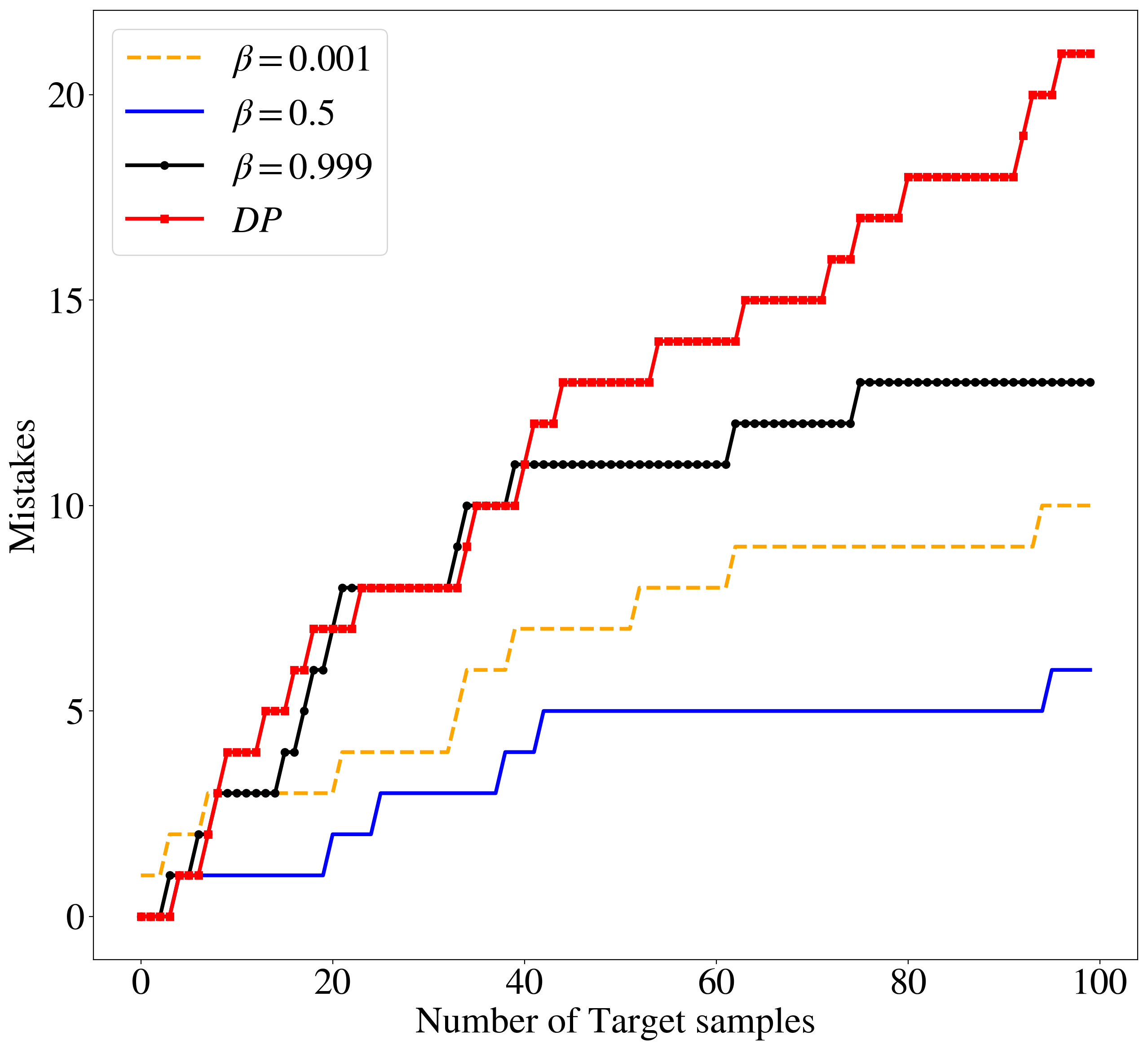}} \quad 
    \subfigure[M $\rightarrow$ U by Mistakes]{\includegraphics[width = 1.8in]{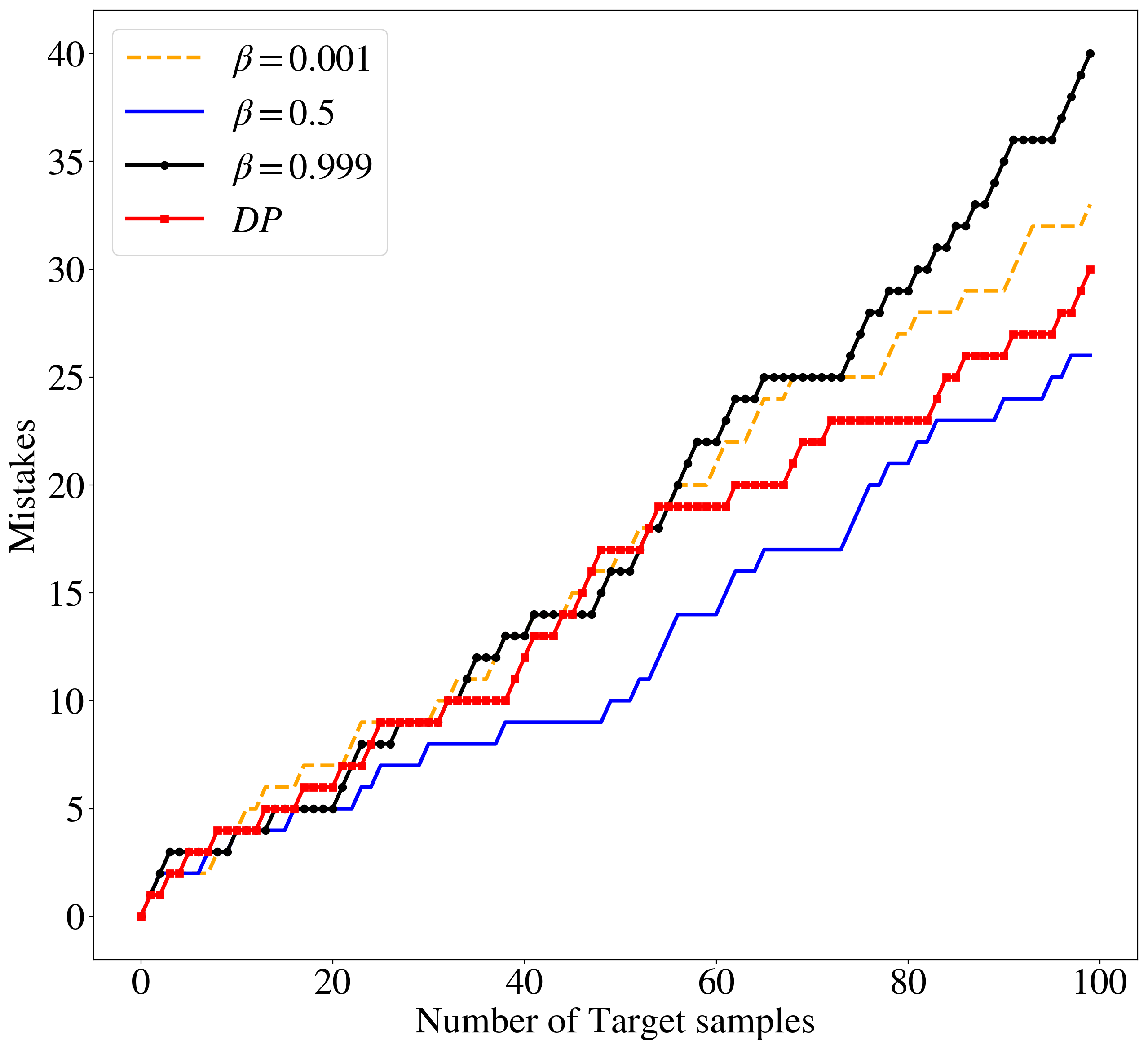}} 
    \caption{The comparisons on the prediction performance for the knowledge transfer between MNIST and USPS datasets by the cross-entropy and the mistakes. We use U $\rightarrow$ M (M $\rightarrow$ U) to denote the transfer from source domain USPS (MNIST) to the target domain MNIST (USPS).}
    \label{fig:comparison_mistake}
\end{figure}

From the experimental results, one can find that in both M $\rightarrow$ U and U $\rightarrow$ M settings, a very large $\beta$ (0.999) or a very small $\beta$ (0.001) may yield unsatisfactory predictive results compared to a moderate $\beta$ value (0.5), which achieves better performance than two extreme cases. Heuristically the performance mainly depends on how large the domain divergence is, or how large the target sample size is. Explicitly, if $\beta$ is too large, the prediction will rely more on the source domains and from which the negative transfer may occur if domain divergence is rather large. When it comes to a small $\beta$, as the prediction does not benefit from the source domain, the performance mainly relies on the target samples with a comparatively small size, leading to an undesirable loss and mistakes. With a moderate choice of $\beta$, the algorithm tries to strike a balance between the two extremes. Furthermore, with the baseline dpMNL method described in \cite{shahbaba2009nonlinear}, the performance is not desirable if we simply treat the source data and the target data equally as one domain. This is possibly due to that the source data is regarded as equally useful as the target data, while the domain discrepancy is not taken into account and thus cannot be alleviated in the learning and prediction, which will lead to a poor result. To achieve better performance, we may select a proper $\beta$ at a moderate level. That is, a moderate $\beta$ will properly extract the information from the limited target data and not overly trust the source data if the distribution differs. We empirical explore such a trade-off between $\beta$ and the performance, which exposes the inherent nature of transfer learning.

\section{Conclusion}\label{Sec7}
We propose a general framework for transfer learning from a Bayesian approach, a learning framework that extends traditional learning regimes to the case where the predictor is learned and deployed on samples drawn from different, yet related probability distributions in terms of parameterization. Specifically, the instantaneous, online and time-variant transfer learning scenarios are examined and the learning performance takes the shape of the conditional mutual information. We give the asymptotic estimation of the conditional mutual information and identify the situations when the negative and positive transfer will happen. However, in our analysis, the i.i.d. properties of both source and target data are crucial. The natural follow-up challenge is to formalize non i.i.d. settings for transfer learning and find similar mixture strategies for efficient utilities. Another future work is to relax the assumptions on parametric conditions to general probability distributions and rigorously find the performance guarantee under nonparameteric framework, which will improve the generality and applicability of our methods.

\section{Proofs}\label{Sec6}
In this section, we give the detailed proofs of the main theorems.

\subsection{Proof of Theorem~\ref{thm:excessrisk-log}}\label{proof:excessrisk-log}
\begin{proof}
We firstly show that given any prior over $\Theta_s$ and $\Theta_t$,
\begin{align}
   I(Z'_t;\Theta_t, \Theta_s|D^n_t, D^{m}_s)      &= I(\Theta_t, \Theta_s; Z'_t, D^n_t, D^{m}_s) - I(\Theta_t, \Theta_s ; D^n_s, D^{m}_s) \\
   &= D(P_{\Theta_t,\Theta_s}(D^n_t,D^{m}_s, Z'_t)\|Q(D^n_t,D^m_s, Z'_t)) - D(P_{\Theta_s,\Theta_t}(D^m_s, D^n_t)\|Q(D^m_s, D^n_t)) \\ 
   &=  \int \left( \mathbb{E}_{\theta_s,\theta_t} \left[ \log \frac{P_{\theta_t,\theta_s}(D^n_t,D^{m}_s, Z'_t)}{Q(D^n_t,D^{m}_s, Z'_t)} \right]  - \mathbb{E}_{\theta_s,\theta_t} \left[ \log \frac{P_{\theta_s,\theta_t}(D^m_s, D^n_t)}{Q(D^m_s, D^n_t)} \right]  \right) \omega(\theta_s,\theta_t) d\theta_s d\theta_t\\
   &= \int \left( \mathbb{E}_{\theta_s,\theta_t} \left[ \log \frac{P_{\theta_t}(Z'_t)}{Q(Z'_t|D^n_t, D^{m}_s)} \right] \right) \omega(\theta_s,\theta_t) d\theta_s d\theta_t,
\end{align}
where in the last equality we use the chain rule and the assumption that both source and target data are drawn in an i.i.d. fashion. The mutual information density at $\Theta_s = \theta^*_s$ and $\Theta_t = \theta^*_t$ is then given by
\begin{align}
     I(Z'_t ;\Theta_t = \theta^*_t, \Theta_s = \theta^*_s | D^n_t,D^{m}_s)  = \mathbb{E}_{\theta^*_s, \theta^*_t} \left[ \log \frac{P_{\theta^*_t}(Z'_t)}{Q(Z'_t|D^n_t, D^{m}_s)} \right] = \mathcal{R}_{I},
\end{align}
which completes the proof.
\end{proof}

\subsection{Proof of Theorem~\ref{thm:excessrisk-general}}
\begin{proof}
We can show that the excess risk for instantaneous transfer learning scenario can be bounded as
\begin{align}
   \mathcal{R}_I &= \mathbb{E}_{\theta^*_t,\theta^*_s} \left[  \ell(b, Z'_t) -  \ell(b^*, Z'_t) \right] \\
   &=   \mathbb{E}_{D^{m}_s,D^{n}_t}\mathbb{E}_{Z'_t}\left[\ell(b, Z'_t)  -  \ell(b^*,Z'_t)|D^m_s, D^{n}_t\right] \\
   &=   \mathbb{E}_{D^{m}_s,D^{n}_t}\int \left( \ell(b, z'_t)  -  \ell(b^*, z'_t) \right) P_{\theta^*_s,\theta^*_t}(z'_t|D^m_s, D^{n}_t) dz'_t \\
  &=   \mathbb{E}_{D^{m}_s,D^{n}_t}\int \left( \ell(b, z'_t)  -  \ell(b^*, z'_t) \right) (P_{\theta^*_s,\theta^*_t}(z'_t|D^m_s, D^{n}_t) - Q(z'_t|D^m_s, D^{n}_t) \noindent +  Q(z'_t|D^m_s, D^{n}_t)) dz'_t \\
  &\overset{(a)}{\leq} \mathbb{E}_{D^{m}_s,D^{n}_t}\int \left( \ell(b, z'_t)  -  \ell(b^*, z'_t) \right) (P_{\theta^*_s,\theta^*_t}(z'_t|D^m_s, D^{n}_t) - Q(z'_t|D^m_s, D^{n}_t)) dz'_t \\
  &\overset{(b)}{\leq}  M \mathbb{E}_{D^{m}_s,D^{n}_t}  \int (P_{\theta^*_s,\theta^*_t}(z'_t|D^m_s, D^{n}_t) - Q(z'_t|D^m_s, D^{n}_t)) dz'_t \\
  &\overset{(c)}{\leq}  M \mathbb{E}_{D^{m}_s,D^{n}_t}  \sqrt{2D\left(P_{\theta^*_s,\theta^*_t}(Z'_t|D^m_s, D^{n}_t) \| Q(Z'_t|D^m_s, D^{n}_t)\right)} \\
  &\overset{(d)}{\leq} M\sqrt{2 \mathbb{E}_{D^{m}_s,D^{n}_t} D\left(P_{\theta^*_s,\theta^*_t}(Z'_t|D^m_s, D^{n}_t) \| Q(Z'_t|D^m_s, D^{n}_t)\right)} \\
  &= M\sqrt{2D\left(P_{\theta^*_t} \| Q | D^m_s, D^{n}_t\right)} \\
  &= M\sqrt{2D(P_{\theta^*_t}(Z'_t) \| Q(Z'_t|D^n_t,D^m_s))} \\
  &= M\sqrt{2I(Z'_t;\Theta_t = \theta^*_t, \Theta_s = \theta^*_s|D^{m}_s,D^n_t) },
\end{align} 
where in $(a)$ we use the definition of $Q$, then $(b)$ holds since we assume the loss function is bounded, $(c)$ follows from the Pinsker's inequality, $(d)$ holds from the Jensen's inequality.
\end{proof}


\subsection{Proof of Theorem \ref{thm:expreg-log}} \label{proof:expreg-log}
\begin{proof}
We firstly show that given any prior over $\Theta_s$ and $\Theta_t$,
\begin{align}
   I(D^n_t;\Theta_t, \Theta_s|D^{m}_s)      &= I(\Theta_t, \Theta_s; D^n_t, D^{m}_s) - I(\Theta_s ; D^{m}_s) \\
   &= D(P_{\Theta_t,\Theta_s}(D^n_t,D^{m}_s)\|Q(D^n_t,D^m_s)) - D(P_{\Theta_s}(D^m_s)\|Q(D^m_s)) \\ 
   &=  \int \left( \mathbb{E}_{\theta_s,\theta_t} \left[ \log \frac{P_{\theta_t,\theta_s}(D^n_t,D^{m}_s)}{Q(D^n_t,D^{m}_s)} \right]  - \mathbb{E}_{\theta_s,\theta_t} \left[ \log \frac{P_{\theta_s}(D^m_s)}{Q(D^m_s)} \right]  \right) \omega(\theta_s,\theta_t) d\theta_s d\theta_t\\
   &= \int \left( \mathbb{E}_{\theta_s,\theta_t} \left[ \log \frac{P_{\theta_t}(D^n_t)}{Q(D^n_t|D^{m}_s)} \right] \right) \omega(\theta_s,\theta_t) d\theta_s d\theta_t,
\end{align}
where in the last equality we use the chain rule and the assumption that source data are independent of $\Theta_t$. The mutual information density at $\Theta_s = \theta^*_s$ and $\Theta_t = \theta^*_t$ is then given by
\begin{align}
     I(D^n_t;\Theta_t = \theta^*_t, \Theta_s = \theta^*_s | D^{m}_s)  = \mathbb{E}_{\theta^*_s, \theta^*_t} \left[ \log \frac{P_{\theta^*_t}(D^n_t)}{Q(D^n_t|D^{m}_s)} \right] = \mathcal{R}_{O}.
\end{align}
\end{proof}

\subsection{Proof of Theorem \ref{thm:expreg-generalloss}}\label{proof:expreg-generalloss}
\begin{proof}
We can show that the expected regret for online transfer learning scenario can be bounded as
\begin{align}
   \mathcal{R}_O &= \mathbb{E}_{\theta^*_t,\theta^*_s} \left[ \sum_{k=1}^{n} \ell(b_k, Z^{(k)}_t) - \sum_{k=1}^{n} \ell(b_k^*, Z^{(k)}_t) \right] \\
   &=  \sum_{k=1}^{n} \mathbb{E}_{D^{m}_s,D^{k-1}_t}\mathbb{E}_{Z^{(k)}_t}\left[\ell(b_k, Z^{(k)}_t)  -  \ell(b^*_k, Z^{(k)}_t)|D^m_s, D^{k-1}_t\right] \\
    &=  \sum_{k=1}^{n} \mathbb{E}_{D^{m}_s,D^{k-1}_t}\int \left( \ell(b_k, Z^{(k)}_t)  -  \ell(b^*_k, Z^{(k)}_t) \right) P_{\theta^*_s,\theta^*_t}(Z^{(k)}_t|D^m_s, D^{k-1}_t) dx^{(k)}_t \\
  &=  \sum_{k=1}^{n} \mathbb{E}_{D^{m}_s,D^{k-1}_t}\int \left( \ell(b_k, Z^{(k)}_t)  -  \ell(b^*_k, Z^{(k)}_t) \right) (P_{\theta^*_s,\theta^*_t}(Z^{(k)}_t|D^m_s, D^{k-1}_t) - Q(Z^{(k)}_t|D^m_s, D^{k-1}_t) \noindent \\
  & \quad +  Q(Z^{(k)}_t|D^m_s, D^{k-1}_t)) dx^{(k)}_t \\
  &\overset{(a)}{\leq} \sum_{k=1}^{n} \mathbb{E}_{D^{m}_s,D^{k-1}_t}\int \left( \ell(b_k, Z^{(k)}_t)  -  \ell(b^*_k, Z^{(k)}_t) \right) (P_{\theta^*_s,\theta^*_t}(Z^{(k)}_t|D^m_s, D^{k-1}_t) - Q(Z^{(k)}_t|D^m_s, D^{k-1}_t)) dx^{(k)}_t \\
  &\overset{(b)}{\leq} \sum_{k=1}^{n} \mathbb{E}_{D^{m}_s,D^{k-1}_t} M \int (P_{\theta^*_s,\theta^*_t}(Z^{(k)}_t|D^m_s, D^{k-1}_t) - Q(Z^{(k)}_t|D^m_s, D^{k-1}_t)) dx^{(k)}_t \\
  &\overset{(c)}{\leq} \sum_{k=1}^{n} \mathbb{E}_{D^{m}_s,D^{k-1}_t} M \sqrt{2D\left(P_{\theta^*_s,\theta^*_t}(Z^{(k)}_t|D^m_s, D^{k-1}_t) \| Q(Z^{(k)}_t|D^m_s, D^{k-1}_t)\right)} \\
  &\overset{(d)}{\leq} M\sum_{k=1}^{n}\sqrt{2 \mathbb{E}_{D^{m}_s,D^{k-1}_t} D\left(P_{\theta^*_s,\theta^*_t}(Z^{(k)}_t|D^m_s, D^{k-1}_t) \| Q(Z^{(k)}_t|D^m_s, D^{k-1}_t)\right)} \\
  &\overset{(e)}{=} M\sum_{k=1}^{n}\sqrt{2D\left(P_{\theta^*_t} \| Q | D^m_s, D^{k-1}_t\right)} \\
  &\overset{(f)}{\leq} Mn \sqrt{\frac{2}{n}\sum_{k=1}^{n} D\left(P_{\theta^*_t} \| Q | D^m_s, D^{k-1}_t\right)} \\
  &\overset{(g)}{=} M\sqrt{2n D(P_{\theta^*_t}(D^n_t) \| Q(D^n_t|D^m_s))} \\
  &= M\sqrt{2n  I(D^n_t;\Theta_t = \theta^*_t, \Theta_s = \theta^*_s|D^{m}_s) }, 
\end{align} 
where in $(a)$ we use the definition of $Q$, then $(b)$ holds since we assume the loss function is bounded, $(c)$ follows from the Pinsker's inequality, $(d)$ and $(f)$ follows from the Jensen's inequality, $(g)$ holds because of the chain rule of the KL divergence.
\end{proof}



\subsection{Proof of Theorem~\ref{coro:log-loss-time-variant}} \label{proof:log-loss-time-variant}
\begin{proof}
We firstly show that given any prior over $\Theta_s$, $\Theta_{t,i-1}$ and $\Theta_{t,i}$ for any episode $i$, we have
\begin{align}
   &I(D^{n_i}_{t,i};\Theta_{t,i},\Theta_{t,i-1}, \Theta_s|D^m_s, D^{n_{i-1}}_{t,i-1}) \\
   &= I(\Theta_{t,i},\Theta_{t,i-1}, \Theta_s; D^{n_i}_{t,i}, D^{n_{i-1}}_{t,i-1}, D^{m}_s) - I(\Theta_{t,i-1}, \Theta_s;  D^{n_{i-1}}_{t,i-1}, D^{m}_s) \\
   &= D(P_{\Theta_{t,i},\Theta_{t,i-1}, \Theta_s}(D^{n_i}_{t,i}, D^{n_{i-1}}_{t,i-1}, D^{m}_s)\|Q(D^{n_i}_{t,i}, D^{n_{i-1}}_{t,i-1}, D^{m}_s)) - D(P_{\Theta_{t,i-1}, \Theta_s}( D^{n_{i-1}}_{t,i-1}, D^{m}_s)\|Q( D^{n_{i-1}}_{t,i-1}, D^{m}_s)) \\ 
   &=  \int \Bigg( \mathbb{E}_{\theta_s,\theta_{t,i}, \theta_{t,i-1}} \left[ \log \frac{P_{\theta_s,\theta_{t,i}, \theta_{t,i-1}}(D^{n_i}_{t,i}, D^{n_{i-1}}_{t,i-1}, D^{m}_s)}{Q(D^{n_i}_{t,i}, D^{n_{i-1}}_{t,i-1}, D^{m}_s)} \right] \noindent \\
   & \quad - \mathbb{E}_{\theta_s,\theta_{t,i-1}} \left[ \log \frac{P_{\theta_s,\theta_{t,i-1}}(D^{n_{i-1}}_{t,i-1}, D^{m}_s)}{Q(D^{n_{i-1}}_{t,i-1}, D^{m}_s)} \right]  \Bigg) \omega(\theta_s,\theta_{t,i}, \theta_{t,i-1}) d\theta_s d\theta_{t,i-1} d\theta_{t,i}\\
   &= \int \left( \mathbb{E}_{\theta_s,\theta_{t,i}, \theta_{t,i-1}} \left[ \log \frac{P_{\theta_t}(D^{n_i}_{t,i})}{Q(D^{n_i}_{t,i}|D^{n_{i-1}}_{t,i-1}, D^{m}_s)} \right] \right) \omega(\theta_s,\theta_{t,i}, \theta_{t,i-1}) d\theta_s d\theta_{t,i-1} d\theta_{t,i},
\end{align}
where in the last equality we use the chain rule and the Assumption~\ref{asp:para-dist-timevarying}. Then the expected regret till episode $l$ can be expressed by the conditional mutual information evaluated at $\Theta_s = \theta^*_s$, $\Theta_{t,i-1} = \theta^*_{t,i-1}$ and $\Theta_{t,i} = \theta^*_{t,i}$ at each episode $i$:
\begin{align}
    \mathcal{R}_{TV}  = \sum_{i=1}^{l} I(D^{n_i}_{t,i};\Theta_{t,i} = \theta^*_{t,i},\Theta_{t,i-1} = \theta^*_{t,i-1}, \Theta_s = \theta^*_s|D^m_s, D^{n_{i-1}}_{t,i-1}).
\end{align}

\end{proof}

\subsection{Proof of Theorem~\ref{coro:general-loss-time-variant}}\label{proof:general-loss-time-variant}
\begin{proof}
 Under the conditions from Theorem~\ref{thm:expreg-generalloss}, we firstly give an upper bound of the expectation term at each episode $i$ as
\begin{align}
       \mathbb{E}_{\theta^*_s,\theta^*_{t,i},\theta^*_{t,i-1}} \left[\sum_{i=1}^{n_i} \ell\left(b_{i}, Z^{(i)}_{t,i}\right) - \sum_{i=1}^{n_i} \ell(b_i^*, Z^{(i)}_{t,i}) \right] \leq M\sqrt{2n_i I(D^{n_i}_{t,i};\Theta_{s} = \theta^*_{s}, \Theta_{t,i-1} = \theta^*_{t,i-1}, \Theta_{t,i} = \theta^*_{t,i}|D^m_s, D^{i-1}_t)},
\end{align}
where we define the conditional mutual information as
\begin{align}
    I(D^{n_i}_{t,i};\Theta_{s} = \theta^*_{s}, \Theta_{t,i-1} = \theta^*_{t,i-1}, \Theta_{t,i} = \theta^*_{t,i}|D^m_s, D^{i-1}_t) := \mathbb{E}_{\theta_s,\theta_{t,i}, \theta_{t,i-1}} \left[ \log \frac{P_{\theta_t}(D^{n_i}_{t,i})}{Q(D^{n_i}_{t,i}|D^{n_{i-1}}_{t,i-1}, D^{m}_s)} \right] 
\end{align}
with the mixture strategy $Q$. By Cauchy-Schwarz inequality,
\begin{align}
    \mathcal{R}_{TV} &= \sum_{i=1}^{l}\mathbb{E}_{\theta^*_s,\theta^*_{t,i},\theta^*_{t,i-1}} \left[\sum_{k=1}^{n_i} \ell\left(b_{k,i}, Z^{(k)}_{t,i}\right) - \sum_{k=1}^{n_i} \ell(b_{k,i}^*, Z^{(k)}_{t,i}) \right] \\
    &\leq \sum_{i=1}^{l} M\sqrt{2n_i I(D^{n_i}_{t,i};\Theta_{s} = \theta^*_{s}, \Theta_{t,i-1} = \theta^*_{t,i-1}, \Theta_{t,i} = \theta^*_{t,i}|D^m_s, D^{i-1}_t)} \\
    &\leq M\sqrt{2l\sum_{i=1}^{l} n_i I(D^{n_i}_{t,i};\Theta_{s} = \theta^*_{s}, \Theta_{t,i-1} = \theta^*_{t,i-1}, \Theta_{t,i} = \theta^*_{t,i}|D^m_s, D^{i-1}_t) },
\end{align}
which complete the proof.
\end{proof}

\noindent Before proving Theorem~\ref{thm:inst-scalar} and Theorem~\ref{thm:instant-gene}, let us prove the results for the OTL case first.

\subsection{Proof of Theorem \ref{thm:consistency-scalar}}\label{proof:conscalar}
\begin{proof}
We give the approximation on the KL divergence to see how the prior will affect the divergence,
\begin{align}
    \mathbb{E}_{\theta^*_s,\theta^*_t} \left[ \log \frac{P_{\theta^*_t}(D^n_t)}{Q(D^n_t| D^{m}_s)} \right] &= \mathbb{E}_{\theta^*_s,\theta^*_t} \left[ \log \frac{P_{\theta^*_t}(D^n_t)P_{\theta^*_s}(D^{m}_s)Q(D^m_s)}{Q(D^n_t,D^{m}_s)P_{\theta^*_s}(D^{m}_s)} \right] \\
    &= \mathbb{E}_{\theta^*_s,\theta^*_t} \left[ \log \frac{P_{\theta^*_t,\theta^*_s}(D^n_t,D^{m}_s)}{Q(D^n_t,D^{m}_s)} \right]  - \mathbb{E}_{\theta^*_s,\theta^*_t} \left[ \log \frac{P_{\theta^*_s}(D^m_s)}{Q(D^m_s)} \right] \\
    &= D(P_{\theta^*_t,\theta^*_s}(D^n_t,D^{m}_s)\|Q(D^n_t,D^m_s)) - D(P_{\theta^*_s}(D^m_s)\|Q(D^m_s)).
\end{align}
We can view that source samples and target samples are jointly sampled given the distribution $P_{\theta^*_s}$ and $P_{\theta^*_t}$. Using the results in \cite{clarke1990information} and \cite{clarke1999asymptotic}, with the proper prior $\omega(\theta_s,\theta_t)$ and parametric conditions, the asymptotic normality of the posterior implies that

\begin{align}
    D(P_{\theta^*_t,\theta^*_s}(D^n_t,D^{m}_s)\|Q(D^n_t,D^m_s)) - \frac{1}{2}\log\operatorname{det}\left( \begin{bmatrix}
    nI_{t}(\theta^*_t) & 0 \\
    0 & m I_{s}(\theta^*_s) 
    \end{bmatrix}
    \right) \rightarrow \log \frac{1}{2\pi e} + \log\frac{1}{\omega(\theta^*_s,\theta^*_t)},
\end{align}
as both $n$ and $m$ goes to infinity, where the fisher information matrices are denoted by
\begin{align}
I_{t}(\theta^*_t) &= - \mathbb{E}_{\theta^*_t}\left[\frac{\partial\log P(x|\theta^*_t)}{\partial \theta_t^2} \right], \\
I_{s}(\theta^*_s) &= - \mathbb{E}_{\theta^*_s}\left[\frac{\partial\log P(x|\theta^*_s)}{\partial \theta_s^2} \right]. 
\end{align}
Similarly,
\begin{align}
D(P_{\theta^*_s}(D^m_s)\|Q(D^m_s)) - \frac{1}{2}\log\operatorname{det}(m I_{s}(\theta^*_s)) \rightarrow \frac{1}{2} \log \frac{1}{2\pi e} + \log\frac{1}{\omega(\theta^*_s)},
\end{align}
as $m$ goes to infinity. Therefore,
\begin{align}
    \lim_{n,m\rightarrow \infty} \left( \mathbb{E}_{\theta^*_s,\theta^*_t} \left[ \log \frac{P_{\theta^*_t}(D^n_t)}{Q(D^n_t| D^{m}_s)} \right]  \right) = &  \frac{1}{2}\log\operatorname{det}\left( \begin{bmatrix}
    nI_{t}(\theta^*_t) & 0 \\
    0 & \alpha n I_{s}(\theta^*_s) 
    \end{bmatrix}
    \right) + \log \frac{1}{2\pi e} + \log\frac{1}{\omega(\theta^*_s,\theta^*_t)} \\
    & -  \frac{1}{2}\log\operatorname{det}(m I_{s}(\theta^*_s)) - \frac{1}{2} \log \frac{1}{2\pi e} - \log\frac{1}{\omega(\theta^*_s)} \\
    = &  \frac{1}{2}\log\operatorname{det}(n I_{t}(\theta^*_t)) + \frac{1}{2}\log\frac{1}{2\pi e} + \log\frac{1}{\omega(\theta^*_t|\theta^*_s)}.
\end{align}
To conclude, as both $n$ and $m$ goes to infinity, the conditional mutual information will converge to
\begin{align}
    I(D^n_t;\Theta_t = \theta^*_t, \Theta_s = \theta^*_s|D^{m}_s)  - \frac{1}{2}\log \frac{n}{2\pi e} \rightarrow \nonumber  \frac{1}{2}\log I_t(\theta^*_t)  + \log\frac{1}{\omega(\theta^*_t|\theta^*_s)}.
\end{align}

\end{proof}

\subsection{Proof of Theorem \ref{theorem:gene-para}}
\label{proof:gene-para}

\begin{proof}
By writing $\theta_s = (\theta_c, \theta_{sr})$ and $\theta_t = (\theta_c, \Theta_{tr,i})$, let us rewrite the conditional mutual information as
\begin{align}
    \mathbb{E}_{\theta^*_s,\theta^*_t} \left[ \log \frac{P_{\theta^*_t}(D^n_t)}{Q(D^n_t| D^{m}_s)} \right] &= \mathbb{E}_{\theta^*_s,\theta^*_t} \left[ \log \frac{P_{\theta^*_t}(D^n_t)P_{\theta^*_s}(D^{m}_s)Q(D^m_s)}{Q(D^n_t,D^{m}_s)P_{\theta^*_s}(D^{m}_s)} \right] \\
    &= \mathbb{E}_{\theta^*_{c},\theta^*_{sr}, \theta^*_{tr}} \left[ \log \frac{P_{\theta^*_t,\theta^*_s}(D^n_t,D^{m}_s)}{Q(D^n_t,D^{m}_s)} \right]  - \mathbb{E}_{\theta^*_{c},\theta^*_{sr}} \left[ \log \frac{P_{\theta^*_s}(D^m_s)}{Q(D^m_s)} \right] \\
    &= D(P_{\theta^*_c,\theta^*_{sr}, \theta^*_{tr}}(D^n_t,D^{m}_s)\|P(D^n_t,D^m_s)) - D(P_{\theta^*_c, \theta^*_{sr}}(D^m_s)\|Q(D^m_s)).
\end{align}
We view that $m$ source samples and $n$ target samples are jointly sampled from the distribution parametrized by the parameters $\boldsymbol{\Theta} = (\Theta_c, \Theta_{sr},\Theta_{tr,i})$. At time $n$, we have the asymptotic approximation under the proper prior and assumption 2 using Theorem 2.1 in \cite{clarke1990information} and \cite{clarke1999asymptotic} as
\begin{align}
    D(P_{\theta^*_c,\theta^*_{sr}, \theta^*_{tr}}(D^n_t,D^{m}_s)\|P(D^n_t,D^m_s)) - \frac{1}{2}\log\operatorname{det}\left( \mathbf{I}_{\boldsymbol \theta^*}
    \right) \rightarrow \frac{2d - j}{2}\log \frac{1}{2\pi e} + \log\frac{1}{\omega(\theta^*_s,\theta^*_t)},
\end{align}
where the Fisher information matrix is defined as
\begin{align}
    \mathbf{I}_{\boldsymbol \theta^*} &= \begin{bmatrix} 
    m I_{cs}({\theta}^*_c) + n I_{ct}({\theta}^*_c) & m I_{cs}({\theta}^*_c,{\theta}^*_{sr}) & n I_{ct}({\theta}^*_c,{\theta}^*_{tr}) \\
    m I^T_{cs}(\theta^*_c,{\theta}^*_{sr}) & m I_s({\theta}^*_{sr}) & \mathbf{0} \\
    nI^T_{ct}({\theta}^*_c,{\theta}^*_{tr}) & \mathbf{0} & n I_t({\theta}^*_{tr})
    \end{bmatrix} 
\end{align}
Similarly,
\begin{align}
    D(P_{\theta^*_c, \theta^*_{sr}}(D^m_s)\|Q(D^m_s)) - \frac{1}{2}\log\operatorname{det}\left( \mathbf{I}_{\boldsymbol \theta^*_s}
    \right) \rightarrow \frac{d}{2}\log \frac{1}{2\pi e} + \log\frac{1}{\omega(\theta^*_s)},
\end{align}
where
\begin{align}
    \mathbf{I}_{\boldsymbol \theta^*_s} &= m\begin{bmatrix}
    I_{cs}({\theta}^*_c) & I_{cs}({\theta}^*_c, {\theta}^*_{sr}) \\ 
    I^T_{cs}({\theta}^*_c,{\theta}^*_{sr}) & I_s({\theta}^*_{sr})
    \end{bmatrix}
\end{align}
as $m$ goes to sufficiently large. As a consequence,
\begin{align}
    \lim_{n,m\rightarrow \infty} \left( \mathbb{E}_{\theta^*_s,\theta^*_t} \left[ \log \frac{P_{\theta^*_t}(D^n_t)}{Q(D^n_t| D^{m}_s)} \right]  \right) = &  \frac{1}{2}\log\operatorname{det}\left( \mathbf{I}_{\boldsymbol \theta^*} \right) + \frac{2d-j}{2}\log \frac{1}{2\pi e} + \log\frac{1}{\omega(\theta^*_s,\theta^*_t)} \noindent \\
    & -  \frac{1}{2}\log\operatorname{det}(\mathbf{I}_{\boldsymbol \theta^*_s}) - \frac{d}{2} \log \frac{1}{2\pi e} - \log\frac{1}{\omega(\theta^*_s)} \\
    = &  \frac{1}{2}\log \frac{\operatorname{det}({\mathbf{I}_{\boldsymbol \theta^*}})}{\operatorname{det}({\mathbf{I}_{\boldsymbol {\theta^*_s}}})} + \frac{d - j}{2}\log\frac{1}{2\pi e} + \log\frac{1}{\omega(\theta^*_t|\theta^*_s)}.
\end{align}
Let us examine the ratio of the determinant, using the block determinant results from~\cite{powell2011calculating},
\begin{align}
\log \frac{\operatorname{det}({\mathbf{I}_{\boldsymbol \theta^*}})}{\operatorname{det}({\mathbf{I}_{\boldsymbol {\theta^*_s}}})} = & \log \operatorname{det}\left(    m I_{cs}({\theta}^*_c) + n I_{ct}({\theta}^*_c) -
\begin{bmatrix} 
m I_{cs}({\theta}^*_c,{\theta}^*_{sr}) & n I_{ct}({\theta}^*_c,{\theta}^*_{tr}) 
\end{bmatrix} \begin{bmatrix} 
 m I_s({\theta}^*_{sr}) & \mathbf{0} \nonumber \\
 \mathbf{0} & n I_t({\theta}^*_{tr})
\end{bmatrix}^{-1} \begin{bmatrix}
m I^T_{cs}({\theta}^*_c,{\theta}^*_{sr})  \\
n I^T_{ct}({\theta}^*_c,{\theta}^*_{tr}) 
\end{bmatrix}
\right) \\
& + \log \operatorname{det}\left(\begin{bmatrix}  m I_s({\theta}^*_{sr}) & \mathbf{0} \\
\mathbf{0} & nI_t({\theta}^*_{tr})
\end{bmatrix} \right) - \log \operatorname{det}\left( m(I_{cs}({\theta}^*_c) - I_{cs}({\theta}^*_c, {\theta}^*_{sr})I^{-1}_s({\theta}^*_{sr}) I^T_{cs}({\theta}^*_c, {\theta}^*_{sr}) )\right) \nonumber \\
& - \log \operatorname{det}\left( m I_s({\theta}^*_{sr}) \right) \\
= & \log \operatorname{det}(m\Delta_s + n\Delta_t) - \log\operatorname{det}(m\Delta_s) + \log\operatorname{det}(nI_t({\theta}^*_{tr}) \\
= & \log \operatorname{det}(\mathbf{I}_{j} + \frac{n}{m}\Delta_t\Delta^{-1}_s) + \log\operatorname{det}(nI_t({\theta}^*_{tr}),
\end{align}
where we define $\Delta_s = I_{cs}(\theta^*_c) - I_{cs}(\theta^*_c,\theta^{*}_{sr}) I^{-1}_s(\theta^{*}_{sr}) I^{T}_{cs}(\theta^*_c,\theta^{*}_{sr})$ and $\Delta_t = I_{ct}(\theta^*_c) - I_{ct}(\theta^*_c,\theta^{*}_{tr}) I^{-1}_t(\theta^{*}_{tr}) I^{T}_{ct}(\theta^*_c,\theta^{*}_{tr}) $. $\mathbf{I}_{j}$ denotes the identity matrix with size $j$ and $\boldsymbol \theta^* = (\theta^*_c, \theta^{*}_{sr}, \theta^{*}_{tr})$ denotes the true parameters. With a little abuse of notation, we define the Fisher information matrices as 
\begin{align}
I_{cs}(\theta^*_c) &= - \mathbb{E}_{\theta^*_s}\left[ \nabla^2_{\Theta_c}  \log P(Z_s |\Theta_c, \theta^*_{sr}) \right]\Big|_{\Theta_c = \theta^*_c} \in \mathbb{R}^{j\times j}, \label{eq:fisher-start} \\
I_{ct}(\theta^*_c) &= - \mathbb{E}_{\theta^*_t}\left[ \nabla^2_{\Theta_c}  \log P(Z_t |\Theta_c, \theta^*_{tr}) \right]\Big|_{\Theta_c = \theta^*_c} \in \mathbb{R}^{j\times j},\\
I_{s}(\theta^*_{sr}) &= - \mathbb{E}_{\theta^*_s}\left[ \nabla^2_{\Theta_{sr}}  \log P(Z_s |\theta^*_c, \Theta_{sr})\right]\Big|_{\Theta_{sr} = \theta^*_{sr}} \in \mathbb{R}^{(d-j) \times (d-j)}, \\
I_{t}(\theta^*_{tr}) &= - \mathbb{E}_{\theta^*_t}\left[ \nabla^2_{\Theta_{tr}}  \log P(Z_t |\theta^*_c, \Theta_{tr})\right]\Big|_{\Theta_{tr} = \theta^*_{tr}} \in \mathbb{R}^{(d-j) \times (d-j)},\\
I_{cs}(\theta^*_c, \theta^*_{sr}) &= -\mathbb{E}_{\theta^*_s}\left[\frac{\partial\log P(Z_s |\theta^*_c, \theta^*_{sr})}{  \partial \Theta_{c,i} \partial \Theta_{sr,k}} \right]_{ \begin{matrix}
i = 1,\cdots, j,\\
 k = 1,\cdots, d-j
 \end{matrix}} \in \mathbb{R}^{j\times (d-j)}, \\
I_{ct}(\theta^*_c, \theta^*_{tr}) &= - \mathbb{E}_{\theta^*_t}\left[\frac{\partial\log P(Z_t |\theta^*_c, \theta^*_{tr})}{\partial \Theta_{c,i} \partial \Theta_{tr,k}} \right]_{ \begin{matrix}
i = 1,\cdots, j,\\
 k = 1,\cdots, d-j
 \end{matrix}} \in \mathbb{R}^{j \times (d-j)}. \label{eq:fisher-end}
\end{align} 
Putting everything together, by setting $m = cn^p$ we reach
\begin{align}
I(D^n_t;\Theta_t = \theta^*_t, \Theta_s = \theta^*_s|D^{m}_s)   - \frac{1}{2} \log \operatorname{det}(n I_t(\theta^{*}_{tr}))  - \frac{1}{2} \log \operatorname{det}(\mathbf{I}_{j} + \frac{1}{cn^{p-1}} \Delta_t \Delta^{-1}_s)\rightarrow \frac{d-j}{2}\log \frac{1}{2\pi e} + \log \frac{1}{\omega(\theta^*_t|\theta^*_s)}, 
\end{align}
as both $n$ and $m$ goes to infinity.
\end{proof}

\noindent Now we can use the results of Theorem~\ref{thm:expreg-log} and Theorem~\ref{thm:expreg-generalloss} to proof Theorem~\ref{thm:inst-scalar} and Theorem~\ref{thm:instant-gene} for ITL case.

\subsection{Proof of Theorem~\ref{thm:inst-scalar}} \label{proof:inst-scalar}
\begin{proof}
From Theorem~\ref{thm:consistency-scalar}, we have that
\begin{align}
     I(D^n_t;\Theta_t = \theta^*_t, \Theta_s = \theta^*_s|D^{m}_s) - \frac{1}{2}\log \frac{n}{2\pi e} \rightarrow   \frac{1}{2}\log I_t(\theta^*_t)  + \log\frac{1}{\omega(\theta^*_t|\theta^*_s)}.
\end{align}
By the i.i.d. property, additionally we have
\begin{align}
     I(Z'_t, D^n_t;\Theta_t = \theta^*_t, \Theta_s = \theta^*_s|D^{m}_s) - \frac{1}{2}\log \frac{n+1}{2\pi e} \rightarrow \frac{1}{2}\log I_t(\theta^*_t)  + \log\frac{1}{\omega(\theta^*_t|\theta^*_s)}.
\end{align}
Therefore as both $m$ and $n$ go to sufficiently large, the instantaneous prediction yields
\begin{align}
I(Z'_t;\Theta_t = \theta^*_t, \Theta_s = \theta^*_s|D^n_t, D^{m}_s) &=  I(Z'_t, D^n_t;\Theta_t = \theta^*_t, \Theta_s = \theta^*_s|D^{m}_s) - I(D^n_t;\Theta_t = \theta^*_t, \Theta_s = \theta^*_s|D^{m}_s) \\
&= \frac{1}{2}\log \frac{n+1}{2\pi e} -  \frac{1}{2}\log \frac{n}{2\pi e} \\
&= \frac{1}{2}\log (1 + \frac{1}{n}) \\
&\asymp \frac{1}{n},
\end{align}
which is the typical result for the optimal rate of $O(\frac{1}{n})$.
\end{proof}

\subsection{Proof of Theorem~\ref{thm:instant-gene}} \label{proof:instant-gene}
\begin{proof}
From Theorem~\ref{theorem:gene-para}, as we assume $m = cn^p$ for some positive constant $c$, the asymptotic approximation of the expected regret is expressed as
\begin{align}
I(D^n_t;\Theta_t = \theta^*_t, \Theta_s = \theta^*_s|D^{m}_s)   - \frac{1}{2} \log \operatorname{det}(n I_t(\theta^{*}_{tr}))  - \frac{1}{2} \log \operatorname{det}(\mathbf{I}_{j} + \frac{n}{cn^p} \Delta_t \Delta^{-1}_s)\rightarrow \frac{d-j}{2}\log \frac{1}{2\pi e} + \log \frac{1}{\omega(\theta^*_t|\theta^*_s)}. 
\end{align}
If we take $Z'_t$ take into consideration, we similarly have
\begin{align}
I(Z'_t, D^n_t;\Theta_t = \theta^*_t, \Theta_s = \theta^*_s|D^{m}_s)   & - \frac{1}{2} \log \operatorname{det}((n+1) I_t(\theta^{*}_{tr}))  - \frac{1}{2} \log \operatorname{det}(\mathbf{I}_{j} + \frac{n+1}{c n^p} \Delta_t \Delta^{-1}_s) \noindent  \\
& \rightarrow \frac{d-j}{2}\log \frac{1}{2\pi e} + \log \frac{1}{\omega(\theta^*_t|\theta^*_s)}. 
\end{align}
As a consequence, when both $n$ and $m$ go to infinity,
\begin{align}
& I(Z'_t;\Theta_t = \theta^*_t, \Theta_s = \theta^*_s|D^n_t, D^{m}_s) =  I(Z'_t, D^n_t;\Theta_t = \theta^*_t, \Theta_s = \theta^*_s|D^{m}_s) - I(D^n_t;\Theta_t = \theta^*_t, \Theta_s = \theta^*_s|D^{m}_s) \\
&= \frac{d - j}{2}\log (1 + \frac{1}{n}) +  \frac{1}{2}\log\det(cn^p\Delta_s + (n+1)\Delta_t) - \frac{1}{2}\log\det(cn^p\Delta_s + n\Delta_t)  
\end{align}
We will use the expansion of determinant:
\begin{align}
    \operatorname{det}(\mathbf{I}+\frac{1}{n^{k}} A) = 1 + \frac{1}{n^{k}} \operatorname{Tr}(A)+o(1 / n^{k}).
\end{align}
For $0 \leq p < 1$:
\begin{align}
    &I(Z'_t;\Theta_t = \theta^*_t, \Theta_s = \theta^*_s|D^n_t, D^{m}_s)  = \frac{d - j}{2}\log (1 + \frac{1}{n}) +  \frac{1}{2}\log\det(c n^p\Delta_s + (n+1)\Delta_t) - \frac{1}{2}\log\det(c n^p\Delta_s + n\Delta_t)   \\
    &= \frac{d - j}{2}\log (1 + \frac{1}{n}) + \frac{j}{2}\log(1+\frac{1}{n}) + \frac{1}{2}\log\det(\mathbf{I}_j + \frac{c n^p}{n+1}\Delta_s\Delta^{-1}_t) - \frac{1}{2}\log\det(\mathbf{I}_j + \frac{cn^p}{n}\Delta_s\Delta^{-1}_t) \\
    &= \frac{d - j}{2}\log (1 + \frac{1}{n}) + \frac{j}{2}\log(1+\frac{1}{n}) + \frac{1}{2}\frac{c n^p}{n+1} \operatorname{Tr}(\Delta_s\Delta^{-1}_t) - \frac{1}{2}\frac{c n^p}{n}\operatorname{Tr}(\Delta_s\Delta^{-1}_t)  + o(\frac{1}{n}) \\
    &= \frac{d - j}{2}\log (1 + \frac{1}{n}) + \frac{j}{2}\log(1+\frac{1}{n}) + \frac{1}{2}\frac{c n^p}{n(n+1)}\operatorname{Tr}(\Delta_s\Delta^{-1}_t) + o(\frac{1}{n}) \\
    &\asymp \frac{d}{n}
\end{align}
For $p \geq 1$:
\begin{align}
    &I(Z'_t;\Theta_t = \theta^*_t, \Theta_s = \theta^*_s|D^n_t, D^{m}_s)  = \frac{d - j}{2}\log (1 + \frac{1}{n}) +  \frac{1}{2}\log\det(c n^p\Delta_s + (n+1)\Delta_t) - \frac{1}{2}\log\det(c n^p\Delta_s + n\Delta_t)   \\
    &= \frac{d - j}{2}\log (1 + \frac{1}{n}) + \frac{1}{2}\log\det(\mathbf{I}_j + \frac{n+1}{cn^p}\Delta_t\Delta^{-1}_s) - \frac{1}{2}\log\det(\mathbf{I}_j + \frac{n}{cn^p}\Delta_t\Delta^{-1}_s) \\
    &=  \frac{d - j}{2}\log (1 + \frac{1}{n}) + \frac{1}{2}\frac{n+1}{c n^p} \operatorname{Tr}(\Delta_t\Delta^{-1}_s) - \frac{1}{2}\frac{n}{c n^p}\operatorname{Tr}(\Delta_t\Delta^{-1}_s)  + o(\frac{1}{n^p}) \\
    &= \frac{d - j}{2}\log (1 + \frac{1}{n}) + \frac{1}{2c n^p}\operatorname{Tr}(\Delta_t\Delta^{-1}_s)  + o(\frac{1}{n^p}) \\
    &\asymp  \frac{d-j}{n} + \frac{j}{n^p}
\end{align}
To conclude,
\begin{equation}
 \mathcal{R}_I  \asymp \frac{d - j}{n} + \frac{j}{n \vee n^p }  , 
\end{equation}
which completes the proof.
\end{proof}

\subsection{Proof of Corollary~\ref{coro:rate-gene-online}} \label{proof:rate-gene-online}
\begin{proof}
From Theorem~\ref{theorem:gene-para}, we have
\begin{align}
I(D^n_t;\Theta_t = \theta^*_t, \Theta_s = \theta^*_s|D^{m}_s)   = \frac{1}{2} \log \operatorname{det}(n I_t(\theta^{*}_{tr}))  + \frac{1}{2} \log \operatorname{det}(\mathbf{I}_{j} + \frac{n}{m} \Delta_t \Delta^{-1}_s) +  O(1).
\end{align}
This quantity depends on the source sample size $m = cn^p$, assume $m = c n^p$ with some positive constant $c$. For $0 \leq p < 1$:
\begin{align}
    & I(D^n_t;\Theta_t = \theta^*_t, \Theta_s = \theta^*_s|D^{m}_s)   = \frac{1}{2} \log \operatorname{det}(n I_t(\theta^{*}_{tr}))  + \frac{1}{2} \log \operatorname{det}(cn^p \Delta_s + n \Delta_t ) -  \frac{1}{2} \log \operatorname{det}(cn^p \Delta_s)  +  O(1) \\
    &= \frac{1}{2} \log \operatorname{det}(n I_t(\theta^{*}_{tr})) + \frac{j}{2}\log n - \frac{jp}{2} \log n + \frac{1}{2}\frac{1}{n^{1-p}} \operatorname{Tr}(c\Delta_s\Delta^{-1}_t) - \frac{1}{2}\log\operatorname{det} (c\Delta_s)+  O(1) \\
    &\asymp (d-j)\log n + j(1-p)\log n.
\end{align}
For $p \geq 1$:
\begin{align}
    & I(D^n_t;\Theta_t = \theta^*_t, \Theta_s = \theta^*_s|D^{m}_s)   = \frac{1}{2} \log \operatorname{det}(n I_t(\theta^{*}_{tr}))  + \frac{1}{2} \log \operatorname{det}(c n^p \Delta_s + n \Delta_t ) -  \frac{1}{2} \log \operatorname{det}(c n^p \Delta_s)  +  O(1) \\
    &= \frac{1}{2} \log \operatorname{det}(n I_t(\theta^{*}_{tr})) + \frac{1}{2}\log \operatorname{det}(\mathbf{I}_j + \frac{1}{cn^{p-1}}\Delta_t\Delta^{-1}_s) + O(1) \\
    &= \frac{1}{2} \log \operatorname{det}(n I_t(\theta^{*}_{tr})) + \frac{1}{2}\frac{1}{cn^{p-1}}\operatorname{Tr}(\Delta_t\Delta^{-1}_s) + O(1) + o(\frac{1}{n^{p-1}}) \\
    &\asymp (d-j)\log n + \frac{j}{n^{p-1}}.
\end{align}
which completes the proof.
\end{proof}

\subsection*{Proof of Theorem~\ref{thm:time-varying}} \label{proof:time-varying}
\begin{proof}
Let us give an asymptotic estimation on $I(D^{n_i}_{t,i};\Theta_{t,i} = \theta^*_{t,i}|D^m_s, D^{n_i}_{t,i})$ as all $n_i$, $n_{i-1}$ and $m$ are large enough shown in the following theorem.
\begin{theorem}
 Under Assumptions~1 and 2, with $\Theta_s,\Theta_{t,i-1}. \Theta_{t,i} \in \mathbb{R}^d$ defined in the paper and as $n_{i-1}, n_i, m \rightarrow \infty$, the mixture strategy with proper prior $\omega(\Theta_s,\Theta_{t,i-1}, \Theta_{t,i})$ yields
 \begin{small}
\begin{align}
I(D^{n_i}_{t,i};\Theta_{t,i} = \theta^*_{t,i}|D^m_s, D^{n_{i-1}}_{t,i-1})  - &  \log\operatorname{det}\left(\mathbf{I}_{j_i \times j_i} + \frac{n_i}{m+n_{i-1}}\Delta_{ct,i}\Delta^{-1}_{cst,i} \right)  - \log \operatorname{det}(\mathbf{I}_{c_i\times c_i} + \frac{n_i}{n_{i-1}} \Delta_{t,i} \Delta^{-1}_{t,i-1}) - \log\operatorname{det}(n_i I_{t,i}(\theta^*_{tr,i})) \nonumber  \\
&\rightarrow  (d - j_i - c_i)\log\frac{1}{2\pi e} + \frac{2}{\omega(\theta^*_{t,i}|\theta^*_{t,i-1},\theta^*_s)}.
\end{align}
 \end{small}
\end{theorem}
\begin{proof}
Rewrite the conditional mutual information in terms of the KL divergence as
\begin{align}
    &I(D^{n_i}_{t,i};\Theta_{t,i} = \theta^*_{t,i}|D^m_s, D^{n_{i-1}}_{t,i-1}) = \mathbb{E}_{\theta^*_s,\theta^*_{t,i},\theta^*_{t,i-1}}\left[\log \frac{P_{\theta^{*}_{t,i}}(D^{n_i}_{t,i})}{Q(D^{n_i}_{t,i}|D^{n_{i-1}}_{t,i-1}, D^m_s)}\right] \\
    &= D(P_{\theta^*_s,\theta^*_{t,i},\theta^*_{t,i-1}}(D^m_s, D^{n_{i-1}}_{t,i-1},D^{n_i}_{t,i})\|Q(D^m_s, D^{n_{i-1}}_{t,i-1},D^{n_i}_{t,i})) -  D(P_{\theta^*_s,\theta^*_{t,i-1}}(D^m_s, D^{n_{i-1}}_{t,i-1})\|Q(D^m_s, D^{n_{i-1}}_{t,i-1})).
\end{align}
Let us align the parameter as $\boldsymbol{\Theta} = (\Theta_{c,i}, \Theta_{v,i}, \Theta_{sr,i}, \Theta_{tr,i-1}, \Theta_{tr,i})$ and $\boldsymbol{\Theta}_{s} = (\Theta_{c,i}, \Theta_{v,i}, \Theta_{sr,i}, \Theta_{tr,i-1})$ and we define a set of Fisher information matrices as
\begin{align}
I_{cs,i}(\theta^*_{c,i}) &= - \mathbb{E}_{\theta^*_s}\left[ \nabla^2_{\Theta_{c,i}}  \log P(Z_s |\Theta_{c,i}, \theta^*_{sr,i}) \right]\Big|_{\Theta_{c,i} = \theta^*_{c,i}} \in \mathbb{R}^{j_i\times j_i}, \\
I_{s,i}(\theta^*_{sr,i}) &= - \mathbb{E}_{\theta^*_s}\left[ \nabla^2_{\Theta_{sr,i}}  \log P(Z_s |\theta^*_{c,i}, \Theta_{sr,i})\right]\Big|_{\Theta_{sr,i} = \theta^*_{sr,i}} \in \mathbb{R}^{(d-j_i) \times (d-j_i)}, \\
I_{ct,i-1}(\theta^*_{c,i}) &= - \mathbb{E}_{\theta^*_{t,i-1}}\left[ \nabla^2_{\Theta_{c,i}}  \log P(Z_{t,i-1} |\Theta_{c,i}, \theta^*_{v,i}, \theta^*_{tr,i-1}) \right]\Big|_{\Theta_{c,i} = \theta^*_{c,i}} \in \mathbb{R}^{j_i\times j_i},\\
I_{ct,i}(\theta^*_{c,i}) &= - \mathbb{E}_{\theta^*_t}\left[ \nabla^2_{\Theta_{c,i}}  \log P(Z_{t,i} |\Theta_{c,i}, \theta^*_{v,i}, \theta^*_{tr,i}) \right]\Big|_{\Theta_{c,i} = \theta^*_{c,i}} \in \mathbb{R}^{j_i\times j_i},\\
I_{t,i-1}(\theta^*_{tr,i-1}) &= - \mathbb{E}_{\theta^*_{t,i-1}}\left[ \nabla^2_{\Theta_{tr,i-1}}  \log P(Z_{t,i-1} |\theta^*_{c,i}, \theta^*_{v,i}, \Theta_{tr,i-1})\right]\Big|_{\Theta_{tr,i-1} = \theta^*_{tr,i-1}} \in \mathbb{R}^{(d-c_i - j_i) \times (d- c_i - j_i)},\\
I_{t,i}(\theta^*_{tr,i}) &= - \mathbb{E}_{\theta^*_t}\left[ \nabla^2_{\Theta_{tr,i}}  \log P(Z_{t,i} |\theta^*_{c,i}, \Theta_{tr,i})\right]\Big|_{\Theta_{tr,i} = \theta^*_{tr}} \in \mathbb{R}^{(d-c_i - j_i) \times (d- c_i - j_i)},\\
I_{v_{i-1}}(\theta^*_{v,i}) &= - \mathbb{E}_{\theta^*_{t,i-1}}\left[ \nabla^2_{\Theta_{v,i}}  \log P(Z_{t,i-1} |\theta^*_{c,i}, \Theta_{v,i}, \theta^*_{tr,i-1}) \right]\Big|_{\Theta_{v,i} = \theta^*_{v,i}} \in \mathbb{R}^{c_i \times c_i}, \\
I_{v_i}(\theta^*_{v,i}) &= - \mathbb{E}_{\theta^*_{t,i}}\left[ \nabla^2_{\Theta_{v,i}}  \log P(Z_{t,i} |\theta^*_{c,i}, \Theta_{v,i}, \theta^*_{tr,i}) \right]\Big|_{\Theta_{v,i} = \theta^*_{v,i}} \in \mathbb{R}^{c_i \times c_i}, \\
I_{sc}(\theta^*_{c,i}, \theta^*_{sr,i}) &= -\mathbb{E}_{\theta^*_s}\left[\frac{\partial\log P(Z_s |\theta^*_{c,i}, \theta^*_{sr,i})}{  \partial \Theta^p_{c,i} \partial \Theta^q_{sr,i}} \right]_{ \begin{matrix}
p = 1,\cdots, j_i,\\
 q = 1,\cdots, d  - j_i
 \end{matrix}} \in \mathbb{R}^{j_i\times (d - j_i)}, \\
I_{ctr,i-1}(\theta^*_{c,i}, \theta^*_{tr,i-1}) &= - \mathbb{E}_{\theta^*_{t,i-1}}\left[\frac{\partial\log P(Z_{t,i-1} |\Theta_{c,i} =\theta^*_{c,i}, \theta^*_{v,i}, \Theta_{tr,i-1} = \theta^*_{tr,i-1})}{\partial \Theta^{p}_{c,i} \partial \Theta^q_{tr,i-1}} \right]_{ \begin{matrix}
p = 1,\cdots, j_i,\\
 q = 1,\cdots, d - c_i - j_i
 \end{matrix}} \in \mathbb{R}^{j_i \times (d- c_i - j_i)}, \\
 I_{ctr,i}(\theta^*_c, \theta^*_{tr,i}) &= - \mathbb{E}_{\theta^*_t}\left[\frac{\partial\log P(Z_{t,i} |\Theta_{c,i} =\theta^*_{c,i}, \theta^*_{v,i}, \Theta_{tr,i} = \theta^*_{tr,i})}{\partial \Theta^{p}_{c,i} \partial \Theta^{q}_{tr,i}} \right]_{ \begin{matrix}
p = 1,\cdots, j_i,\\
 q = 1,\cdots, d - c_i - j_i
 \end{matrix}} \in \mathbb{R}^{j_i \times (d- c_i - j_i)}, \\
 I_{vt,i-1}(\theta^*_{v,i},\theta^*_{tr,i-1}) &= - \mathbb{E}_{\theta^*_t}\left[\frac{\partial\log P(Z_{t,i-1} |\theta^*_{c,i}, \Theta_{v,i} = \theta^*_{v,i}, \Theta_{tr,i-1} = \theta^*_{tr,i-1})}{\partial \Theta^{p}_{v,i} \partial \Theta^q_{tr,i-1}} \right]_{ \begin{matrix}
p = 1,\cdots, c_i,\\
 q = 1,\cdots, d - c_i - j_i
 \end{matrix}} \in \mathbb{R}^{c_i \times (d- c_i - j_i)},  \\
  I_{vt,i}(\theta^*_v,\theta^*_{tr,i}) &= - \mathbb{E}_{\theta^*_t}\left[\frac{\partial\log P(Z_{t,i} |\theta^*_{c,i}, \Theta_{v,i} = \theta^*_{v,i}, \Theta_{tr,i} = \theta^*_{tr,i})}{\partial \Theta^p_{v,i} \partial \Theta^q_{tr,i}} \right]_{ \begin{matrix}
p = 1,\cdots, c_i,\\
 q = 1,\cdots, d - c_i - j_i
 \end{matrix}} \in \mathbb{R}^{c_i \times (d- c_i - j_i)},  \\
 I_{cv_{i-1}}(\theta^*_{c,i},\theta^*_{v,i}) &= - \mathbb{E}_{\theta^*_{t,i-1}}\left[\frac{\partial\log P(Z_{t,i-1} |\Theta_{c,i} = \theta^*_{c,i}, \Theta_{v,i} = \theta^*_{v,i}, \theta^*_{tr,i-1} = \theta^*_{tr,i-1})}{\partial \Theta^p_{v,i} \partial \Theta^q_{tr,i-1}} \right]_{ \begin{matrix}
p = 1,\cdots, j_i,\\
 q = 1,\cdots,  c_i 
 \end{matrix}} \in \mathbb{R}^{j_i \times c_i},  \\
  I_{cv_i}(\theta^*_{c,i},\theta^*_{v,i}) &= - \mathbb{E}_{\theta^*_t}\left[\frac{\partial\log P(Z_{t,i} |\Theta_{c,i} = \theta^*_{c,i}, \Theta_{v,i} = \theta^*_{v,i}, \theta^*_{tr,i})}{\partial \Theta^p_{v,i} \partial \Theta^q_{tr,i}} \right]_{ \begin{matrix}
p = 1,\cdots, j_i,\\
 q = 1,\cdots, c_i
 \end{matrix}} \in \mathbb{R}^{j_i \times c_i},
\end{align} 
where $\Theta^p$ ($\Theta^q$) denotes the $p$th ($q$th) element in $\Theta$. To simplify the notations, we omit the function variables for all Fisher information matrices (for example, we write $I_{cs,i}(\theta^*_{c,i})$ as $I_{cs,i}$). Then the asymptotic normality implies that
\begin{align}
    D(P_{\theta^*_s,\theta^*_{t,i},\theta^*_{t,i-1}}(D^m_s, D^{n_{i-1}}_{t,i-1},D^{n_i}_{t,i})\|Q(D^m_s, D^{n_{i-1}}_{t,i-1},D^{n_i}_{t,i})) - \frac{1}{2}\log\operatorname{det}\left( \mathbf{I}_{\boldsymbol \theta^*}
    \right) \rightarrow \frac{3d - 2j_i - c_i}{2}\log \frac{1}{2\pi e} + \log\frac{1}{\omega(\theta^*_s,\theta^*_t)},
\end{align}
where the Fisher information matrix is defined as
\begin{align}
    \mathbf{I}_{\boldsymbol \theta^*} =  \begin{bmatrix} 
    m I_{cs,i} + n_{i-1}I_{ct,i-1} + n_iI_{ct,i}  & n_i  I_{cv_i}  + n_{i-1}  I_{cv_{i-1}} & mI_{sc,i} & n_{i-1}I_{ctr,i-1} & n_iI_{ctr,i}   \\
    n_i I^T_{ct,i}  + n_{i-1} I^T_{ct,i-1} & n_iI_{v_i} + n_{i-1}I_{v_{i-1}} & \mathbf{0} & n_{i-1}I_{vt,i-1} & n_iI_{vt,i} \\
    mI^T_{sc,i} & \mathbf{0} & mI_{s,i} & \mathbf{0} & \mathbf{0} \\   
    n_{i-1}I^T_{ctr,i-1} & n_{i-1}I^T_{vt,i-1} & \mathbf{0} & n_{i-1} I_{t,i-1}  & \mathbf{0} \\
    n_iI^T_{ctr,i} & n_iI^T_{vt,i} & \mathbf{0} & \mathbf{0} & n_i I_{t,i}
    \end{bmatrix},
\end{align}
as all $n_i$,$n_{i-1}$ and $m$ go to infinity. Similarly,
\begin{align}
D(P_{\theta^*_s,\theta^*_{t,i-1}}(D^m_s, D^{n_{i-1}}_{t,i-1})\|Q(D^m_s, D^{n_{i-1}}_{t,i-1})) - \frac{1}{2}\log\operatorname{det}\left( \mathbf{I}_{\boldsymbol \theta^*_s}
    \right) \rightarrow \frac{2d - j_i}{2}\log \frac{1}{2\pi e} + \log\frac{1}{\omega(\theta^*_s,\theta^*_t)},
\end{align}
where
\begin{align}
    \mathbf{I}_{\boldsymbol \theta^*_s} &= \begin{bmatrix} 
     m I_{cs,i} + n_{i-1}I_{ct,i-1}   &  n_{i-1}  I_{cv_{i-1}} &  mI_{sc,i}   &  n_{i-1}I_{ctr,i-1}  \\
    n_{i-1}  I^T_{cv_i-1} & n_{i-1}I_{v_i-1} & \mathbf{0} & n_{i-1}I_{vt,i-1}   \\
     mI^T_{sc,i} & \mathbf{0} & mI_{s,i} & \mathbf{0}  \\   
    n_{i-1}I^T_{ctr,i-1} & n_{i-1}I^T_{vt,i-1} & \mathbf{0} & n_{i-1} I_{t,i-1}\\
    \end{bmatrix}. 
\end{align}
Then by subtraction, we have
\begin{align}
    \frac{\log\operatorname{det}(\mathbf{I}_{\boldsymbol \theta^*})}{\log\operatorname{det}(\mathbf{I}_{\boldsymbol \theta^*_s} )} = \log\operatorname{det}\left(\mathbf{I}_{j_i \times j_i} + \frac{n_i}{m+n_{i-1}}\Delta_{ct,i}\Delta^{-1}_{cst,i} \right)  + \log \operatorname{det}(\mathbf{I}_{c_i\times c_i} + \frac{n_i}{n_{i-1}} \Delta_{t,i} \Delta^{-1}_{t,i-1}) + \log\operatorname{det}(n_iI_{t,i}), 
\end{align}
where $\Delta_{ct,i} = I_{ct,i} - I_{ctr,i}I^{-1}_{t,i}I^T_{ctr,i}$, $\Delta_{cst,i} = \frac{m}{m+n_{i-1}}(I_{cs} - I_{sc,i}I^{-1}_{s,i}I^T_{sc,i} )+ \frac{n_{i-1}}{m+n_{i-1}} (I_{ct,i-1} - I_{ctr,i-1}I^{-1}_{t,i-1}I^T_{ctr,i-1})$, $\Delta_{t,i} = I_{v_i} - I_{vt,i}I^{-1}_{t,i}I^T_{vt,i}$, and $\Delta_{t,i-1} = I_{v_{i-1}} - I_{vt,i-1}I^{-1}_{t,i-1}I^T_{vt,i-1}$. Then we have the asymptotic estimation as
\begin{small}
\begin{align}
I(D^{n_i}_{t,i};\Theta_{t,i} = \theta^*_{t,i}|D^m_s, D^{n_{i-1}}_{t,i-1}) \rightarrow &  \log\operatorname{det}\left(\mathbf{I}_{j_i \times j_i} + \frac{n_i}{m+n_{i-1}}\Delta_{ct,i}\Delta^{-1}_{cst,i} \right)  + \log \operatorname{det}(\mathbf{I}_{c_i\times c_i} + \frac{n_i}{n_{i-1}} \Delta_{t,i} \Delta^{-1}_{t,i-1}) + \log\operatorname{det}(n_iI_{t,i}) \nonumber  \\
&+(d - j_i - c_i)\log\frac{1}{2\pi e} + \frac{2}{\omega(\theta^*_{t,i}|\theta^*_{t,i-1},\theta^*_s)}.
\end{align}
\end{small}
\end{proof}
\noindent Based on the theorem above, by putting things together, finally we reach
\begin{align}
    \mathcal{R}_{TV} \leq & M \Bigg( l \sum_{i=1}^{l} n_i \Big( \log\operatorname{det}\left(\mathbf{I}_{j_i \times j_i} + \frac{n_i}{m+n_{i-1}}\Delta_{ct,i}\Delta^{-1}_{cst,i} \right) \nonumber + \log \operatorname{det}(\mathbf{I}_{c_i\times c_i} + \frac{n_l}{n_{l-1}} \Delta_{t,i} \Delta^{-1}_{t,i-1}) + \log\operatorname{det}(n_iI_{t,i}) \nonumber  \\
    &+(d - j_i - c_i)\log\frac{1}{2\pi e} + \frac{2}{\omega(\theta^*_{t,i}|\theta^*_{t,i-1},\theta^*_s)}\Big) \Bigg)^{\frac{1}{2}}. 
\end{align}
Since $n_{l-1} \asymp n_l$ and $m \asymp n^p_l$, using the same procedure from Corollary~\ref{coro:rate-gene-online} we will arrive at
\begin{align}
    \mathcal{R}_{TV} \lesssim &  \sqrt{ k \sum_{l=1}^{k}  n_l \Big( j_i (1 \wedge n_l^{1-p})+ c_i   + (d-c_i-j_i)\log n_i + \frac{2}{\omega(\theta^*_{t,i}|\theta^*_{t,i-1},\theta^*_s)} \Big)  }.
\end{align}

\end{proof}

\subsection{Proof of Proposition~\ref{prop:neg-inst}}\label{proof:neg-inst}
\begin{proof}
Since we have
\begin{align}
    \mathcal{R}_I = \mathbb{E}_{\theta^*_t,\theta^*_s}\left[  \log \frac{P_{\theta^*_t}(Z'_t)}{Q(Z'_t|D^n_t, D^m_s)}\right], 
\end{align}
where we use the mixture strategy for the conditional distribution $Q$ as
\begin{align}
    \mathbb{E}\left[ \log Q(Z'_t|D^n_t, D^m_s)\right] &= \mathbb{E}\left[ \log  \frac{\int P_{\theta_t}(D^n_t, Z'_t) P_{\theta_s}(D^m_s) \omega(\theta_t,\theta_s) d\theta_t d\theta_s }{\int P_{\theta_t}(D^n_t) P_{\theta^*_s}(D^m_s) \omega(\theta_t,\theta_s) d\theta_t d\theta_s  }\right] \\
    &= \mathbb{E}\left[ \log  \frac{\int P_{\theta_t}(Z'_t) P_{\theta_t}(D^n_t) \omega(\theta_t|\theta_s) d\theta_t P(\theta_s|D^m_s) d\theta_s }{\int P_{\theta_t}(D^n_t) \omega(\theta_t|\theta_s) d\theta_t P(\theta_s|D^m_s) d\theta_s} \right] \\
    &= \mathbb{E}\left[ \log  \frac{\int P_{\theta_t}(Z'_t) P_{\theta_t}(D^n_t) \omega(\theta_t|\theta^*_s) d\theta_t }{\int P_{\theta_t}(D^n_t) \omega(\theta_t|\theta^*_s) d\theta_t } \right] \\
    &= \mathbb{E}\left[ \log  \int P_{\theta_t}(Z'_t) P_{\omega}(\theta_t|D^n_t) d\theta_t \right] \\
&\leq \max_{\tilde{\theta}_t \in\operatorname{supp}(\omega(\Theta_t|\Theta^*_s))} \mathbb{E}\left[ \log  \int  P_{\tilde{\theta}_t}(Z'_t) P_{\omega}(\theta_t|D^n_t) d\theta_t \right] \\
    & =  \mathbb{E}\left[ \log  \int P_{\tilde{\theta}_t}(Z'_t) P_{\omega}(\theta_t|D^n_t) d\theta_t \right] \\
    &= \mathbb{E}\left[ \log  P_{\tilde{\theta}_t}(Z'_t) \right].
\end{align}
The inequality holds as $\tilde{\theta}_t = \argmin_{\theta_t \in\operatorname{supp}(\omega(\Theta_t|\Theta^*_s))} D_{\textup{KL}}(P_{\theta^*_t}(Z_t)\|P_{{\theta_t}}(Z_t))$ and we define the conditional posterior as
\begin{align}
    P_{\omega}(\theta_t|D_t) = \frac{\omega(\theta_t|\theta^*_s)P_{\theta_t}(D^n_t)}{\int \omega(\theta_t|\theta^*_s)P_{\theta_t}(D^n_t) d\theta_t }.
\end{align}
The domain of this posterior is the same as the support of $\omega(\theta_s|\theta_t)$. Therefore, minimising the KL divergence $D_{\textup{KL}}(P_{\theta^*_t}(Z_t)\|P_{{\theta_t}}(Z_t))$ w.r.t. $\theta_t$ is equivalent to maximising the cross entropy as
\begin{align}
    \argmin_{\theta_t \in\operatorname{supp}(\omega(\Theta_t|\Theta^*_s))} D_{\textup{KL}}(P_{\theta^*_t}(Z_t)\|P_{{\theta_t}}(Z_t)) = \argmax_{\theta_t \in\operatorname{supp}(\omega(\Theta_t|\Theta^*_s))} \mathbb{E}_{\theta^*_t}[\log P_{\theta_t}(Z_t)].
\end{align}
Then the result follows. 
\end{proof}

\subsection{Proof of Proposition \ref{prop:neg-online}} \label{proof:neg-online}
\begin{proof}
We need to prove there exists a prior $\omega(\theta_s, \theta_t)$, the expected regrets such that
\begin{align}
    \mathbb{E}_{\theta^*_s,\theta^*_t} \left[ \log \frac{P_{\theta^*_t}(D^n_t)}{Q(D^n_t| D^{m}_s)} \right] > \mathbb{E}_{\theta^*_t} \left[ \log \frac{P_{\theta^*_t}(D^n_t)}{\hat{Q}(D^n_t)} \right]. 
\end{align}
It is equivalent to prove that, 
\begin{align}
    \mathbb{E}_{\theta^*_t, \theta^*_s} \left[ \log \frac{Q(D^m_s)\hat{Q}(D^n_t)}{Q(D^n_t, D^{m}_s)} \right] > 0.
\end{align}
Let us examine the logarithmic term,
\begin{align}
    \log \frac{Q(D^m_s)\hat{Q}(D^n_t)}{Q(D^n_t, D^{m}_s)} &= \log \frac{\hat{Q}(D^n_t) \int P_{\theta_t,\theta_s}(D^m_s)\omega(\theta_t, \theta_s)d \theta_t d\theta_s}{\int P_{\theta_t,\theta_s}(D^n_t, D^m_s)\omega(\theta_t, \theta_s)d \theta_t d\theta_s} \\
    &= \log \frac{\hat{Q}(D^n_t) \int P_{\theta_s}(D^m_s)\omega(\theta_s) d\theta_s}{\int P_{\theta_t,\theta_s}(D^n_t, D^m_s)\omega(\theta_t, \theta_s)d \theta_t d\theta_s} \\
    &= \log \frac{1}{\int \int \hat{Q}(\theta_t|D^n_t)\frac{\omega(\theta_t|\theta_s)}{\hat{\omega}(\theta_t)}d \theta_t Q(\theta_s|D^m_s)d\theta_s}.
\end{align}
When both $m$ and $n$ are sufficient enough and the marginal prior distribution $\omega(\theta_s)$ and $\hat{\omega}(\theta_t)$ are proper, $\hat{Q}(\theta_t|D^n_t)$ and $Q(\theta_s|D^m_s)$ will be concentrated near $\theta^*_t$ and $\theta^*_s$. Then the above equation becomes
\begin{align}
       \log \frac{Q(D^m_s)\hat{Q}(D^n_t)}{Q(D^n_t, D^{m}_s)}  &= - \log {\int^{\theta^*_s + \delta_s}_{\theta^*_s -\delta_s}\int^{\theta^*_t + \delta_t}_{\theta^*_t -\delta_t} \hat{Q}(\theta_t|D^n_t)\frac{\omega(\theta_t|\theta^*_s)}{\hat{\omega}(\theta_t)}  d \theta_t Q(\theta_s|D^m_s)d\theta_s},
\end{align}
for some small $\delta_t$ and $\delta_s$. Since the prior $\omega(\theta_t|\theta_s)$ is imposed improperly, then $\omega(\theta_t|\theta_s)$ has zero density around $\theta^*_t$, then since for any $\hat{\omega}(\theta_t) > 0$, the following inequality holds.
\begin{align}
     \mathbb{E}_{\theta^*_s,\theta^*_t} \left[ \log \frac{\hat{Q}(D^n_t)Q(D^m_s)}{Q(D^n_t,D^m_s)} \right]  &=  -\mathbb{E}_{\theta^*_s,\theta^*_t} \left[ \log {\int^{\theta^*_s + \delta_s}_{\theta^*_s -\delta_s}\int^{\theta^*_t + \delta_t}_{\theta^*_t -\delta_t} \hat{Q}(\theta_t|D^n_t)\frac{\omega(\theta_t|\theta^*_s)}{\hat{\omega}(\theta_t)}  d \theta_t Q(\theta_s|D^m_s)d\theta_s} \right] \\
    & > -\mathbb{E}_{\theta^*_s,\theta^*_t} \left[  \log {\int^{\theta^*_s + \delta_s}_{\theta^*_s -\delta_s}\int^{\theta^*_t + \delta_t}_{\theta^*_t -\delta_t} \hat{Q}(\theta_t|D^n_t)\frac{\hat{\omega}(\theta_t)}{\hat{\omega}(\theta_t)}  d \theta_t Q(\theta_s|D^m_s)d\theta_s} \right] \\
    & = 0,
\end{align}
when the source and target are sufficiently large. Therefore, it implies that
\begin{align}
    \mathcal{R}_{\omega(\Theta_s,\Theta_t)}(n) > \mathcal{R}_{\hat{\omega}(\Theta_t)}(n).
\end{align}
Furthermore, if we rewrite the regret as
\begin{align*}
    \mathcal{R}_O &= \mathbb{E}_{\theta^*_t}\left[ \log\frac{P_{\theta^*_t}(D^n_t)}{Q(D^n_t|D^m_s)} \right] \\
    &= H_{P_{\theta^*_t}}(D^n_t) - \mathbb{E}_{\theta^*_t}\left[\log Q(D^n_t|D^m_s) \right] \\
    &= H_{P_{\theta^*_t}}(D^n_t) - \mathbb{E}_{\theta^*_t}\left[ \int \int P_{\theta_t}(D^n_t)\omega(\theta_t|\theta_s)  d\theta_t Q(\theta_s|D^m_s) d\theta_s \right] \\
    &\overset{(a)}{=} H_{P_{\theta^*_t}}(D^n_t) - \mathbb{E}_{\theta^*_t}\left[ \int P_{\theta_t}(D^n_t)\omega(\theta_t|\theta^*_s)  d\theta_t \right] \\
    &\overset{(b)}{\geq} H_{P_{\theta^*_t}}(D^n_t) - \mathbb{E}_{\theta^*_t}\left[ P_{\tilde{\theta}_t}(D^n_t) \right] \\
    &= D_{\textup{KL}}(P_{\theta^*_t}(D^n_t)\| P_{\tilde{\theta}_t}(D^n_t)) \\
    &\overset{(c)}{=} n D_{\textup{KL}}(P_{\theta^*_t}(Z_t)\| P_{\tilde{\theta}_t}(Z_t)),
\end{align*}
where $(a)$ follows that when $m$ goes to sufficiently large and we assume that the prior $\omega(\theta_s)$ is proper, the posterior $Q(\theta_s|D^m_s)$ will approach the true parameter $\theta^*_s$. $(b)$ holds as we define
\begin{equation}
\tilde{\theta}_t = \argmin_{\theta_t \in\operatorname{supp}(\omega(\Theta_t|\Theta^*_s))} D_{\textup{KL}}(P_{\theta^*_t}(Z_t)\|P_{{\theta_t}}(Z_t)),
\end{equation}
which is equivalent to maximise the cross entropy as
\begin{equation}
\tilde{\theta}_t = \argmax_{\theta_t \in\operatorname{supp}(\omega(\Theta_t|\Theta^*_s))} \mathbb{E}_{\theta^*_t}\left[ \log P_{\theta_t}(Z_t)  \right].
\end{equation}
With the assumption~\ref{asp:para-dist}, with the i.i.d. property we have
\begin{equation}
\tilde{\theta}_t = \argmax_{\theta_t \in\operatorname{supp}(\omega(\Theta_t|\Theta^*_s))} \mathbb{E}_{\theta^*_t}\left[ \log P_{\theta_t}(D^n_t)  \right].
\end{equation}
Therefore, $(b)$ holds as
\begin{equation}
\mathbb{E}_{\theta^*_t}\left[ \int P_{\theta_t}(D^n_t)\omega(\theta_t|\theta^*_s)  d\theta_t \right]  \leq \max_{\tilde{\theta}_t \in\operatorname{supp}(\omega(\Theta_t|\Theta^*_s))} \mathbb{E}_{\theta^*_t}\left[ \int P_{\tilde{\theta}_t}(D^n_t) \omega(\theta_t|\theta^*_s)  d\theta_t \right] = \mathbb{E}_{\theta^*_t}\left[ P_{\tilde{\theta}_t}(D^n_t) \right], 
\end{equation}
and finally $(c)$ follows as target data are drawn i.i.d. from $P_{\theta^*_t}$. With the dual properties, the proof of the upper bound follows the same procedures as the lower upper.
\end{proof}

\subsection{Proof of Proposition \ref{claim:positive}} \label{proof:positive}
\begin{proof}
We need to prove under the certain assumptions, there exists a prior $\omega(\theta_s, \theta_t)$, the expected regrets such that
\begin{align}
    \mathbb{E}_{\theta^*_s,\theta^*_t} \left[ \log \frac{P_{\theta^*_t}(D^n_t)}{Q(D^n_t| D^{m}_s)} \right] < \mathbb{E}_{\theta^*_t} \left[ \log \frac{P_{\theta^*_t}(D^n_t)}{\hat{Q}(D^n_t)} \right]. 
\end{align}
Rewrite the expectation and it is equivalent to prove that
\begin{align}
    \mathbb{E}_{\theta^*_t, \theta^*_s} \left[ \log \frac{Q(D^m_s)\hat{Q}(D^n_t)}{Q(D^n_t, D^{m}_s)} \right] < 0.
\end{align}
Similarly, when $D^m_s$ is sufficient enough, the density $P(\theta^*_s|D^m_s)$ will concentrate around $\theta^*_s$ , say $P(|\theta_s - \theta^*_s| < \delta_s) \rightarrow 0$ as $m$ goes to infinity, furthermore if $D^n_t$ is also very large, $p(\theta_t|D^n_t)$ will be concentrated near $\theta^*_t$ such that
\begin{align}
       \log \frac{Q(D^m_s)\hat{Q}(D^n_t)}{Q(D^n_t, D^{m}_s)}  &= - \log {\int^{\theta^*_s + \delta_s}_{\theta^*_s -\delta_s}\int^{\theta^*_t + \delta_t}_{\theta^*_t -\delta_t} \hat{Q}(\theta_t|D^n_t)\frac{\omega(\theta_t|\theta^*_s)}{\hat{\omega}(\theta_t)}  d \theta_t Q(\theta_s|D^m_s)d\theta_s}.
\end{align}
Since we assume the support $\omega(\theta_t|\theta_s)$ is a proper subset of $\Theta$ and $\omega(\theta_t|\theta_s)$ is proper over $\theta_t$, that is, this conditional prior has  positive density around $\theta^*_t$. Let us define
\begin{align}
    \hat{\Omega} &= \int_{\theta^*_t - \delta_t}^{\theta^*_t + \delta_t}\hat{\omega}(\theta_t)d\theta_t \\
    \Omega &= \int_{\theta^*_t - \delta_t}^{\theta^*_t + \delta_t} \omega(\theta_t|\theta_s) d\theta_t.
\end{align}
Then there always exists a prior such that for any $\theta_s$ around $\theta^*_s$,
$$
\Omega - \hat{\Omega} = \Delta > 0,
$$
with the choice of the prior
\begin{align}
\omega(\theta_t|\theta_s) = \hat{\omega}(\theta_t) + \frac{\Delta}{2\delta_t}.
\end{align}
This specific prior will lead to $\log\frac{\hat{\omega}(\theta_t)}{\omega(\theta_t|\theta_s)} < 1$ and the quantity above will be strictly less than zero, which is, positive transfer.
\end{proof}

\bibliographystyle{apalike}
\bibliography{reference}

\appendix

\section{Appendix}

\subsection{Proof of Lemma~\ref{lemma:bernoulli}}
\begin{proof}
\begin{align}
Q(D^n_t|D^m_s) &= \int_{\theta_s}\int_{\theta_t|\theta_s}  P(D^n_t|\theta_t) P(\theta_t|\theta_s)d\theta_t P(\theta_s|D^m_s)  d\theta_s \\
&= \int^{\theta^*_s + c}_{\theta^*_s - c} \frac{1}{2c} P(D^n_t|\theta_t) d\theta_t \\
&=\frac{1}{2c} \int^{\theta^*_s + c}_{\theta^*_s - c} (\theta_t)^{k_t} (1-\theta_t)^{n - k_t} d\theta_t 
\end{align}
We denote,
\begin{equation}
I(n, k)=\left(\begin{array}{l}
n \\
k
\end{array}\right) \int_{0}^{a} x^{k}(1-x)^{n-k} d x
\end{equation}
Then we have,
\begin{align}
I(n, k) &=\left(\begin{array}{l}
n \\
k
\end{array}\right)\left(\left[ \frac{-x^{k}(1-x)^{n-k+1}}{n-k+1}\right]_{0}^{a}+\frac{k}{n-k+1} \int_{0}^{a} x^{k-1}(1-x)^{n-k+1} d x\right) \\
&= \left(\begin{array}{l}
n \\
k
\end{array}\right)\frac{-a^{k}(1-a)^{n-k+1}}{n-k+1} + \left(\begin{array}{l}
n \\
k
\end{array}\right) \frac{k}{n-k+1}\left(\begin{array}{c}
n \\
k-1
\end{array}\right)^{-1}I(n,k-1) \\
&= \left(\begin{array}{l}
n \\
k
\end{array}\right)\frac{-a^{k}(1-a)^{n-k+1}}{n-k+1} + I(n,k-1)
\end{align}
By induction,
\begin{align}
I(n, k) &=  \sum^k_{i=1} \left(\begin{array}{l}
n \\
i
\end{array}\right)\frac{-a^{i}(1-a)^{n-i+1}}{n-i+1} + I(a,0) \\
&= \sum^k_{i=1} \left(\begin{array}{l}
n \\
i
\end{array}\right)\frac{-a^{i}(1-a)^{n-i+1}}{n-i+1} + \frac{1-(1-a)^{n+1}}{n+1}
\end{align}
Hence,
\begin{align}
\int_{0}^{a} x^{k}(1-x)^{n-k} d x = \left(\begin{array}{l}
n \\
k
\end{array}\right)^{-1} \left( \sum^k_{i=1} \left(\begin{array}{l}
n \\
i
\end{array}\right)\frac{-a^{i}(1-a)^{n-i+1}}{n-i+1} + \frac{1-(1-a)^{n+1}}{n+1} \right)
\end{align}
and for any $b > a$,
\begin{align}
\int_{0}^{b} x^{k}(1-x)^{n-k} d x = \left(\begin{array}{l}
n \\
k
\end{array}\right)^{-1} \left( \sum^k_{i=1} \left(\begin{array}{l}
n \\
i
\end{array}\right)\frac{-b^{i}(1-b)^{n-i+1}}{n-i+1} + \frac{1-(1-b)^{n+1}}{n+1} \right)
\end{align}
By subtraction,
\begin{align}
\frac{1}{b - a}\int_{a}^{b} x^{k}(1-x)^{n-k} d x =  \frac{1}{b-a}\left(\begin{array}{l}
n \\
k
\end{array}\right)^{-1} \left( \sum^k_{i=1} \left(\begin{array}{l}
n \\
i
\end{array}\right)\frac{a^{i}(1-a)^{n-i+1} - b^{i}(1-b)^{n-i+1} }{n-i+1} + \frac{(1-a)^{n+1}-(1-b)^{n+1}}{n+1} \right)
\end{align}
\end{proof}

\end{document}